\documentclass[sigconf]{acmart}

\AtBeginDocument{%
	\providecommand\BibTeX{{%
			\normalfont B\kern-0.5em{\scshape i\kern-0.25em b}\kern-0.8em\TeX}}}
\usepackage{booktabs}
\usepackage{subfigure}
\usepackage{bbm}
\usepackage{graphicx}
\usepackage{multirow}
\usepackage{enumitem}
\usepackage{amsmath}
\usepackage{stmaryrd}
\usepackage[linesnumbered,ruled]{algorithm2e}
\usepackage{threeparttable}
\usepackage{footmisc}
\usepackage{hyperref}
\usepackage{xspace}
\usepackage{amsmath}

\newcommand{\grand}{GRAND\xspace}
\newcommand{\model}{GRAND+\xspace}

\newcommand{\full}{\textsc{Graph Random Neural Networks}}

\newcommand{\norm}[1]{\left\lVert#1\right\rVert}

\newtheorem{thm}{Theorem} 
\newtheorem{lemma}{Lemma}
\newcommand{\tabincell}[2]{\begin{tabular}{@{}#1@{}}#2\end{tabular}}

\newcommand{\vpara}[1]{\vspace{0.05in}\noindent\textbf{#1}\xspace}

\newcommand{\reminder}[1]{\textbf{\color{red}[** #1 **]}} 

\newcommand{\dong}[1]{\textit{{\color{red}#1 }}}

\makeatletter

\usepackage{footnote}
\usepackage{footmisc}
\makesavenoteenv{tabular}
\makesavenoteenv{table}

\newcommand{\astfootnote}[1]{
\let\oldthefootnote=\thefootnote
\setcounter{footnote}{0}
\footnote{#1}
\let\thefootnote=\oldthefootnote
}

\usepackage{threeparttable}
\theoremstyle{definition}

\theoremstyle{problemstyle}  % <name>
\newcommand{\hide}[1]{} %hide
\usepackage[T1]{fontenc}
\usepackage{aecompl}

\setlist[itemize]{leftmargin=3.5mm}

\copyrightyear{2022} 
\acmYear{2022} 
\setcopyright{acmcopyright}\acmConference[WWW '22]{Proceedings of the ACM Web Conference 2022}{April 25--29, 2022}{Virtual Event, Lyon, France}
\acmBooktitle{Proceedings of the ACM Web Conference 2022 (WWW '22), April 25--29, 2022, Virtual Event, Lyon, France}
\acmPrice{15.00}
\acmDOI{10.1145/3485447.3512044}
\acmISBN{978-1-4503-9096-5/22/04}

\begin{document}

%
% The "title" command has an optional parameter, allowing the author to define a "short title" to be used in page headers.
%\title{Graph Random Network: Random Is All You Need For Semi-supervised Learning On Graphs}
\title{GRAND+: Scalable Graph Random Neural Networks}
%\jt{\normalsize can you please move all figure files to a sub-folder, e.g., pics\\
%also change all files names to 0main.tex; 1abstract.tex; 2intro.tex, so that all the files can be ordered and can be easily located.}

% The "author" command and its associated commands are used to define the authors and their affiliations.
% Of note is the shared affiliation of the first two authors, and the "authornote" and "authornotemark" commands
% used to denote shared contribution to the research.
\author{Wenzheng Feng$^1$, Yuxiao Dong$^1$, Tinglin Huang$^1$, Ziqi Yin$^2$, Xu Cheng$^1$$^3$}
\author{Evgeny Kharlamov$^4$, Jie Tang$^1$}
%\affiliation{\institution{$^1$Tsinghua University, $^{2}$Zhejiang University, $^{3}$Beijing Institute of Technology}} 
%\affiliation{\institution{$^4$Tencent Inc., $^{5}$Bosch Center for Artificial Intelligence}} 

 \affiliation{$^1$Tsinghua University \country{} $^{2}$Beijing Institute of Technology  \country{}}
 \affiliation{ $^3$Tencent Inc.  \country{} $^{4}$Bosch Center for Artificial Intelligence  \country{}}
\email{fwz17@mails.tsinghua.edu.cn,{yuxiaod, tl091, jietang}@tsinghua.edu.cn, ziqiyin18@gmail.com}%,  jietang@tsinghua.edu.cn} 
\email{alexcheng@tencent.com, evgeny.kharlamov@de.bosch.com}
% \small{\texttt{tinglin.huang@zju.edu.cn, ericdongyx@gmail.com, dm18@mails.tsinghua.edu.cn, zheny2751@gmail.com}}
\thanks{The code is available at \url{https://github.com/wzfhaha/GRAND-plus}.
\thanks{Jie Tang is the corresponding author.}
}

\renewcommand{\shortauthors}{Feng et al.}
\renewcommand{\authors}{Wenzheng Feng, Yuxiao Dong, Tinglin Huang, Ziqi Yin, Xu Cheng, Evegeny Kharlamov, Jie Tang}

% By default, the full list of authors will be used in the page headers. Often, this list is too long, and will overlap
% other information printed in the page headers. This command allows the author to define a more concise list
% of authors' names for this purpose.

%\renewcommand{\shorttitle}{GRAND+: Scalable Graph Random Neural Networks} 

%\sloppy

%
% The abstract is a short summary of the work to be presented in the article.
\begin{abstract}
%Graph data is a commonplace on the Web where information is naturally interconnected as in social networks or online academic publications.
%The inherit incompleteness of Web data spark interest in the problem of graph-based semi-supervised learning (GSSL) and Graph neural networks (GNNs) are considered by many as the best way to address it.

%However, existing GNNs suffer either from the lack of generalization performance or from weak scalability, while in the context of the Web both are required. 

Graph neural networks (GNNs) have been widely adopted for semi-supervised learning on graphs. 
A recent study shows that the graph random neural network (GRAND) model can generate state-of-the-art performance for this problem. %~\cite{}. \dong{to add} 
%GRAND is a GNN consistency regularization framework 
However, it is difficult for GRAND to handle large-scale graphs since its effectiveness relies on computationally expensive data augmentation procedures.  %operations. 
%the power iteration of the input feature matrix. 
In this work, we present a scalable and high-performance GNN framework \model for semi-supervised graph learning. 
%\model extends the existing framework \grand in several important dimensions. 
To address the above issue, 
%In particular, in order to accelerate model training on large data, 
we develop a generalized forward push (GFPush) algorithm in \model 
%\model\ adopts a novel approximation algorithm called GFPush 
to pre-compute a general propagation matrix and employ it to perform graph data augmentation in a mini-batch manner. 
We show that both the low time and space complexities of GFPush enable \model to efficiently scale to large graphs. 
%are independent of the graph size. 
%could be efficiently applied onto large graphs with detailed complexity analyses of GFPush.
Furthermore, we introduce a confidence-aware consistency loss into the model optimization of \model, facilitating GRAND+'s generalization superiority. 
%Additionally, to further improve the performance of \model\ on large data, we introduce a novel confidence-aware consistency loss into model optimization.
We conduct extensive experiments on seven public datasets of different sizes. 
The results demonstrate that \model\ 1) is able to scale to large graphs and costs less running time than existing scalable GNNs, and 2) can  offer consistent accuracy improvements over both full-batch and scalable GNNs across all datasets. 

%scales well and achieves the best accuracy compared with representative GNNs across all datasets.
\end{abstract}

\hide{%============

Graph data is a commonplace on the Web where information is naturally interconnected as in social networks or online academic publications.
The inherit incompleteness of Web data spark interest in the problem of graph-based semi-supervised learning (GSSL) and Graph neural networks (GNNs) are considered by many as the best way to address it.
However, existing GNNs suffer either from the lack of generalization performance or from weak scalability, while in the context of the Web both are required. In this work we propose \model, 
a scalable and high-performance GNN framework for GSSL.
\model extends the existing framework \grand in several important dimensions. 
In particular, in order to accelerate model training on large data, \model\ adopts a novel approximation algorithm called GFPush to pre-compute a general propagation matrix, and employs it to perform data augmentation and model learning in a mini-batch manner. We show that this method could be efficiently applied onto large graphs with detailed complexity analyses of GFPush.
Additionally, to further improve the performance of \model\ on large data, we introduce a novel confidence-aware consistency loss into model optimization.
We conduct extensive experiments on seven public datasets of different sizes, the results demonstrate that \model\ scales well and achieves the best accuracy compared with representative GNNs across all datasets.
}%===========================

\hide{
%Evgeny: this was the version I replaced

In this paper, we study graph-based semi-supervised learning (GSSL).
%(GSSL), whose objective is to make predictions for graph-structured samples with limited supervision information.  
Graph neural networks (GNNs) have been widely explored and become the most popular solution to this problem.
However, existing GNNs often suffer from bottlenecks in terms of generalization performance and scalability. Previous works in this area often focus on solving one of the two problems while ignoring the other one. Recently, Graph Random Neural Network (GRAND)~\cite{feng2020grand} %has been proposed to improve generalization capability by utilizing graph data augmentation, and 
has achieved state-of-the-art performance for this task by utilizing graph data augmentation and consistency regularization. 
%For example, the recent proposed Graph Random Neural Network (GRAND)~\cite{feng2020grand}  achieves good generalization capability and state-of-the-art performance by utilizing graph data augmentation. 
However, this approach is difficult to be deployed onto large realistic graphs as its data augmentation step relies on an expensive power iteration process to propagate node features. How to efficiently and accurately perform GSSL on large graphs still remains an open challenge.
Therefore, we propose \model,   
a scalable and high-performance GNN framework for GSSL. 
%To address the scalability limitation, 
To accelerate model training, \model\ adopts a novel approximation algorithm called GFPush to pre-compute a general propagation matrix, and employs it to perform data augmentation and model learning in a mini-batch manner. We theoretically show that this paradigm could be efficiently applied onto large graphs with detailed complexity analyses of GFPush.
% the model to perform data augmentation on large graphs. 
Additionally, \model\ introduces a new confidence-aware consistency loss into model optimization, further improving performance over \grand.
%We also theoretically prove the efficiency of GFPush with a detailed complexity analysis. 
% under semi-supervised setting. 
We conduct extensive experiments on seven public datasets of different sizes, the results demonstrate that \model\ scales well and achieves the best accuracy compared with representative GNNs accross all datasets.
%not only achieves better generalization performance over \grand\ on three small benchmark graphs, but also significantly outperforms scalable GNN baselines on four relatively large graphs with good efficiency. 
%As far as we know, we are the first to explore how to maintain both the scalability and generalization capability of GNNs in semi-supervised setting.
}

\hide{
In this paper, we study graph-based semi-supervised learning (GSSL), whose objective is to make predictions for graph-structured samples with limited supervision information.  Recently, graph neural networks (GNNs) have been explored to solve this problem and achieved promising results.
However, existing GNNs often suffer from bottlenecks in terms of generalization and scalability. Recent works in this area often focus on solving one of the two problems while ignoring the other one.
%Recent efforts in this area are often devoted to solve one of the two problems but eliminates the other one.
%1) Conventional GNNs are only trained with supervised objective, which would be inclined to overfit the limited labeled samples.
%2) Most of GNNs heavily rely on the expensive recursive feature propagation and can not handle large graphs. 
%Though several efforts have been devoted to solving one of the two problems, how to design a scalable GNN model with a strong generalization capacity (good performance on unlabeled nodes) for this task is still an open question. 
For example, the recent proposed Graph Random Neural Network (GRAND)~\cite{feng2020grand}  achieves good generalization capability and state-of-the-art performance by utilizing graph data augmentation. However, this approach is difficult to be deployed to large graphs as it employs expensive mixed-order feature propagation in data augmentation. Thus how to design a GNN framework that could avoid the two limitations of generalization and scalability simultaneously is still an open challenge.
%with both strong generalization (good performance on unlabeled nodes)  and scalability for this task is still an open question. 
%  to be effective in improving generalization capacity by utilizing graph data augmentation. 
To fulfill this gap, building upon GRAND, we propose \model\ in this paper. 
%  \model. \model\ builds on Graph Random Neural Networks (GRAND),  a data augmentation based approach to GSSL. 
To address the scalability limitation, we develop an efficient approximation method---GFPush---to approximate a generalized mixed-order sub-matrix for feature propagation, which further serves mini-batch training over large graphs. %, enabling the model to perform data augmentation on large graphs.  
We also devise a confidence-aware consistency training strategy, further improving model's generalization by leveraging massive unlabeled data.
% under semi-supervised setting. 
We conduct extensive experiments on various graphs with different genres and scales. Results show that \model\ not only achieves better generalization performance over \grand\ on three small benchmark graphs, but also significantly outperforms scalable GNN baselines on four relatively large graphs with good efficiency. %As far as we know, we are the first to explore how to maintain GNN's scalability and generalization capability simultaneously for graph-based semi-supervised learning
%As far as we know, this is the first GNN-based framework dedicated to GSSL with both scalability and high generalization capability.
As far as we know, we are the first to explore how to maintain both the scalability and generalization capability of GNNs in semi-supervised setting.
}

%capacity with a confident-based consistency regularization strategy. \model\ achieves
%Motivated by the recent proposed Graph Random Neural Network (GRAND), 
%Recently, Feng et. al.~\cite{feng2020grand} proposes Graph Random Neural Network (GRAND), a highly effective framework for semi-supervised graph learning. GRAND first augments graph data with a randomized mixed-order feature propagation and then optimizes model's parameters with both supervised loss and consistency loss among multiple augmentations of unlabeled data. This paradigm has been proved to be powerful in improving models generalization by fully leveraging the massive unlabeled data. Nevertheless, it can not handle large graphs as it requires 

\hide{
Graph neural networks (GNNs) have generalized deep learning methods into graph-structured data with promising performance on graph mining tasks. 
However, existing GNNs often meet complex graph structures with scarce labeled nodes and suffer from the limitations of 
non-robustness~\cite{zugner2018adversarial,ZhuRobust}, over-smoothing~\cite{chen2019measuring,li2018deeper,Lingxiao2020PairNorm}, and overfitting~\cite{goodfellow2016deep,Lingxiao2020PairNorm}. 
To address these issues, we propose a simple yet effective GNN framework---Graph Random Neural Network (\model). 
Different from the deterministic propagation in existing GNNs, \model\ adopts a random propagation strategy to enhance model robustness. 
This strategy also naturally enables \model\ to decouple the propagation from feature transformation, reducing the risks of over-smoothing and overfitting. 
Moreover, random propagation acts as an efficient method for graph data augmentation. 
Based on this, we propose the consistency regularization for \model\ by leveraging the distributional consistency of unlabeled nodes in multiple augmentations, improving the generalization capability of the model. 
Extensive experiments on graph benchmark datasets suggest that \model\ significantly outperforms state-of-the-art GNN baselines on semi-supervised graph learning tasks. 
Finally, we show that \model\ mitigates the issues of over-smoothing and overfitting, and its performance is married with robustness. }

 %\yd{to prof tang: shall we keep these references? debating}\jie{to me it is useful}

\hide{
Graph neural networks (GNNs) have generalized deep learning methods into graph-structured data, with promising performance on many semi-supervised graph tasks. 
However, existing GNNs often meet complex graph structures with scarce labeled nodes in practice, and suffer from the limitation of 
non-robustness~\cite{zugner2018adversarial,ZhuRobust}, over-smoothing~\cite{chen2019measuring,li2018deeper}, and overfitting~\cite{goodfellow2016deep,Lingxiao2020PairNorm}. 
%robustness, over-smoothing and overfitting. 
%and suffer from the limitation of expressive capacity, robustness and generalization. 
To address these problems, we propose the Graph Random Network (\model), 
a highly effective and robust graph based semi-supervised learning method.
%a general graph based semi-supervised learning framework. 
Different from the deterministic propagation process in existing GNNs~\cite{gori2005new, bruna2013spectral, gilmer2017neural}, \model\ adopts a random propagation strategy to enhance model robustness. \model\ also decouples the propagation from feature transformation, reducing the risks of over-smoothing and overfitting. 
 %explores the input graph in a random way to generate exponentially many data augmentations, leading to the implicit ensemble training style.
 What's more, random propagation acts as an efficient method for graph  data augmentation.  
 \model\ improves the generalization capacity by leveraging the distributional consistency of unlabeled nodes in multiple augmentations. 
 %mitigates the reliance on labeled nodes by leveraging the distributional consistency of unlabeled nodes in multiple samplings. 
%Extensive experiments on GNN benchmark datasets show the performance superiority of the proposed \model\ framework. 
%original GCNs with random modules, i.e., \textit{random propagation layer} and \textit{random consistency loss}, 
Essentially, \model\  significantly improves  the semi-supervised node classification accuracy compared with state-of-the-art GNN baselines.
%Essentially, \model\ can significantly improve  the graph classification accuracy of the simplest prediction model---MLP, outperforming state-of-the-art GNN baselines. 
%, which may inspire the rethinking of GCNs.

Graph convolutional networks (GCNs) have generalized deep learning methods into graph-structured data, with a promising performance improvement on many graph semi-supervised tasks. However, existing GCNs often meet complex graph structures with scarce labeled nodes in practice, and suffer from training insufficiency and even overfitting. To address this problem, we propose Graph Random Networks (\model), a general graph based semi-supervised learning paradigm. The basic idea is to train a GCN in exponentially many random subgraphs to disentangle the co-dependency between node representation and aggregation nodes' features and thus make \model\ expressive and robust for complex graphs. This offers an economic graph data augmentation. \model\ further mitigates the reliance on labeled nodes, leveraging the distributional consistency of unlabeled nodes in random subgraphs. Extensive experiments on different types of graph benchmarks show that in \model\ framework, original GCNs with random modules, i.e., \textit{random propagation layer} and \textit{random consistency loss},  can significantly improve the graph classification accuracy when compared to state-of-the-art baselines.
}


%\ccsdesc[300]{Computing methodologies~Regularization}


%\footnotetext[*]{Equal Contribution}
%\footnotetext[\S]{Corresponding Author}
%
% Keywords. The author(s) should pick words that accurately describe the work being
% presented. Separate the keywords with commas.
% \keywords{graph convolutional network, network representation learning}

% This command processes the author and affiliation and title information and builds
% the first part of the formatted document.

\maketitle

\section{Introduction}

Graph structure is a commonplace of both our physical and virtual worlds, such as social relationships, chemical bonds, and information diffusion. 
%Graph data is a commonplace on the Web where information is naturally interconnected as in social networks or online academic publications~\cite{giles1998citeseer,tang2008arnetminer}.
The inherit incompleteness of the real-world graph data sparks enormous interests in the problem of semi-supervised learning on graphs~\cite{zhu2003semi,kipf2016semi}.
To date, graph neural networks (GNNs) have been considered by many as the \textit{de facto} way to address this problem~\cite{kipf2016semi,Velickovic:17GAT,wu2019simplifying,ding2018semi,abu2019mixhop}. %\dong{add 4+ refs}
Briefly, GNNs leverage the graph structure among data samples to facilitate model predictions, enabling them to  produce prominent performance improvements over traditional semi-supervised learning methods~\cite{zhou2004learning}.

However, there are remaining challenges for GNN-based semi-supervised learning solutions. 
%often face challenges in real-world applications.  
Notably, the generalization of GNNs usually does not form their strengths, 
%Different from traditional semi-supervised methods, 
as most of them only use a supervised loss to learn parameters~\cite{kipf2016semi,Velickovic:17GAT,wu2019simplifying,chen2020simple}. %\dong{4 refs} 
This setup makes the model prone to overfit the limited labeled samples, thereby degrading the prediction performance over unseen samples. 
%Second, it is relatively difficult for them to handle large-scale graphs. The full-batch GNNs are commonly trained with an expensive recursive feature propagation procedure, inducing enormous time and memory overhead when processing large graphs. 
To overcome this issue, the graph random neural network (\grand)~\cite{feng2020grand} designs graph data augmentation and consistency regularization strategies for GNNs.  
These designs enable it to bring significant performance gains over existing GNNs for semi-supervised learning on Cora, Citeseer and Pubmed. 

Specifically, \grand develops the random propagation operation to generate effective structural data augmentations. 
It is then trained with both the supervised loss on labeled nodes and the consistency regularization loss  on different augmentations of unlabeled nodes. 
To achieve a good graph augmentation, random propagation in \grand proposes to use a mixed-order adjacency matrix %, i.e., the power iteration of the adjacency matrix, 
to propagate the feature matrix. 
The propagation essentially requires the power iteration of the adjacency matrix at every training step, making it computationally challenging to scale \grand to large-scale graphs. 

%Indeed, methods like \grand~\cite{feng2020grand} improve predictive performance, e.g., by data augmentation, 
%and achieves state-of-the-art accuracy on benchmarks, 
%while it does not scale to large graphs, e.g., due to power iterations for feature propagation at every training step. 

%On the other hand, existing scalable GNNs (e.g., GraphSAGE~\cite{hamilton2017inductive} and FastGCN~\cite{FastGCN}) mainly focus on reducing computational complexity with sampling methods, while with little considerations on predictive performance. How to efficiently and accurately perform GSSL on large graphs still remains an open challenge.

To address this issue, we present the \model framework for large-scale semi-supervised learning on graphs. 
\model is a scalable GNN consistency regularization method. 
In \model, we introduce efficient approximation techniques to perform random propagation in a mini-batch manner, addressing the scalability limitation of \grand. 
Furthermore, we improve \grand by adopting a confidence-aware loss for regulating the consistency between different graph data augmentations. 
This design stabilizes the training process and provides \model with good generalization. 
Specifically, \model comprises the following techniques: 

\begin{itemize}[topsep=3pt,parsep=0pt,partopsep=0pt,itemsep=0pt,leftmargin=*]

\item \textit{Generalized feature propagation}: 
We propose a generalized mixed-order matrix to perform random feature propagation. 
Such matrix offers a set of tunable weights to control the importance of different orders of neighborhoods and thus offers a flexible mechanism for dealing with complex real-world graphs. %, that is, this design furthers \grand with flexibility superiority. 

%To handle large graphs in the real world, we develop

\item \textit{Efficient approximation}: Inspired by recent matrix approximation based GNNs~\cite{chen2020scalable,bojchevski2020scaling}, \model adopts an approximation method---{Generalized Forward Push} ({GFPush})---to efficiently calculate the generalized propagation matrix. 
This enables \model to perform random propagation and model learning in a mini-batch manner, offering the model with significant scalability. 
%We show that both the time and space complexities of GFPush are irrelevant to the graph size, making \model scalable to large graphs.  

%redesigns the consistency loss used in \grand\ 

\item \textit{Confidence-aware loss}: We design a confidence-aware loss for the \model regularization framework. This helps filter out potential noises during the consistency training by ignoring highly uncertain unlabeled samples, thus improving the generalization performance of \model. 
\end{itemize}

%Note the generalized feature propagation gives \model superiority over \grand in terms of flexibility, approximation methods---in terms of scalability, and confidence-aware loss---in terms of generalization performance.\looseness=-1

We conduct comprehensive experiments on seven public graph datasets with different genres and scales to demonstrate the performance of \model. 
Overall, \model yields the best classification results compared to ten GNN baselines on three benchmark datasets and surpasses five representative scalable GNNs on the other four relatively large datasets with  efficiency benefits. 
For example, \model\ achieves a state-of-the-art accuracy of 85.0\% on Pubmed. % which is a significant improvement over \grand (2.3\%). % and 6 GNNs (4.4\%-7.4\%)
On MAG-Scholar-C with 12.4 million nodes, \model\ is about 10 fold faster than FastGCN and GraphSAINT, and offers a 4.9\% accuracy gain over PPRGo---previously the fastest method on this dataset---with a comparable running time. 
%, with a comparable run time. 

%\vpara{Orgnization.} In Section~\ref{sec:preliminaries} we define the GSSL problem and review related work. In Section~\ref{sec:method} we present our solution \model. In Section~\ref{sec:exp} we evaluate \model. In Section~\ref{sec:conclusion} we conclude.  

\hide{%---------------------------------------------------------------------
\section{Introduction}

Graph data is a commonplace on the Web where information is naturally interconnected as in social networks or online academic publications~\cite{giles1998citeseer,tang2008arnetminer}.
The inherit incompleteness of Web data spark interest in the problem of graph-based semi-supervised learning (GSSL)~\cite{zhu2003semi,kipf2016semi}.
Graph neural networks (GNNs) are considered by many as the best way to address GSSL since they employ the intrinsic graph structure among data samples to facilitate model prediction and they have prominent performance improvements over traditional methods~\cite{zhou2004learning}.

However, existing GNN-based semi-supervised learning methods often face two challenges in real-world applications. First, they lack \textit{generalization performance}: Conventional GNNs only use supervised loss to learn parameters, making the model prone to overfit the limited labeled samples, thereby degrading the prediction performance over unseen samples. Moreover, they have a weak \textit{scalability}: Most GNNs adopt full-batch training methods with an expensive recursive feature propagation procedure, inducing enormous time and memory overhead when processing large graphs. 

Recently, graph random neural network (\grand)~\cite{feng2020grand} has brought significant performance gains on GSSL benchmarks by using graph data augmentation and consistency regularization, 
%Indeed, methods like \grand~\cite{feng2020grand} improve predictive performance, e.g., by data augmentation, 
%and achieves state-of-the-art accuracy on benchmarks, 
while it does not scale to large graphs, e.g., due to power iterations for feature propagation at every training step. On the other hand, existing scalable GNNs (e.g.,
GraphSAGE~\cite{hamilton2017inductive} and FastGCN~\cite{FastGCN}) 
%GraphSAINT~\cite{zeng2020graphsaint} 
mainly focus on reducing computational complexity with sampling methods, while with little considerations on predictive performance. How to efficiently and accurately perform GSSL on large graphs still remains an open challenge.
%As identified by previous works~\cite{chen2018stochastic}, these methods often attain worse performance than full-batch GNNs due to the sampling variance. 
%Indeed, as we have witnessed in our evaluation (see details in Section~\ref{sec:exp},  Table~\ref{tab:small_graph}) %Figure~\ref{fig:intro} (a)),
%the prediction performance of all these methods on, for example, Cora, Citeseer and Pubmed does not exceed 83.9, 72.9, and 80.6 respectively, which is lower than many full-batch methods. 
%
%, as depicted in Figure~\ref{fig:intro} (a).
%This raises a research question: \textit{Can a GNN framework be with both good scalability and generalization capability for GSSL?}

In order to fulfill this gap, %achieve both good generalization performance and scalability for GSSL
we propose \model in this paper. \model does the former analogously to \grand, while significantly extends it in order to achieve scalability. %, thus giving us a positive answer to the research question above.
In particular, \grand adopts a random propagation strategy to generate multiple graph data augmentations at each epoch, and uses an additional consistency regularization loss to enforce neural network model to give similar predictions among multiple augmentations of unlabeled data.
While \model extends and improves \grand in three dimensions:

\begin{itemize}[topsep=3pt,parsep=0pt,partopsep=0pt,itemsep=0pt,leftmargin=*]

\item \textit{Generalized Feature Propagation}: \model relies on a generalized mixed-order matrix to perform random feature propagation. Such matrix offers a set of tunable weights to control the importances of different orders of neighborhoods and thus offers a flexible mechanism for dealing with complex real-world graphs. 

%To handle large graphs in the real world, we develop

\item \textit{Approximation Method}: Inspired by recent matrix approximation based GNNs~\cite{chen2020scalable,bojchevski2020scaling}, \model adopts a novel approximation method---\textit{Generalized Forward Push} (\textit{GFPush})---to efficiently calculate generalized propagation matrices. 
This enables \model to perform random propagation and model learning in a mini-batch manner. We theoretically show that such approach allows to deal with large graphs by studying the complexity of GFPush. 

%redesigns the consistency loss used in \grand\ 

\item \textit{Confidence-aware Loss}: \model relies on a novel confidence-aware loss, which filters out potential noises during consistency training by ignoring highly uncertain unlabeled samples.
\end{itemize}
Note the generalized feature propagation gives \model superiority over \grand in terms of flexibility, approximation methods---in terms of scalability, and confidence-aware loss---in terms of generalization performance.\looseness=-1

We conduct comprehensive experiments on 7 graph datasets with different genres and scales to demonstrate the benefits of \model. Overall, \model achieves the best classification performance over 10 GNNs on 3 benchmark datasets and surpasses 5 representative scalable GNNs on the other 4 relatively large datasets with a good efficiency. 
Remarkably, \model\ achieves a new state-of-the-art accuracy of 85.0\% on Pubmed which is a significant improvement over \grand (2.3\%). % and 6 GNNs (4.4\%-7.4\%)
On MAG-Scholar-C with 12.4 million nodes, \model\ is about 10 fold faster than FastGCN and GraphSAINT, and achieves 4.9\% improvement in accuracy over PPRGo, the previous fastest method on this data, with a comparable running time. 
%, with a comparable run time. 

\vpara{Orgnization.} In Section~\ref{sec:preliminaries} we define the GSSL problem and review related work. In Section~\ref{sec:method} we present our solution \model. In Section~\ref{sec:exp} we evaluate \model. In Section~\ref{sec:conclusion} we conclude.  

}%-------------------------------------------------------

\hide{
%Evgeny: this is the version that I replaced

The target of semi-supervised learning (SSL)~\cite{zhu2005semi} 
%serves as a label efficient learning paradigm with the target of 
is to leverage massive unlabeled data to improve model's generalization performance.
%As a label efficient learning paradigm, semi-supervised learning attracts tremendous attention from both academic and industry with the target of leveraging massive unlabeled data to improve model's generalization capacity~\cite{zhu2005semi}.
%Semi-supervised learning aims to leverage
%It has attracted tremendous attention from both academic and industry as its superiority in saving label acquisition costs.
As an important branch of SSL,
graph-based semi-supervised learning (GSSL)~\cite{zhu2003semi,kipf2016semi} employs the intrinsic graph structure among data samples to facilitate model prediction.
Recently, graph neural networks (GNNs) are considered as recent breakthroughs for GSSL because of their prominent performance improvements over traditional methods~\cite{zhou2004learning}.
However, existing GNN-based semi-supervised learning methods always face two challenges in realistic application. The first one is the lack of \textit{generalization performance}. Conventional GNNs only use supervised loss to learn parameters, making the model prone to overfit the limited labeled samples, thereby degrading the prediction performance over unseen samples. The other one is weak \textit{scalability}, most of GNNs adopt full-batch training method with an expensive recursive feature propagation procedure, inducing enormous time and memory overhead when processing large graphs.

Though many efforts are devoted to addressing the two limitations, they always tackle one of the two problem while ignoring the other one. 
To address the limitation of generalization, 
%As a typical example, 
Feng et al.~\cite{feng2020grand} recently propose Graph Random Neural Network (\grand), in which they adopt a random propagation strategy to generate multiple graph data augmentations at each epoch, and use an additional consistency regularization loss to enforce neural network model to give similar predictions among multiple augmentations of unlabeled data.
%In random propagation, feature vectors are randomly dropped and then propagated over graph with a mixed-order adjacency matrix. With this strategy, model could generate multiple different data augmentations from original features for each node. 
%After that the augmented features are fed into an MLP for predictions. In the training period, besides supervised classification loss, \grand\ adopts an additional consistency loss to enforce model to give similar predictions among multiple augmentations of unlabeled data. 
In this way, \grand\ could fully leverage the unlabeled samples to improve performance, and achieves state-of-the-art accuracy on several benchmarks. However, it is hard to scale to large graphs as it requires to perform time-consuming power iterations for random propagation at every training step. On the other hand,  existing scalable GNNs only focus on reducing computational complexity with little considerations on prediction performance. As shown in Figure~\ref{fig:intro} (a), typical scalable GNNs---including GraphSAGE~\cite{hamilton2017inductive}, FastGCN~\cite{FastGCN}, GraphSAINT~\cite{zeng2019graphsaint}, SGC~\cite{wu2019simplifying},  PPRGo~\cite{bojchevski2020scaling} and GBP~\cite{chen2020scalable}---have much lower classification accuracy than advanced full-batch GNNs, e.g., GRAND~\cite{feng2020grand}. This raises an interesting question: \textit{Is it possible to design a GNN framework with both good scalability and generalization capability for GSSL?}

Aiming at answering this question, based on \grand, we propose \model\ in this paper. \model\ inherits the basic idea of \grand, while improves it from three aspects: 1) \textit{Flexibility}. In \model, we use a generalized mixed-order matrix to perform random feature propagation. This form of matrix offers a set of tunable weights to control the importances of different orders of neighborhoods, which is more flexible for dealing with various complex graphs in real world. 2) \textit{Scalability}. To handle large graphs in the real world, we develop a novel approximation method---\textit{Generalized Forward Push} (\textit{GFPush})---to efficiently calculate the generalized propagation matrix. %Before training, \model\ firstly pre-computes a sub-matrix with GFPush, and then utilizes the obtained sub-matrix 
With the matrix approximation, \model\ is enabled
to perform random propagation and model learning in a mini-batch manner. We theoretically show this paradigm can deal with large graphs with detailed complexity analyses for GFPush. 3) \textit{Better Generalization}. \model\ redesigns the consistency loss used in \grand\ to a confidence-aware loss, which filters out potential noises during consistency training by ignoring highly uncertain unlabeled samples.
%with high uncertainty and adopts dynamic loss weight scheduling strategy to facilitate model convergence. 
This advancement is shown to be effective in improving model's generalization performance.

We conduct comprehensive experiments on \textit{seven} graph datasets with different genres and scales to demonstrate the performance of \model. Overall, \model\ achieves the best classification performance over 15 GNNs on small datasets and surpasses six scalable GNNs on the other four relative large datasets with good efficiency. 
Remarkably, \model\ achieves a new state-of-the-art accuracy of 85.0\% on Pubmed dataset, which gets 2.3\% significant improvement over \grand\ and 4.4\%-7.4\% improvements over six typical scalable GNNs (Cf. Figure~\ref{fig:intro} (a)). In terms of efficiency, \model\ is 50$\times$ faster than \grand\ on AMiner-CS dataset that has 593,486 nodes and 6.2 million edges, 11$\times$ faster than FastGCN and 15$\times$ faster than GraphSAINT on MAG-Scholar-C dataset which has 1.2 million nodes and 173 million edges. Compared with the previous fastest model PPRGo on MAG-Scholar-C, \model\ has comparable running time while significantly outperforms it in accuracy with 3.6\% improvement (Cf. Figure~\ref{fig:intro} (b)).

\begin{figure}[t]
	\centering
	\mbox
	{
    	\hspace{-0.1in}
		%\hfill
		\begin{subfigure}[Classification Accuracy on Pubmed.]{
				\centering
				\includegraphics[width = 0.5 \linewidth]{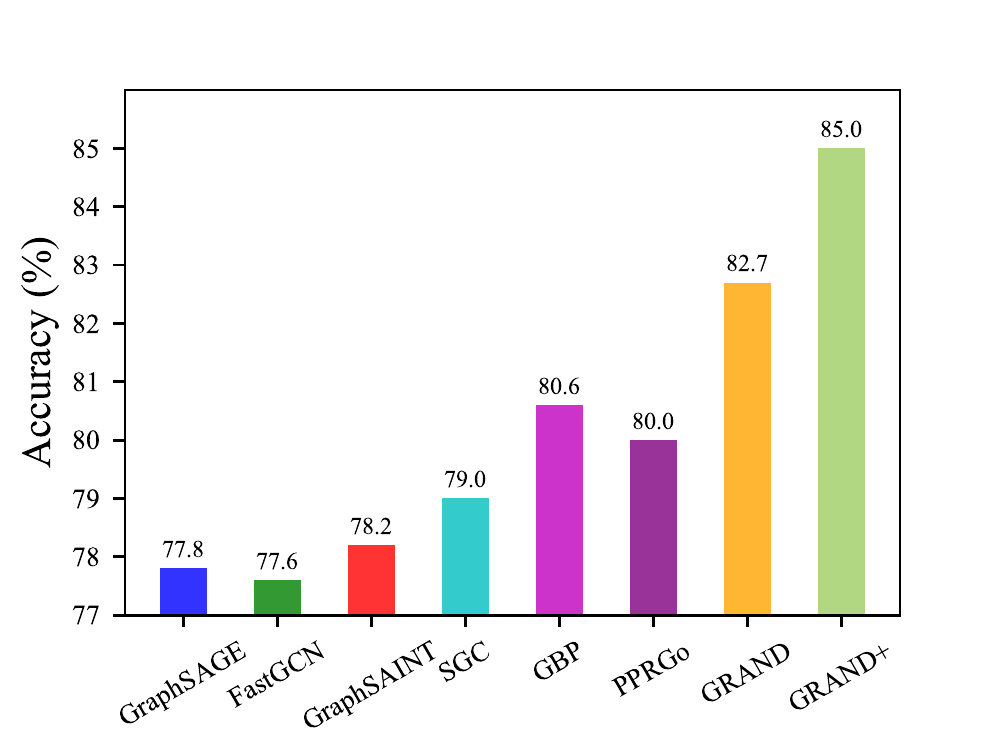}
			}
		\end{subfigure}
		\hspace{-0.1in}
		\begin{subfigure}[Results on MAG-Scholar-C.]{
				\centering
				\includegraphics[width = 0.5 \linewidth]{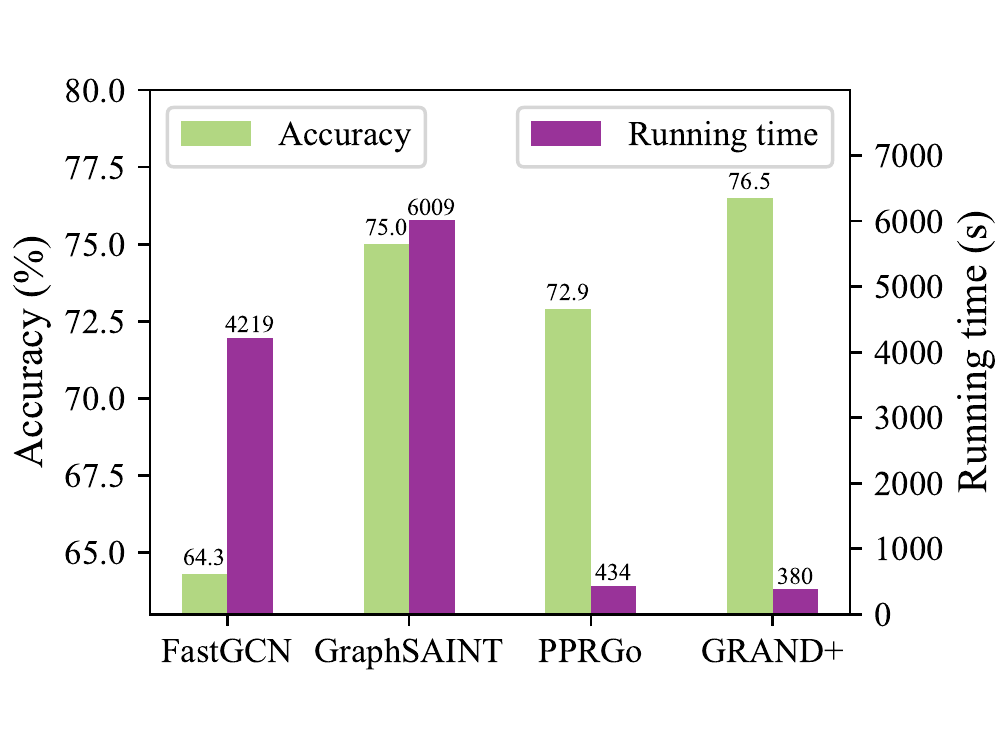}
			}
		\end{subfigure}
	}
	\vspace{-0.15in}
	\caption{Performance and Efficiency of \model.}
	\label{fig:intro}
	\vspace{-0.15in}
\end{figure}
}

\hide{
\begin{itemize}
    \item \textit{Flexibility}. In \model, we use a generalized mixed-order matrix to perform random feature propagation. This form of matrix offers a set of tunable weights to control the importances of different orders of neighborhoods, which is more flexible for dealing with various complex graphs in real world.
    %that could  (sub-)matrix approximation method---RwPush to enable the random propagation to be applied on large graphs. RwPush is insp 
    \item \textit{Scalability}. To enable the model to handle large graphs, we develop an efficient approximation method---\textit{Generalized Forward Push} (\textit{GFPush})---for the generalized mixed-order matrix. Before training, \model\ firstly pre-computes an approximation sub-matrix with GFPush, and then utilizes the obtained sub-matrix to perform random propagation and model learning in a mini-batch manner. We also theoretically prove the efficiency and effectiveness of GFPush with a detailed analyses for  computational complexity and approximation error. 
    \item \textit{Effectiveness}. 
    %Different from the consistency loss used in \grand, 
    \model\ redesigns the consistency loss used in \grand\ to a confidence-aware loss, which filters out potential noises during training by ignoring unlabeled samples with high uncertainty. We also adopts dynamic loss weight scheduling to facilitate model convergence. These advances are shown to be effective in improving model's generalization performance.
    \end{itemize}
}
%employs a proposed matrix approximation method---RwPush---to make it able to handle large graphs. 

%enable it to scale up to large graphs with a proposed matrix approximation method---RwPush. Specifically.
%---a scalable semi-supervised GNN framework with strong generalization capability. 
%avoiding the limitation of 

%large performance gap between their  . Compared with the two state-of-the-art full-batch GNNs, GRAND and GCNII, all the existing scalable GNNs have inferior classification results on the three benchmark datasets (Cf. Table~\ref{tab:small_graph}). 

%Comparing with full-batch GNN methods, scalable GNNs are all inferior than state-of-the-art full batch 
%both supervised loss and consistency loss among multiple predictions of unlabeled data augmentations are used for optimization.

%with learning hidden 

%Due to its effectiveness in practice, GSSL has been widely applied in many real applications, such as social networks~\cite{}, academic graph~\cite{akujuobi2020recurrent} and E-commerce~\cite{chen2019semi}.

%build models with very few labeled data and a huge amount of unlabeled data, which  applications as its superiority in saving label acquisition costs~\cite{zhu2005semi}. As a typical SSL paradigm, graph-based semi-supervised learning (GSSL) models the pairwise relationships between data samples by leveraging a intrinsic graph structure. 

% value in saving % in practical applications where labels are hard to acquire~\cite{}.

\hide{
Graph-structured data is ubiquitous and exists in many real applications, such as social networks, the World Wide Web, and the knowledge graphs. 
With the remarkable success of deep learning~\cite{lecun1995convolutional}, 
graph neural networks (GNNs)~\cite{bruna2013spectral,henaff2015deep, defferrard2016convolutional,kipf2016semi,duvenaud2015convolutional,niepert2016learning, monti2017geometric}  generalize neural networks to graphs and offer powerful graph representations for a wide variety of applications, such as node
classification~\cite{kipf2016semi}, link prediction~\cite{zhang2018link}, graph classification~\cite{ying2018hierarchical} and knowledge graph completion~\cite{schlichtkrull2018modeling}.

As a pioneer work of GNNs, Graph Convolutional Network (GCN)~\cite{kipf2016semi} first formulate the convolution operation on a graph as feature propagation in neighborhoods, and achieves promising performance on semi-supervised learning over graphs. 
They model a node's neighborhood as a receptive field and enable a recursive neighborhood propagation process by stacking multiple graph convolutional layers, and thus information from multi-hop neighborhoods is utilized. 
After that, much more advanced GNNs have been proposed by designing more complicated propagation patterns, e.g., the multi-head self-attention in GAT~\cite{Velickovic:17GAT}, differentiable cluster in DIFFPOOL~\cite{ying2018hierarchical} and Neighborhood Mixing~\cite{abu2019mixhop}. Despite strong successes achieved by these models, existing GNNs still suffer from some inherent issues:
}
\section{Semi-Supervised Graph Learning}
\label{sec:preliminaries}
%\subsection{Semi-Supervised Graph Learning}
\subsection{Problem}

In graph-based semi-supervised learning, the data samples are organized as a graph $G=(V, E)$, where each node $v \in V$ represents a data sample and $E \in V \times V$ is a set of edges that denote the relationships between each pair of nodes. 
We use $\mathbf{A}\in \{0,1\}^{|V| \times |V|}$ to represent $G$'s adjacency matrix, with each element $\mathbf{A}(s,v) = 1$ indicating that there exists an edge between $s$ and $v$, otherwise $\mathbf{A}(s,v) = 0$. 
$\mathbf{D}$ is the diagonal degree matrix where $\mathbf{D}(s,s) = \sum_v \mathbf{A}(s,v)$. 
$\widetilde{G}$ is used to denote the graph $G$ with added self-loop connections. 
The corresponding adjacency matrix is $\widetilde{\mathbf{A}} = \mathbf{A} + \mathbf{I}$ and the degree matrix is $\widetilde{\mathbf{D}}= \mathbf{D} + \mathbf{I}$.
%\vpara{Semi-Supervised Node Classification.} 

In this work, we focus on the classification problem, in which each sample $s$ is associated with 
1) a feature vector $\mathbf{X}_s \in \mathbf{X} \in \mathbb{R}^{|V|\times d_f}$ 
and 
2) a label vector $\mathbf{Y}_s \in \mathbf{Y} \in \{0,1\}^{|V| \times C}$ with $C$ representing the number of classes. 
In the semi-supervised setting, only limited nodes $L\in V$ have observed labels ($ 0 < |L| \ll |V|$), and the labels of remaining nodes $U = V - L$ are unseen.
The objective of semi-supervised graph learning is to infer the missing labels $\mathbf{Y}_{U}$ for unlabeled nodes $U$ based on graph structure $G$, node features $\mathbf{X}$, and the observed labels $\mathbf{Y}_{L}$\footnote{For a matrix $\mathbf{M}\in \mathbb{R}^{a\times b}$, we use $\mathbf{M}_i\in\mathbb{R}^b$ to denote its $i$-th row vector and let $\mathbf{M}(i,j)$ represent the element of the $i$-th row and the $j$-th column.}.
%to infer the missing labels $\mathbf{Y}^U$ for unlabeled nodes. 

\subsection{Related Work}
\label{sec:related}

Graph neural networks (GNNs) have been widely adopted for addressing the semi-supervised graph learning problem. 
In this part, we review the progress of GNNs with an emphasis on their large-scale solutions to semi-supervised graph learning.

\vpara{Graph Convolutional Network.}
%Graph neural networks (GNNs)~\cite{gori2005new, scarselli2009graph, kipf2016semi} generalize neural techniques into graph-structured data. In general, GNNs update node representations with two core operations: feature propagation with a predefined proximity matrix and non-linear transformation.
The graph convolutional network (GCN)~\cite{kipf2016semi} generalizes the convolution operation into graphs. 
Specifically, the $l$-th GCN layer is defined as:
%adopts the following propagation rule:
\begin{equation}
\small
\label{equ:gcn}
\mathbf{H}^{(l+1)} = \sigma (\hat{\mathbf{A}}\mathbf{H}^{(l)}\mathbf{W}^{(l)}),
\end{equation}
where $\mathbf{H}^{(l)}$ denotes the hidden node representations of the $l$-th layer with $\mathbf{H}^{(0)}=\mathbf{X}$, $\hat{\mathbf{A}}=\widetilde{\mathbf{D}}^{-\frac{1}{2}}\widetilde{\mathbf{A}}\widetilde{\mathbf{D}}^{-\frac{1}{2}}$ is the symmetric normalized adjacency matrix of $\widetilde{G}$, $\mathbf{W}^{(l)}$ denotes the weight matrix of the $l$-th layer, and $\sigma(\cdot)$ denotes the activation function. 
In practice, this graph convolution procedure would be repeated multiple times, and the final representations are usually fed into a logistic regression layer for classification.

\vpara{Simplified Graph Convolution.} 
By taking a closer look at Equation~\ref{equ:gcn}, we can observe that graph convolution consists of two operations: \textit{feature propagation} $\hat{\mathbf{A}} \mathbf{H}^{(l)}$ and \textit{non-linear transformation} $\sigma(\cdot)$. %To improve model's efficiency, 
%And the feature propagation procedure $\hat{\mathbf{A}} \mathbf{H}^{(l)}$ can be explained as low-pass spectral filtering over $\mathbf{H}^{(l)}$.
Wu et al.~\cite{wu2019simplifying} simplify this procedure %propose simplified graph convolution (SGC) %, which simplifies graph convolution (Cf. Equation~\ref{equ:gcn})
by removing the non-linear transformations in hidden layers. 
The resulting simplified graph convolution (SGC) is formulated as:
%performs feature propagation on feature matrix with $T-$order symmetric normalized adjacency matrix, followed by a logistic regression classifier. The resulting predicting model is:
%the feature propagation and feature transformation, and generates predictions via:
\begin{equation}
\small
    \hat{\mathbf{Y}} = \text{softmax}(\hat{\mathbf{A}}^N\mathbf{X}\mathbf{W}),
\end{equation}
where $\hat{\mathbf{A}}^N\mathbf{X}$ is considered as a simplified $N$-layer graph convolutions %(Cf. Equation~\ref{equ:gcn}) 
on $\mathbf{X}$, $\mathbf{W}$ refers to the learnable weight matrix for classification,
 and $\hat{\mathbf{Y}}$ denotes the model's predictions.

%Where features are first propagated with
%can be explained as a special form of laplacian smoothing over hidden representations~\cite{li2018deeper}. This indicates that stacking many layers of GCN will result in over-smoothing problem, i.e., node representations of different classes become indistinguishable with two much focus on high-order neighborhoods.

%which hinders the model aggregating information from high order neighborhoods. 

\vpara{GNNs with Mixed-Order Propagation.} As pointed by Li et al.~\cite{li2018deeper}, $\hat{\mathbf{A}}^{N}\mathbf{X}$ will converge to a fix point as $N$ increases %\rightarrow +\infty$
according to the Markov chain convergence theorem, namely, the over-smoothing issue. 
%, which indicates that directly increasing propagation order will make GCN and SGC lose expressive ability by ignoring local information.
To address it, a typical kind of methods~\cite{bojchevski2020scaling,klicpera2018predict,klicpera2019diffusion} suggest to use a more complex mixed-order matrix for feature propagation. 
For example, APPNP~\cite{klicpera2018predict} adopts the truncated personalized PageRank (ppr) matrix $\mathbf{\Pi}^{\text{ppr}}_{\text{sym}} = \sum_{n=0}^{N}\alpha (1 - \alpha)^n\hat{\mathbf{A}}^{n}$, where the hyperparameter $\alpha \in (0, 1]$ denotes the teleport probability, 
%for multi-order propagation matrix
allowing the model to preserve the local information even when $N \rightarrow +\infty$. 
%In addition, APPNP  decouples feature transformation and propagation as done in SGC. 
\hide{
In particular, it first uses a multi-layer perceptron (MLP) to perform prediction, and then propagates the prediction logits using $\mathbf{\Pi}_{\text{sym}}^{\text{ppr}}$ to get the final predictions:
\begin{equation}
\small
\label{equ:appnp}
    \hat{\mathbf{Y}} = \text{softmax}(\mathbf{\Pi}^{\text{ppr}}_{\text{sym}}\cdot \text{MLP}(\mathbf{X})).
\end{equation}

\dong{possible to add a new more related work here?}
}

\vpara{Scalable GNNs.} %In full-batch GNNs, the recursive neighborhood expansion process causes severe time and memory challenges for learning on giant graphs. 
%To address this challenge, there are three lines of ideas to scale GNNs.
Broadly, there are three categories of methods proposed for making GNNs scalable:
%In order to make GNNs scalable, three categories of methods have been proposed recently: 
1) The {node sampling methods} employ sampling strategies to speed up the recursive feature aggregation procedure. 
The representative methods include GraphSAGE~\cite{hamilton2017inductive}, FastGCN~\cite{FastGCN}, and LADIES~\cite{zou2019layer}; 
2) The {graph partition methods} attempt to divide the original large graph into several small sub-graphs and run GNNs on sub-graphs. 
This category consists of Cluster-GCN~\cite{Chiang2019ClusterGCN} and GraphSAINT~\cite{zeng2020graphsaint}; 
3) The {matrix approximation methods} follow the design of SGC~\cite{wu2019simplifying} to decouple  feature propagation and non-linear transformation, and to utilize some approximation methods to accelerate feature propagation. 
The proposed \model framework is highly related to matrix approximation based methods such as PPRGo~\cite{bojchevski2020scaling} and GBP~\cite{chen2020scalable}. 
We will analyze their differences in Section~\ref{sec:model_analysis}.

\hide{%=====================

\section{Preliminary}
\label{sec:preliminaries}
\subsection{Problem Definition}
In graph-based semi-supervised learning, the data samples are organized as a graph $G=(V, E)$, where each node $v \in V$ refers to a data sample, and $E \in V \times V$ is a set of edges denoting relationships of node pairs. We use 
$\mathbf{A}\in \{0,1\}^{|V| \times |V|}$ to represent $G$'s adjacency matrix, with each element $\mathbf{A}(s,v) = 1$ indicating there exists an edge between $s$ and $v$, otherwise $\mathbf{A}(s,v) = 0$. $\mathbf{D}$ is diagonal degree matrix where $\mathbf{D}(s,s) = \sum_v \mathbf{A}(s,v)$. We further use $\widetilde{G}$ to denote graph $G$ with added self-loop connections. The corresponding adjacency matrix is $\widetilde{\mathbf{A}} = \mathbf{A} + \mathbf{I}$ and degree matrix is $\widetilde{\mathbf{D}}= \mathbf{D} + \mathbf{I}$.
%\vpara{Semi-Supervised Node Classification.} 

This work focuses on classification problem, in which each sample $s$ is associated with 
1) a feature vector $\mathbf{X}_s \in \mathbf{X} \in \mathbb{R}^{|V|\times d_f}$ 
and 
2) a label vector $\mathbf{Y}_s \in \mathbf{Y} \in \{0,1\}^{|V| \times C}$ with $C$ representing the number of classes. In the semi-supervised setting, only limited nodes $L\in V$ have observed labels ($ 0 < |L| \ll |V|$), and the labels of remaining nodes $U = V - L$ are unseen.
The objective of graph-based semi-supervised learning is to infer the missing labels $\mathbf{Y}_{U}$ for unlabeled nodes $U$ based on graph structure $G$, node features $\mathbf{X}$ and the observed labels $\mathbf{Y}_{L}$\footnote{For a matrix $\mathbf{M}\in \mathbb{R}^{a\times b}$, we use $\mathbf{M}_i\in\mathbb{R}^b$ to denote its $i$-th row vector and let $\mathbf{M}(i,j)$ represent the element of the $i$-th row and the $j$-th column.}.
%to infer the missing labels $\mathbf{Y}^U$ for unlabeled nodes. 

\subsection{Related Work}
\label{sec:related}
\vpara{Graph Convolutional Network.}
%Graph neural networks (GNNs)~\cite{gori2005new, scarselli2009graph, kipf2016semi} generalize neural techniques into graph-structured data. In general, GNNs update node representations with two core operations: feature propagation with a predefined proximity matrix and non-linear transformation.
Graph convolutional network (GCN)~\cite{kipf2016semi} generalizes convolution into graphs, where the $l$-th layer is defined as:
%adopts the following propagation rule:
\begin{equation}
\small
\label{equ:gcn}
\mathbf{H}^{(l+1)} = \sigma (\hat{\mathbf{A}}\mathbf{H}^{(l)}\mathbf{W}^{(l)}),
\end{equation}
where $\mathbf{H}^{(l)}$ denotes hidden node representations of the $l$-th layer 
with $\mathbf{H}^{(0)}=\mathbf{X}$, $\hat{\mathbf{A}}=\widetilde{\mathbf{D}}^{-\frac{1}{2}}\widetilde{\mathbf{A}}\widetilde{\mathbf{D}}^{-\frac{1}{2}}$ is the symmetric normalized adjacency matrix of $\widetilde{G}$, $\mathbf{W}^{(l)}$ denotes weight matrix of the $l$-th layer, and $\sigma(.)$ denotes activation function. In practice, this procedure would be repeated multiple times, the final representations are fed into a logistic regression layer for classification.

\vpara{Simplified Graph Convolution.} Taking a closer look at Equation~\ref{equ:gcn}, we could observe that graph convolution consists of two operations: \textit{feature propagation} $\hat{\mathbf{A}} \mathbf{H}^{(l)}$ and \textit{non-linear transformation} $\sigma(\cdot)$. %To improve model's efficiency, 
%And the feature propagation procedure $\hat{\mathbf{A}} \mathbf{H}^{(l)}$ can be explained as low-pass spectral filtering over $\mathbf{H}^{(l)}$.
Wu et al.~\cite{wu2019simplifying} simplify this procedure %propose simplified graph convolution (SGC) %, which simplifies graph convolution (Cf. Equation~\ref{equ:gcn})
by removing the non-linear transformations in hidden layers, the resulting simplified graph convolution (SGC) is formulated as:
%performs feature propagation on feature matrix with $T-$order symmetric normalized adjacency matrix, followed by a logistic regression classifier. The resulting predicting model is:
%the feature propagation and feature transformation, and generates predictions via:
\begin{equation}
\small
    \hat{\mathbf{Y}} = \text{softmax}(\hat{\mathbf{A}}^N\mathbf{X}\mathbf{W}),
\end{equation}
where $\hat{\mathbf{A}}^N\mathbf{X}$ is considered as a simplified $N-$layer graph convolutions (Cf. Equation~\ref{equ:gcn}) on $\mathbf{X}$, $\mathbf{W}$ refers to learnable weight matrix for classification and $\hat{\mathbf{Y}}$ denotes model's predictions.

%Where features are first propagated with
%can be explained as a special form of laplacian smoothing over hidden representations~\cite{li2018deeper}. This indicates that stacking many layers of GCN will result in over-smoothing problem, i.e., node representations of different classes become indistinguishable with two much focus on high-order neighborhoods.

%which hinders the model aggregating information from high order neighborhoods. 
\vpara{GNNs with Mixed-order Propagation.} As pointed by Li et al.~\cite{li2018deeper}, $\hat{\mathbf{A}}^{N}\mathbf{X}$ will converge to a fix point as $N$ increases %\rightarrow +\infty$
according to markov chain convergence theorem. This problem is called over-smoothing.
%, which indicates that directly increasing propagation order will make GCN and SGC lose expressive ability by ignoring local information.
To address it, a typical kind of methods~\cite{bojchevski2020scaling,klicpera2018predict,klicpera2019diffusion} suggest to use a more complex mixed-order matrix for feature propagation. For example, APPNP~\cite{klicpera2018predict} adopts truncated personalized pagerank (PPR) matrix: $\mathbf{\Pi}^{\text{ppr}}_{\text{sym}} = \sum_{n=0}^{N}\alpha (1 - \alpha)^n\hat{\mathbf{A}}^{n}$, here the hyperparameter $\alpha \in (0, 1]$ denotes the teleport probability 
%for multi-order propagation matrix
which allows model to preserve local information even when $N \rightarrow +\infty$. APPNP also decouples feature transformation and propagation as done in SGC. In particular, it first uses a Multi-layer Perceptron (MLP) to perform prediction, and then propagates the prediction logits using $\mathbf{\Pi}_{\text{sym}}^{\text{ppr}}$ to get final predictions:
\begin{equation}
\small
\label{equ:appnp}
    \hat{\mathbf{Y}} = \text{softmax}(\mathbf{\Pi}^{\text{ppr}}_{\text{sym}}\cdot \text{MLP}(\mathbf{X})).
\end{equation}
\vpara{Graph Random Neural Network.}
Recently, Feng et al.~\cite{feng2020grand} exploits the mixed-order propagation to achieve graph data augmentation. They develop \textit{random propagation}, in which node features $\mathbf{X}$ are first randomly dropped with a variant of dropout---DropNode. Then the resultant corrupted feature matrix is propagated over the graph with a mixed-order matrix. Instead of ppr matrix, they use an average pooling matrix %from order $0$ to $N$:
$\mathbf{\Pi}^{\text{avg}}_{\text{sym}} = \sum_{n=0}^{N} \hat{\mathbf{A}}^n/(N+1)$ for propagation.
%And the random propagation is defined as
%Instead of the propagation rule used in GCN, a mixed-order proximity matrix $\mathbf{\Pi}^{\text{mix}}= \sum_{k=0}^{K}\frac{1}{K+1}(\mathbf{D}^{-\frac{1}{2}}\mathbf{A}\mathbf{D}^{-\frac{1}{2}})^k$ is adopted to preserve both local and global information. 
More formally, the random propagation strategy is formulated as:
\begin{equation}
\small
\label{equ:randprop}
    \overline{\mathbf{X}} = \mathbf{\Pi}^{\text{avg}}_{\text{sym}} \cdot \text{diag}(\mathbf{z}) \cdot \mathbf{X}, \quad \mathbf{z}_i \sim \text{Bernoulli}(1 - \delta),
\end{equation}
where $\mathbf{z} \in \{0,1\}^{|V|}$ denotes random DropNode masks drawn from Bernoulli($1 - \delta$), $\delta$ represents DropNode probability.
%, the detail of \text{DropNode} is described in Algorithm~\ref{alg:dropnode}.
In doing so, the dropped information of each node is compensated by its neighborhoods. Under the homophily assumption of graph data, the resulting matrix $\overline{\mathbf{X}}$ can be seen as an effective data augmentation of the original feature matrix $\mathbf{X}$. Owing to the randomness of DropNode, this method could generate exponentially many augmentations for each node theoretically. 

Based on random propagation, the authors further propose graph random neural network (GRAND).
%, a simple and effective framework for semi-supervised learning on graphs.  
In each training step of \grand, the random propagation procedure is performed for $M$ times, leading to $M$ augmented feature matrices $\{\overline{\mathbf{X}}^{(m)}|1\leq m \leq M\}$. Then all the augmented feature matrices are fed into an MLP to get $M$ predictions. In optimization, \grand\ is trained with both standard classification loss on labeled data and an additional \textit{consistency regularization} loss~\cite{berthelot2019mixmatch} on the unlabeled node set $U$:
\begin{equation}
\label{equ:grand_consis}
\small
    \frac{1}{M\cdot|U|}\sum_{s\in U}\sum_{m=1}^M \norm{\hat{\mathbf{Y}}_s^{(m)} - \overline{\mathbf{Y}}_s}_2^2, \quad \overline{\mathbf{Y}}_s = \sum_{m=1}^M\frac{1}{M}\hat{\mathbf{Y}}_s^{(m)},%\footnote{For brevity, here we omit the sharpening trick adopted in \grand. For the complete formulation of GRAND's consistency loss, please refer to Equation 3 in \cite{feng2020grand}.}
\end{equation}
where $\hat{\mathbf{Y}}_s^{(m)}$ is MLP's prediction probability for node $s$ when using $\overline{\mathbf{X}}^{(m)}$ as input. The consistency loss provides an additional regularization effect by enforcing neural network to give similar predictions for different augmentations of unlabeled data.
%to minimize the discrepancy among $M$ predictions of each unlabeled node's augmentations
With random propagation and consistency regularization, \grand\ achieves better generalization capability over other conventional GNNs.

\vpara{Scalable GNNs.} %In full-batch GNNs, the recursive neighborhood expansion process causes severe time and memory challenges for learning on giant graphs. 
%To address this challenge, there are three lines of ideas to scale GNNs.
There have been three categories of methods proposed for making GNNs scalable:
%In order to make GNNs scalable, three categories of methods have been proposed recently: 
1) \textit{Node sampling methods} employ sampling strategy to speed up the recursive feature aggregation procedure, the representative methods include GraphSAGE~\cite{hamilton2017inductive}, FastGCN~\cite{FastGCN} and LADIES~\cite{zou2019layer}; 2) \textit{Graph partition methods} divide the original large graph into several small sub-graphs and run GNNs on sub-graphs, this category consists of Cluster-GCN~\cite{Chiang2019ClusterGCN} and GraphSAINT~\cite{zeng2020graphsaint}; 3) \textit{Matrix approximation methods} follow the design of SGC~\cite{wu2019simplifying} which decouples feature propagation and non-linear transformation, and utilize some approximation methods to accelerate feature propagation. \model is highly related to matrix approximation methods including PPRGo~\cite{bojchevski2020scaling} and GBP~\cite{chen2020scalable}, we will analyze their differences in Section~\ref{sec:model_analysis}.

}%==========================

\hide{
\begin{algorithm}[h]
\small
\caption{DropNode}
\label{alg:dropnode}
\begin{algorithmic}[1]
\REQUIRE ~~\\
Feature matrix $\mathbf{X} \in \mathbb{R}^{n \times d}$, DropNode probability 
$\delta \in (0,1)$. \\
\ENSURE ~~\\
Perturbed feature matrix  $\mathbf{\widetilde{X}}\in \mathbb{R}^{n \times d}$.
%\IF{mode == Inference}
%\STATE $\widetilde{\mathbf{X}} = \mathbf{X}$.
%\ELSE
\STATE Randomly sample $n$ masks: $\{\epsilon_i \sim Bernoulli(1-\delta)\}_{i=0}^{n-1}$.
 %$\mathcal{L}^l + \lambda \mathcal{L}^u$
%\STATE 
\STATE Obtain deformity feature matrix by  multiplying each node's feature vector with the corresponding  mask: $\widetilde{\mathbf{X}}_{i} = \epsilon_i \cdot \mathbf{X}_{i} $.
\STATE Scale the deformity features: $\widetilde{\mathbf{X}} = \frac{\widetilde{\mathbf{X}}}{1-\delta}$.
%\ENDIF
\end{algorithmic}
\end{algorithm}
}

\hide{
\vpara{Complexity Analysis for \grand.}  
As for Random Propagation (Cf. Equation~\ref{equ:randprop}), directly calculating the dense matrix $\mathbf{\Pi}^{\text{avg}}_{\text{sym}}$ is rather time-consuming. Hence the authors propose to calculate $\overline{\mathbf{X}}$ with power iteration, i.e., iteratively calculating and summing up the product of sparse matrix $\hat{\mathbf{A}}$ and $\hat{\mathbf{A}}^t \cdot \text{diag}(\mathbf{z}) \cdot {\mathbf{X}}$ ($0 \leq t \leq T-1$). This procedure has time complexity $\mathcal{O}(M\cdot(|V|+ |E|))$ and memory complexity $\mathcal{O}(|V| + |E|)$. This makes \grand\  cannot scale to large graphs for two reasons: 1) The random propagation process is conducted with the full batch of nodes and the time and memory complexity are positively correlated with the size of graph, which is costly when dealing with large graphs; 2) This time-consuming procedure needs be performed for $M$ times at each training step to obtain multiple augmented features with random DropNode masks. 
}

\hide{

\vpara{Graph Neural Networks.}
Graph neural networks (GNNs)~\cite{gori2005new, scarselli2009graph, kipf2016semi} generalize neural techniques into graph-structured data.
The core operation in GNNs is graph propagation, in which information is propagated from each node to its neighborhoods with some deterministic propagation rules. 
%which is usually performed by some deterministic propagation rules derived from the graph structures. 
For example, the graph convolutional network (GCN) \cite{kipf2016semi} adopts the following propagation rule:
\begin{equation}
\small
\label{equ:gcn_layer}
	\mathbf{H}^{(l+1)} = \sigma (\mathbf{\hat{A}}\mathbf{H}^{(l)}\mathbf{W}^{(l)}),
\end{equation}
\noindent where $\mathbf{\hat{A}}$ is the symmetric normalized adjacency matrix,  $\mathbf{W}^{(l)}$ is the weight matrix of the $l$-th layer, and $\sigma(.)$ denotes ReLU function. $\mathbf{H}^{(l)}$ is the hidden node representation in the $l$-th layer % with the hidden representation dimension $D^{(l)}$. 
with $\mathbf{H}^{(0)}=\mathbf{X}$. The propagation in Eq. \ref{equ:gcn_layer} could be explained via 1) an %first-order 
approximation of the spectral graph convolutional operations~\cite{bruna2013spectral,henaff2015deep, defferrard2016convolutional}, 2) neural message passing~\cite{gilmer2017neural}, 
and 3) convolutions on direct neighborhoods~\cite{monti2017geometric, hamilton2017inductive}. 
%Hence, in the rest of the paper, we will not distinguish between spectral or non-spectral convolutional neural networks, and between graph convolutions and neural message passing, strictly, for the  convenience of narration.
Recent attempts to advance this architecture include GAT~\cite{Velickovic:17GAT}, GMNN~\cite{qu2019gmnn}, MixHop~\cite{abu2019mixhop},  GraphNAS~\cite{gao2019graphnas}, and so on. 
Often, these models face the challenges of overfitting and over-smoothing due to the deterministic graph propagation process~\cite{li2018deeper,chen2019measuring,Lingxiao2020PairNorm}. 
%\yd{to add}
Differently, we propose  random propagation 
%process 
for GNNs, which decouples  feature propagation and non-linear transformation in Eq. \ref{equ:gcn_layer}, reducing the risk of over-smoothing and overfitting. 
Recent efforts have also been devoted to performing node sampling for fast and scalable GNN training, such as GraphSAGE~\cite{hamilton2017inductive}, FastGCN~\cite{FastGCN}, AS-GCN~\cite{huang2018adaptive}, and LADIES~\cite{ladies}. 
%To improve the efficiency and scalability of GNNs, efforts have been devoted to performing node sampling for fast and scalable GNN training, such as GraphSAGE~\cite{hamilton2017inductive}, FastGCN~\cite{FastGCN}, AS-GCN~\cite{huang2018adaptive}, and LADIES~\cite{ladies}. 
Different from these work, in this paper, a new sampling strategy DropNode, is proposed for improving the robustness and generalization of GNNs for semi-supervised learning.
Compared with GraphSAGE's node-wise sampling, DropNode 1) enables the decoupling of feature propagation and transformation, and 2) is more efficient as it does not require recursive sampling of neighborhoods for every node. Finally, it drops each node based an i.i.d. Bernoulli distribution, differing from the importance sampling in FastGCN, AS-GCN, and LADIES. 
%In Section \ref{sec:theory}, we also provide a theoretical explanation for DropNode from the perspective of regularization. 
%\yd{Modify the writing of this section, since DropNode has not been introduced.}

%\yd{need to add the difference between dropnode and sampling based gnn, such as graphsage} 

%\subsection
\vpara{Regularization Methods for GCNs.}
Broadly, a popular regularization method in deep learning is data augmentation, which expands the training samples by applying some transformations or injecting noise into input data \cite{devries2017dataset,goodfellow2016deep, vincent2008extracting}. 
%A general kind of augmentation method is injecting noises into the input data~\cite{goodfellow2016deep, vincent2008extracting}. 
Based on data augmentation, we can further leverage consistency regularization~\cite{berthelot2019mixmatch,laine2016temporal} for semi-supervised learning, which enforces the model to output the same distribution on different augmentations of an example.
%the same unlabeled example under different augmentations in semi-supervised learning.
%for semi-supervised learning. 
%Roughly speaking, it enforces the model to output the same class distribution for the same unlabeled example under different augmentations \cite{berthelot2019mixmatch,laine2016temporal}. 
 %VBAT~\cite{deng2019batch}, $\text{G}^3$NN~\cite{ma2019flexible},  GraphMix~\cite{verma2019graphmix} and Dropedge~\cite{YuDropedge}. 
Following this idea, a line of work has aimed to design powerful regularization methods for GNNs, such as 
VBAT~\cite{deng2019batch},
$\text{G}^3$NN~\cite{ma2019flexible}, 
GraphMix~\cite{verma2019graphmix}, 
and 
DropEdge~\cite{YuDropedge}. 
%VBAT~\cite{deng2019batch} applies virtual adversarial training~\cite{miyato2015distributional} into GCNs to encourage model invariant to adversarial attacks.  
%$\text{G}^3$NN~\cite{ma2019flexible} proposes a generative framework for semi-supervised learning on graphs, where utilize an additional link prediction task to perform regularization. 
For example, GraphMix~\cite{verma2019graphmix} introduces the MixUp~\cite{zhang2017mixup} for training GCNs. Different from \model, GraphMix augments graph data by performing linear interpolation between two samples in hidden space, and regularizes the GCNs by encouraging the model to predict the same interpolation of corresponding labels. 
DropEdge~\cite{YuDropedge} aims to militate over-smoothing by randomly removing some edges during training but does not bring significant performance gains for semi-supervised learning task. 
%Our 
%Our DropNode could be similar if we try to drop all out-coming edges of dropped nodes. 
However, DropNode is designed to 1) enable the separation of feature propagation and transformation for random feature propagation and 2) further augment graph data augmentations and facilitate the consistency regularized training. 

%random feature propagation and 

% for tackling the inherent issues of GNNs. % (Cf. Figure~\ref{fig:grand_intro}).
%challenges faced by GNNs (Cf. Figure ~\ref{fig:grand_intro}).
%, including 
% \yd{Modify the writing of this section, since DropNode and \model\ have not been introduced.}

%the issues of 
%non-robustness, over-smoothing,  overfitting, and weak generalization (Cf. Figure ~\ref{fig:grand_intro}).

%\yd{to wz: add ref the Figure in Introduction}).

}

\hide{

\section{Definition and Preliminaries}
\label{sec:problem}

In this section, we first present the problem formalization of graph-based semi-supervised learning. Then we briefly introduce the concept of graph convolutional networks and consistency regularization. 

%\vpara{Definition.} We represent 
%Given a set of labeled data samples $L=\{(x_1, y_1),(x_2, y_2),...,(x_{|L|}, y_{|L|})\}$, where $x_i \in \mathbb{R}^d$ and $y_i$ are the feature and label of the $i^{th}$ sample, as well as unlabeled samples $U=\{x_{|L|+1}, x_{|L|+2},...,x_{|L|+|U|}\}$. We further denote $X \in $
Let $G=(V, E)$ denote a graph, where $V$ is a set of $|V|=n$ nodes, $E \subseteq V \times V$ is a set of $|E|=k$ edges between nodes. Let $A\in \{0,1\}^{n \times n}$ denote the adjacency matrix of the graph $G$, with an element $A_{ij}>0$ indicating node $v_i \in V$ has a link to $v_j \in V$. We further assume each node $v_i$ is associated with a $d-$dimensional feature vector $x_i$ and a class label $y_i \in \mathcal{Y} = \{1,...,C\}$. 
%Each node $v_i \in V$ is also associated with a $d-$dimensional feature vector $x_i$ and a label $y_i$. 
%We represent a partially labeled graph $G$ as $(A, V, E, X)$, where $A\in \{0,1\}^{n \times n}$ denotes the adjacency matrix, $V=L\cup U$ is a set of $n$ nodes, consisting of labeled node set $L=\{v_0,  v_1,..., v_{|L|-1}\}$ and unlabeled node set $U=\{v_{|L|},v_{|L|+1}, ..., v_{|L|+|U|-1}\}$, E is the edge set, and $X \in R^{n \times d}$ is the feature matrix assuming node $v_i$ has a $d$-dimensional feature $X_i$. Labeled node $v_i$ has a label $y_i \in \mathcal{Y}=\{0, ..., M-1\}$. For convenience, we use $v_i$ and $i$ interchangeably to represent a node if there is no ambiguity

 %Given a partially labeled graph $G$ with labeled data $\{(X_i, y_i): i\in L\}$ and unlabeled data $\{X_i, i\in U\}$, the goal of semi-supervised node classification is to learn a mapping $f_{(A, X)}:V \mapsto \mathcal{Y}$.

%the input node feature matrix with the feature dimension $d$. 
%and $X \in R^{n \times d}$ is a feature matrix recording $d$ features associated with each node. 
In this work, we focus on semi-supervised learning on graph. In this setting, only a part of nodes have observed their labels and the labels of other samples are missing. Without loss of generality
, we use $V^L=\{v_1, v_2, ..., v_{m}\}$ to represent the set of labeled nodes, and $V^U=\{v_{m+1}, v_{m+2}, ..., v_{n}\}$ refers to unlabeled nodes set, where $ 0 < m \ll n$. Denoting the feature matrix of all nodes as $X \in \mathbf{R}^{n\times d}$, the labels of labeled nodes as $Y^L$, 
the goal of semi-supervised learning on graph is to infer the missing labels $Y^U$ based on $(X, Y^L, G)$. 
%Given a partially labeled graph $G=(V^L, V^U, E, X, Y^L)$, where $V^L$ is a set of nodes with labels $Y^L$ and $V^U$ is the remaining set of unlabeled nodes, the objective is to learn a function $f:G=(V^L, V^U, E, X, Y^L)\mapsto Y^U$ to infer unlabeled nodes' labels. 
%When talking about representation learning for nodes, our objective is to improve the node classification performance with better representations.

%with node features as $(A, X)$, where $A\in \{0,1\}^{n \times n}$ is the adjacency matrix, and $X \in R^{n \times d}$ is the feature matrix with feature dimension $d$. 
%Given a node set of labeled  nodes $V^L = \{(v_1, y_1),(v_2, y_2), ...\}$ where $y_i \in \mathcal{Y} $ is the corresponding label and usually  $|V^L| \ll |V|$, the goal of graph node classification is to learn a mapping $f_{(A, F)}:V \mapsto \mathcal{Y}$. 

%Different from the standard supervised learning for i.i.d data, graph node classification has much less but co-dependent labeled data. Hence, we should find a way to represent nodes in a graph so that massive unlabeled data structured by the graph can be leveraged.

\subsection{Graph Propagation}
%\reminder{Graph Propagation}
Almost all the graph models, whether based on probabilistic graph theory or neural networks, have a very important phase, i.e., graph propagation, which is usually performed by some deterministic rules derived from the graph structures. Here we take the propagation mechanism in GCN \cite{kipf2016semi}, a popular graph neural model, as an example to introduce it.
%In this work, our node representation framework is built upon graph convolutional networks.
%Here, we refer to the definition of graph convolution in spectral methods and it can be fast approximated   %fast in the original GCNs~\cite{kipf2016semi} 
%by the multi-layer graph convolution networks (GCNs)~\cite{kipf2016semi} with the following layer-wise propagation rule:
GCN performs information propagation by the following rule:
\begin{equation}
\label{equ:gcn_layer}
	H^{(l+1)} = \sigma (\hat{A}H^{(l)}W^{(l)}),
\end{equation}
\noindent where $\hat{A}=\widetilde{D}^{-\frac{1}{2}}\widetilde{A}~\widetilde{D}^{-\frac{1}{2}}$ is the symmetric normalized adjacency matrix, $\widetilde{A} = A + I_{n}$ is the adjacency matrix with augmented self-connections, $I_{n}$ is the identity matrix, $\widetilde{D}$ is the diagonal degree matrix with $\widetilde{D}_{ii} = \sum_j \widetilde{A}_{ij}$, and $W^{(l)}$ is a layer-specific learnable weight matrix. Function $\sigma_l(.)$ denotes a nonlinear activation function, e.g., $\mathrm{ReLu}(.)=\max(0,.)$. $H^{(l)}\in R^{n \times D^{(l)}}$ is the matrix of activations, or the $D^{(l)}$-dimensional hidden node representation in the $l^{th}$ layer % with the hidden representation dimension $D^{(l)}$. 
with $H^{(0)}=X$. 

The form of propagation rule in Equation \ref{equ:gcn_layer} could be motivated via a first-order approximation of the spectral graph convolutional operations~\cite{bruna2013spectral,henaff2015deep, defferrard2016convolutional}, and it also obeys neural message passing~\cite{gilmer2017neural}, similar to convolutions on direct neighborhoods~\cite{monti2017geometric, hamilton2017inductive}. Hence, in the rest of the paper, we will not distinguish between spectral or non-spectral convolutional neural networks, and between graph convolutions and neural message passing, strictly, for the  convenience of narration.

\subsection{Regularization  in Semi-supervised Learning.} 
\subsubsection{Data Augmentation}
A widely used regularization method in deep learning is data augmentation, which expands the training samples by applying some transformations on input data \cite{devries2017dataset}. In practical, the transformations used here can be some domain-specific methods, such as rotation, shearing on images. What's more, a general kind of data augmentation method is injecting noise to the input data~\cite{goodfellow2016deep}, this idea has been adopted in a variety of deep learning methods, e.g., denoising auto-encoders~\cite{vincent2008extracting}. In this way, the model is encouraged to be invariant to transformations and noises, leading to better generalization and robustness. 
\subsubsection{Consistency Regularization}
%A widely used regularization method in supervised learning is data augmentation, which expands the training dataset by performing small perturbations and transformation on input data \cite{devries2017dataset}. 
%Roughly speaking, this method requires 
%by performing small perturbations in input feature space~.

Consistency regularization is based on data augmentation and applies it into semi-supervised learning. 
Roughly speaking, it enforces the model to output the same class distribution for the same unlabeled example under different augmentations \cite{berthelot2019mixmatch}. Let $f_{aug}(x)$ denote a stochastic augmentation function for unlabeled data $x$, e.g., randomly transformation or adding noise. One alternative method to perform consistency regularization is adding a regularization loss term \cite{laine2016temporal}:
\begin{equation}
\label{eq_consis}
    \|P(y|\widetilde{x}^{(1)}, \theta) - P(y|\widetilde{x}^{(2)}, \theta) \|^2,
\end{equation}
where $\theta$ refers to model parameters, $\widetilde{x}^{(1)}$ and $\widetilde{x}^{(2)}$ are two random augmentations of $x$.

}%end of hide 

\begin{figure*}[ht]
	\centering
	\includegraphics[width=0.95\linewidth]{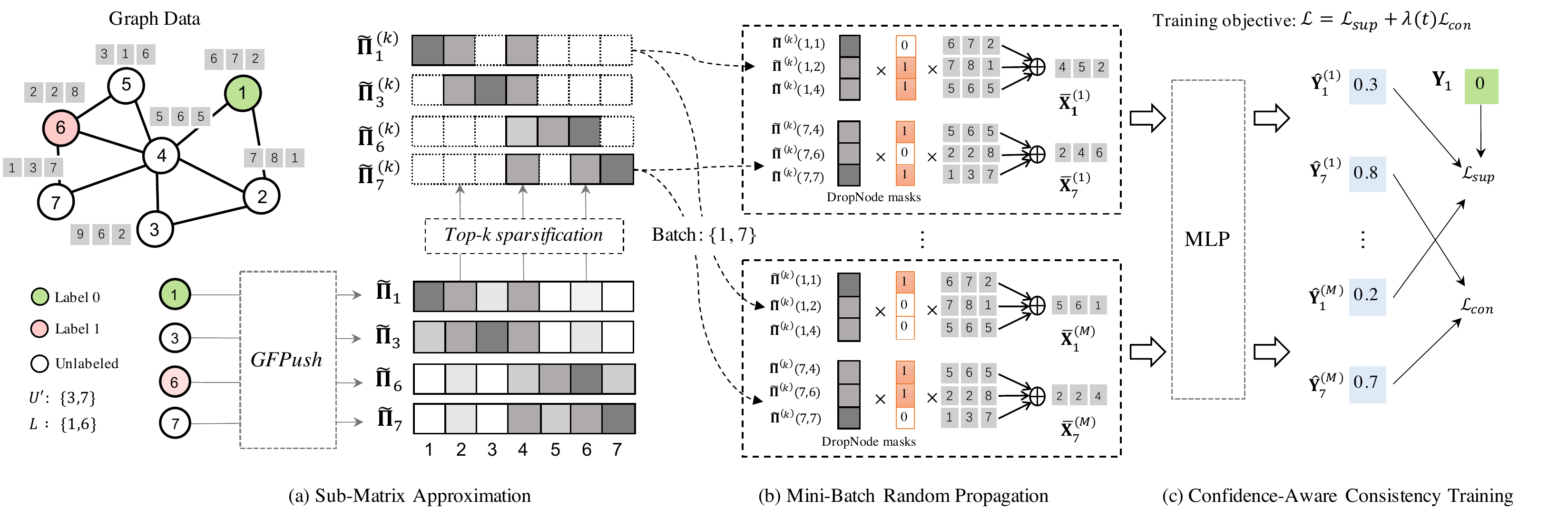}
	\vspace{-0.15in}
	\caption{Illustration of \model.  \textmd{\small (a) \model adopts \textit{Generalized Forward Push} (\textit{GFPush}) and \textit{Top-k sparsification} to approximate the corresponding rows of propagation matrix $\mathbf{\Pi}$ for nodes in $L \cup U'$. (b) The obtained sparsified row approximations are then used to perform mini-batch random propagation to generate augmentations for nodes in the batch. (c) Finally, the calculated feature augmentations are fed into an MLP to conduct confidence-aware consistency training, which employs both supervised loss $\mathcal{L}_{sup}$ and confidence-aware consistency loss $\mathcal{L}_{con}$  for model optimization.}}
		% \textmd{In this example, the graph has seven nodes indicated with 1 to 7, where only node 1 and node 6 are labeled with two different classes, and other nodes' labels are missing. The sampled unlabeled subset $U'=\{3,7\}$. In top-$k$ sparsification, the maximum neighborhood size $k$ is set to 3. The batch size is set to 2.}}
	\label{fig:arc}
	\vspace{-0.06in}
\end{figure*}

\section{The \model Framework}
\label{sec:method}

In this section, we briefly review the graph random neural network (GRAND) and present its scalable solution \model for large-scale semi-supervised graph learning. 

\subsection{The Graph Random Neural Network}

Recently, Feng et al.~\cite{feng2020grand} introduce the graph neural neural network (GRAND) for semi-supervised node classification. 
GRAND is a GNN consistency regularization framework that optimizes the prediction consistency of unlabeled nodes in different augmentations.

Specifically, it designs 
%it exploits the mixed-order propagation to achieve graph data augmentation. 
%They develop 
 \textit{random propagation}---a mixed-order propagation strategy---to achieve graph data augmentations. 
%, in which 
First, the node features $\mathbf{X}$ are  randomly dropped with DropNode---a variant of dropout. 
Then the resultant corrupted feature matrix is propagated over the graph with a mixed-order matrix. 
Instead of the PPR matrix, GRAND uses an average pooling matrix %from order $0$ to $N$:
$\mathbf{\Pi}^{\text{avg}}_{\text{sym}} = \sum_{n=0}^{N} \hat{\mathbf{A}}^n/(N+1)$ for propagation.
%And the random propagation is defined as
%Instead of the propagation rule used in GCN, a mixed-order proximity matrix $\mathbf{\Pi}^{\text{mix}}= \sum_{k=0}^{K}\frac{1}{K+1}(\mathbf{D}^{-\frac{1}{2}}\mathbf{A}\mathbf{D}^{-\frac{1}{2}})^k$ is adopted to preserve both local and global information. 
Formally, the random propagation strategy is formulated as: 
\begin{equation}
\small
\label{equ:randprop}
    \overline{\mathbf{X}} = \mathbf{\Pi}^{\text{avg}}_{\text{sym}} \cdot \text{diag}(\mathbf{z}) \cdot \mathbf{X}, \quad \mathbf{z}_i \sim \text{Bernoulli}(1 - \delta),
\end{equation}
where $\mathbf{z} \in \{0,1\}^{|V|}$ denotes the random DropNode masks drawn from Bernoulli($1 - \delta$), and $\delta$ represents DropNode probability. 
%, the detail of \text{DropNode} is described in Algorithm~\ref{alg:dropnode}.
In doing so, the dropped information of each node is compensated by its neighborhoods. 
Under the homophily assumption of graph data, the resulting matrix $\overline{\mathbf{X}}$ can be seen as an effective data augmentation of the original feature matrix $\mathbf{X}$. 
Owing to the randomness of DropNode, this method can in theory generate exponentially many augmentations for each node.

%Based on random propagation, the authors further propose graph random neural network (GRAND).
%, a simple and effective framework for semi-supervised learning on graphs.  
In each training step of \grand, the random propagation procedure is performed for $M$ times, leading to $M$ augmented feature matrices $\{\overline{\mathbf{X}}^{(m)}|1\leq m \leq M\}$. 
Then all the augmented feature matrices are fed into an MLP to get $M$ predictions. 
During optimization, \grand\ is trained with both the standard classification loss on labeled data and an additional \textit{consistency regularization} loss~\cite{berthelot2019mixmatch} on the unlabeled node set $U$, that is, 
\begin{equation}
\label{equ:grand_consis}
\small
    \frac{1}{M\cdot|U|}\sum_{s\in U}\sum_{m=1}^M \norm{\hat{\mathbf{Y}}_s^{(m)} - \overline{\mathbf{Y}}_s}_2^2, \quad \overline{\mathbf{Y}}_s = \sum_{m=1}^M\frac{1}{M}\hat{\mathbf{Y}}_s^{(m)},%\footnote{For brevity, here we omit the sharpening trick adopted in \grand. For the complete formulation of GRAND's consistency loss, please refer to Equation 3 in \cite{feng2020grand}.}
\end{equation}
where $\hat{\mathbf{Y}}_s^{(m)}$ is MLP's prediction probability for node $s$ when using $\overline{\mathbf{X}}^{(m)}_s$ as input. 
The consistency loss provides an additional regularization effect by enforcing the neural network to give similar predictions for different augmentations of unlabeled data.
%to minimize the discrepancy among $M$ predictions of each unlabeled node's augmentations
With random propagation and consistency regularization, \grand\ achieves better generalization capability over conventional GNNs~\cite{feng2020grand}.

 %As for Random Propagation (Cf. Equation~\ref{equ:randprop}), directly calculating the dense matrix $\mathbf{\Pi}^{\text{avg}}_{\text{sym}}$ is rather time-consuming. Hence the authors propose to calculate $\overline{\mathbf{X}}$ with power iteration, i.e., iteratively calculating and summing up the product of sparse matrix $\hat{\mathbf{A}}$ and $\hat{\mathbf{A}}^t \cdot \text{diag}(\mathbf{z}) \cdot {\mathbf{X}}$ ($0 \leq t \leq T-1$). 
\vpara{Scalability  of \grand.} 
In practice, the $n$-th power of the adjacency matrix $\hat{\mathbf{A}}^n$ is computationally infeasible when $n$ is large~\cite{qiu2018network}. 
To avoid this issue, \grand adopts the power iteration to directly calculate the entire augmented feature matrix $\overline{\mathbf{X}}$ (in Equation~\ref{equ:randprop}), i.e., iteratively calculating and summing up the product of $\hat{\mathbf{A}}$ and $\hat{\mathbf{A}}^n\cdot\text{diag}(\mathbf{z})\cdot{\mathbf{X}}$ for $0\leq n < N$. % {\mathbf{X}}$ ($0 \leq t \leq T-1$). 
 %The corresponding time and memory complexity are $\mathcal{O}(|V| + |E|)$. 
 %This procedure has time complexity $\mathcal{O}(M\cdot(|V|+ |E|))$ and memory complexity $\mathcal{O}(|V| + |E|)$. =
This procedure is implemented with the sparse-dense matrix multiplication and has a linear time complexity of $\mathcal{O}(|V| + |E|)$. However, it needs to be performed for $M$ times at \textit{every} training step to generate different feature augmentations. 
Thus the total complexity of $T$ training steps becomes $\mathcal{O}(T\cdot M\cdot (|V|+|E|))$, which is prohibitively expensive when dealing with large graphs.
\subsection{Overview of \model}
\label{sec:overview}

We present \model to 
%To 
achieve both scalability and accuracy for graph based semi-supervised learning. 
 %in solving GSSL. %, we propose \model. 
It follows the general consistency regularization principle of \grand and  
comprises techniques to make it scalable to large graphs 
%significantly 
%extends \grand %in order to achieve 
%to have a good scalability 
while maintaining %and even exceeding
\grand's flexibility and generalization capability. 
%It follows the general framework of \grand, while significantly refining its key components with several novel techniques: 
%Compared with \grand\, the novelty of \model\ lies in two folds: 

Briefly, instead of propagating features with power iteration, 
we develop an efficient approximation algorithm---generalized forward push (GFPush)---in \model to pre-compute the required row vectors of propagation matrix and perform random propagation in a mini-batch manner. 
The time complexity of this procedure is controlled by a predefined hyperparameter, % (Cf. Section~\ref{sec:model_analysis}), 
avoiding the scalability limitation faced by \grand. 
  % and adopts a novel matrix approximation approach---GFPush to accelerate this procedure, reducing the training complexity to be
   % irrelevant to graph size (Cf. Section~\ref{sec:model_analysis}); 
Furthermore, \model\ adopts a new confidence-aware loss for consistency regularization, which makes the training process more stable and leads to better generalization performance than \grand.

\vpara{Propagation Matrix.} In \model, we propose the following generalized mixed-order matrix for feature propagation:
\begin{equation}
\small
\label{equ:prop}
\mathbf{\Pi} = \sum_{n=0}^{N} w_n \cdot\mathbf{P}^n, \quad \mathbf{P} = \widetilde{\mathbf{D}}^{-1}\widetilde{\mathbf{A}},
\end{equation}
where $\sum_{n=0}^Nw_n=1$ and $w_n \geq 0$, $\mathbf{P}$ is the row-normalized adjacency matrix.
%Different from the symmetric normalization used in GCN, we use row normalization for $\widetilde{\mathbf{A}}$. 
% Here we use row normalization for $\widetilde{\mathbf{A}}$. 
%And the $n$-order normalized adjacency matrix $\mathbf{P}^n = (\widetilde{\mathbf{D}}^{-1}\widetilde{\mathbf{A}})^n$ is also the \textit{$n$-order random walk reverse transition matrix} of $\widetilde{G}$, where the element $\mathbf{P}^{n}(s,v)$ denotes the probability that a $n-$step random-walk goes from source node $s$ to target node $v$. 
%Thus $\mathbf{\Pi}$ can be seen as a \textit{generalized mixed-order random walk transition matrix}. 
Different from the propagation matrices used in GRAND and other GNNs, the form of $\mathbf{\Pi}$ adopts a set of tunable weights $\{w_n|0\leq n \leq N\}$ to fuse different orders of adjacency matrices. 
By adjusting $w_n$, \model\ can {flexibly} manipulate the importance of different orders of neighborhoods to suit the diverse graphs of distinct structural properties in the real world. 
%We will examine this advantage in Section~\ref{sec:overall}.
%which brings us more \textit{flexibility} to deal with diverse graphs.
%compared with GRAND and other typical GNNs which uses fixed weights for different orders of adjacency matrices.
%In practice, we could flexibly manipulate the importance of different orders of transition matrix by $\{w_t|0 \leq t \leq T\}$. 
%Inspired by APPNP~\cite{klicpera2018predict}, \grand~\cite{feng2020grand} and SGC~\cite{wu2019simplifying}
\hide{
Specifically, we consider three setups for $\mathbf{\Pi}$: 
1) {Truncated personalized PageRank matrix} $\mathbf{\Pi}^{\text{ppr}} =  \sum_{n=0}^{N}\alpha (1-\alpha)^n \mathbf{P}^{n}$; 
2) {Average pooling matrix} $\mathbf{\Pi}^{\text{avg}} = \sum_{n=0}^{N} \mathbf{P}^{n}/(N+1)$; 
3) {Single order matrix} $\mathbf{\Pi}^{\text{single}} =  \mathbf{P}^{N}$.
}

\vpara{Training Pipeline.} 
%Based on previous analysis, this process needs to be performed for multiple times at each training epoch, which is rather time-consuming if calculated using power iteration. 
To achieve fast training, \model\ abandons the power iteration method which directly calculates the entire augmented feature matrix $\overline{\mathbf{X}}$, and instead computes each augmented feature vector separately for each node. 
Ideally, the augmented feature vector $\overline{\mathbf{X}}_s$ of node $s$ is  calculated by:
% matrix $\overline{\mathbf{X}}$
%To accelerate model training, instead of calculating the whole augmented feature matrix $\overline{\mathbf{X}}$ with power iteration as done in GRAND, we compute each augmented feature vector $\overline{\mathbf{X}}_s$ individually using the corresponding row vector $\mathbf{\Pi}_s$ of node $s$:

\begin{equation}
	\label{equ:mini}
\small
    \overline{\mathbf{X}}_{s} = \sum_{v \in \mathcal{N}^{\pi}_s} \mathbf{z}_v \cdot \mathbf{\Pi}{(s,v)} \cdot \mathbf{X}_v, \quad \mathbf{z}_v \sim \text{Bernoulli}(1-\delta).
\end{equation}
Here we use $\mathbf{\Pi}_s$ to denote the row vector of $\mathbf{\Pi}$ corresponding to node $s$, $\mathcal{N}^{\pi}_s$ is used to represent the indices of non-zero elements of $\mathbf{\Pi}_s$, $\mathbf{\Pi}{(s,v)}$ denotes the $v$-th element of $\mathbf{\Pi}_s$. 
This paradigm allows us to generate augmented features for only a batch of nodes in each training step, and thus enables us to use efficient mini-batch gradient descent for optimization. 

However, it is difficult to calculate the exact form of $\mathbf{\Pi}_s$ in practice. 
To address this problem, we develop several efficient methods to approximate $\mathbf{\Pi}_s$ in \model. 
The approximation procedure consists of two stages. 
In the first stage, we propose an efficient method \textit{Generalized Forward Push} (\textit{GFPush}) to compute an error-bounded approximation $\widetilde{\mathbf{\Pi}}_s$ for the row vector $\mathbf{\Pi}_s$. 
In the second stage, we adopt a \textit{top-$k$ sparsification} strategy to truncate $\widetilde{\mathbf{\Pi}}_s$ to only contain the top $k$ largest elements. 
The obtained sparsified row approximation $\widetilde{\mathbf{\Pi}}^{(k)}_s$ is used to calculate $\overline{\mathbf{X}}_s$ as a substitute of $\mathbf{\Pi}_s$ (Eq.~\ref{equ:mini}). 
For efficiency, it is required to pre-compute the corresponding row approximations for all nodes used in training.
 %sparsified row vectors for all nodes used in training. 
 In addition to labeled nodes,
 %different from conventional GNNs only using labeled nodes, 
 \model\ also requires unlabeled nodes to perform consistency regularization during training. 
 To further improve efficiency, instead of using the full set of  $U$,  \model\ samples a smaller subset of unlabeled nodes $U'\subseteq U$ for consistency regularization.
%Instead of using the full set of unlabeled nodes $U$,  we sample a smaller subset of unlabeled nodes $U'\subseteq U$ to serve consistency regularization in \model for efficiency. 
 %As the number of unlabeled nodes $|U|$ could be huge, which might induce enormous computational cost into approximation stage if used all of them.
 %we sample a smaller subset of unlabeled nodes $U'\subseteq U$ to serve consistency regularization. 
%We will empirically examine the effects of $|U'|$ in Appendix A.2.
As illustrated in Figure~\ref{fig:arc}, the training pipeline of \model\ consists of three steps:
\begin{itemize}
    \item \textit{Sub-matrix approximation.} 
    We obtain a sparsified row approximation $\mathbf{\Pi}^{(k)}_s$ for each node $s \in L \cup U'$ through GFPush and top-$k$ sparsification. 
    The resultant sparsified sub-matrix is used to support random propagation. 
    \item \textit{Mini-batch random propagation.} 
    At each training step, we sample a batch of nodes from $L \cup U'$ and generate multiple augmentations for each node in the batch with the approximated row vector.
    \item \textit{Confidence-aware consistency training.} 
    We feed the augmented features into an MLP to get corresponding predictions and optimize the model with both supervised loss and confidence-aware consistency loss.
\end{itemize}

\subsection{Sub-Matrix Approximation}
% We propose
%We first show how to approximate a row vector $\mathbf{\Pi}_s$ of the propagation matrix for a specific source node $s$. To this end, 

\vpara{Generalized Forward Push (GFPush).} It can be observed that the row-normalized adjacency matrix $\mathbf{P}=\widetilde{\mathbf{D}}^{-1}\widetilde{\mathbf{A}}$ is also the reverse random walk transition probability matrix~\cite{chen2020scalable} on $\widetilde{G}$, where row vector $\mathbf{P}_s$ denotes random walk transition probabilities starting from node $s$.
%As for row vector $\mathbf{\Pi}_s = \sum_{n=0}^Nw_n\mathbf{P}^n_s$, each $\mathbf{P}^n_s=(\widetilde{\mathbf{D}}^{-1}\widetilde{\mathbf{A}})^n_s$ 
%denotes the $n$-step random walk transition probabilities starting from node $s$ on graph $\widetilde{G}$. 
Based on this fact, we propose an efficient algorithm called \textit{Generalized Forward Push} (\textit{GFPush}) to approximate row vector $\mathbf{\Pi}_s=\sum_{n=0}^Nw_n\mathbf{P}^n_s$ with a bounded error.
%, an efficient push-flow algorithm that could generate an error-bounded approximation for each row vector $\mathbf{\Pi}_s$. 
GFPush is inspired by the \textit{Forward Push}~\cite{andersen2006local} algorithm for approximating personalized PageRank vector, while has much higher flexibility with the ability to approximate the generalized mixed-order matrix $\mathbf{\Pi}$.
%Inspired by \textit{Forward Push}~\cite{andersen2006local} for approximating personalized pagerank matrix. We propose a push-flow method named GFPush to estimate $\mathbf{\Pi}_s$ with a bounded $L_1$ error. 
%The core idea of GFPush is to simulate an deterministic $T-$step random walk from $s$, however, in each step of which probability masses below a certain threshold are pruned for efficiency.
The core idea of GFPush is to simulate an $N$-step random walk probability diffusion process from $s$ with a series of pruning operations for acceleration.
%on probability masses below a certain threshold.
%at each step.
%are pruned for efficiency.
%, while eliminating the probability masses below a predefined threshold $r_{max}$ at each step. 
To achieve that, we should maintain a pair of vectors at each step $n$ ($0 \leq n \leq N$): 1) \textit{Reserve vector $\mathbf{q}^{(n)} \in \mathbb{R}^{|V|}$}, denoting the probability masses reserved at step $n$; 2) \textit{Residue vector $\mathbf{r}^{(n)}\in \mathbb{R}^{|V|}$}, representing the probability masses to be diffused beyond step $n$.
%Generally speaking, GFPush takes a source node $s$ as input, and generates an underestimation of row vector $\mathbf{\Pi}^{\text{rw}}_s$ by simulating a  deterministic $T-$hop random walk from $s$ with a series of push operations. 

Algorithm~\ref{alg:GFPush} shows the pseudo-code of GFPush. At beginning, $\mathbf{r}^{(0)}$ and $\mathbf{q}^{(0)}$ are both initialized as the indicator vector $\mathbf{e}^{(s)}$ where $\mathbf{e}^{(s)}_s=1$ and  $\mathbf{e}^{(s)}_v=0$ for $v\neq s$, meaning the random walk starts from $s$ with the probability mass of 1. Other reserve and residue vectors (i.e., $\mathbf{r}^{(n)}$ and $\mathbf{q}^{(n)}, 1 \leq n \leq N$) are set to $\vec{0}$.
%that 
%$\mathbf{R}^{(0)}{(s,v)} = 0$ for all $v \neq s$ while
% and $\mathbf{Q}^{(0)}_s=\mathbf{e}_s$, representing that the random walk starts from $s$ with the probability mass of 1. 
Then the algorithm iteratively updates reserve and residue vectors with $N$ steps. 
%At step $t$, $\mathbf{Q}^{(t)}_s$ is firstly assigned as the residue vector $\mathbf{R}^{(t)}_s$ obtained from the former step. 
In the $n$-th iteration, the algorithm conducts a \textit{push} operation (Line 5--9 of algorithm~\ref{alg:GFPush}) for node $v$ which satisfies $\mathbf{r}^{(n-1)}_v>\mathbf{d}_v\cdot r_{max}$. Here $\mathbf{d}_v = \widetilde{\mathbf{D}}(v,v)$ represents the degree of $v$,
% in self-loop augmented graph $\widetilde{G}$, 
$r_{max}$ is a predefined threshold. In the push operation, the residue $\mathbf{r}^{(n-1)}_v$ of $v$ is evenly distributed to its neighbors, and the results are stored into the $n$-th residue vector
%the residue vector of the next step, i.e., 
$\mathbf{r}^{(n)}$. Meanwhile, the reserve vector $\mathbf{q}^{(n)}$ is also updated to be identical with $\mathbf{r}^{(n)}$. After finishing the push operation on $v$, we reset $\mathbf{r}_v^{(n-1)}$ to 0 to avoid duplicate updates.
%updated residue $\mathbf{r}^{(n)}_u$ is assigned to $\mathbf{q}^{(n)}_u$. 
%After that, $\mathbf{r}^{(n-1)}_v$ is reset to $0$. 
%When the final iteration terminates, reserve vector $\mathbf{Q}^{(T)}_s$ is set as a copy of $\mathbf{R}^{(T)}_s$.

To gain more intuition of this procedure, we could observe that $\mathbf{r}^{(n-1)}_v/\mathbf{d}_v$ is the conditional probability that a random walk moves from $v$ to a neighboring node $u$, conditioned on it reaching $v$ with probability $\mathbf{r}^{(n-1)}_v$ at the previous step. Thus each push operation on $v$ can be seen as a one-step random walk probability diffusion process from $v$ to its neighborhoods. To ensure efficiency, GFPush only conducts push operations for node $v$ whose residue value is greater than $\mathbf{d}_v \cdot r_{max}$. %, where nodes whose residues less than this threshold are pruned. %will while ignoring the nodes with small residues (i.e., $\{v |\mathbf{R}^{(t)}(s,v)\leq d_v \cdot r_{max}\}$), which ensures the algorithm's efficiency. 
Thus when the $n$-th iteration is finished, $\mathbf{q}^{(n)}$ can be seen as an approximation of the $n$-step random walk transition vector $\mathbf{P}^n_s$. And  $\widetilde{\mathbf{\Pi}}_s=\sum_{n=0}^N w_n\mathbf{q}^{(n)}$ is accordingly considered as the approximation of $\mathbf{\Pi}_s$ as returned by the algorithm. % with a bounded $L_1$ error. 
% \vpara{Theoretical Analysis for GFPush.} 
 %, which is returned by Algorithm~\ref{alg:GFPush}. 
 
\vpara{Theoretical Analysis.} We have the following theorem about the bounds of time complexity, memory complexity, and approximation error of GFPush.% Algorithm~\ref{alg:GFPush}.

 \begin{thm}
 \label{thm1}
 Algorithm~\ref{alg:GFPush} has  $\mathcal{O}(N/r_{max})$  time complexity and $\mathcal{O}(N/r_{max})$ memory complexity, and returns $\widetilde{\mathbf{\Pi}}_s$ as an approximation of $\mathbf{\Pi}_s$ with the $L_1$ error bound: $\parallel\mathbf{\Pi}_s - \widetilde{\mathbf{\Pi}}_s\parallel_1 \leq N\cdot (2|E| +|V|) \cdot r_{max}$.
 \end{thm}
 
 \begin{proof}
 See Appendix~\ref{sec:proof}.
 \end{proof}
 
%The proof details can be found in Appendix~\ref{sec:proof}. 
Theorem~\ref{thm1} 
suggests that 
%the time complexity and memory complexity of GFPush are independent of the graph size. 
the approximation precision and running cost of GFPush are negatively correlated with $r_{max}$. In practice, we could use $r_{max}$ to control the trade-off between efficiency and approximation precision.
 
\hide{
Let $\mathcal{T}_t$ be the total number of push operations performed in step $t$. When the $i-$th push operation performed on $v_i$, the value of $||\mathbf{R}^{(t)}_s||_1$ will be decreased by at least $r_{max} \cdot d_{v_i}$. Since $||\mathbf{R}^{(t)}_s||_1 \leq 1$, we must have $\sum^{\mathcal{T}_t}_{i=1}  d_{v_i} \cdot r_{max} \leq 1$, thus
\begin{equation}
\label{equ:tbound}
    \sum_{i=1}^{\mathcal{T}_t} d_{v_i} \leq \frac{1}{r_{max}}.
\end{equation}
In each push operation, we need to perform $d_{v_i}$ updates to $v_i$'s neighborhoods. So the total time of push operations in step $t$ and the count of non-zero elements of $\mathbf{R}^{(t+1)}_s$ (or $\mathbf{Q}^{(t+1)}_s$) are both bounded by $\sum_{i=1}^{\mathcal{T}_t} d_{v_i}$. 
According to Equation~\ref{equ:tbound}, we can further conclude that $\widetilde{\Pi}_s^{(\text{rw})}$ has at most $\frac{T}{r_{max}}$ non-zero elements. In implementation, all the vectors are stored as sparse vectors. Thus GFPush has an identical time complexity and memory complexity $\mathcal{O}(\frac{T}{r_{max}})$.

\vpara{Error Bound.} According to Lemma~\ref{lemma1}, we can conclude the following equations:
\begin{equation}
\begin{aligned}
    \norm{ \mathbf{P}_s^{(t)} - \mathbf{Q}_s^{(t)} }_1 
    &= \norm{ \sum_{h=0}^{t-1}  (\mathbf{A}^\mathsf{T}\mathbf{D}^{-1})^h\mathbf{R}_s^{(t-1-h)} }_1  \\
    &=  \norm{\sum_{h=0}^{t-1}\mathbf{R}_s^{(t-1-h)} }_1 \\
    & \leq \sum_{h=0}^{t-1} \norm{\mathbf{R}_s^{(t-1-h)}}_1.
\end{aligned}
\end{equation}
}
%it only considers a subset of nodes $\{v|\mathbf{R}^{(t)}(s,v)>d_v \cdot r_{max}\}$.
%computes an estimation of transition vector $\mathbf{\Pi}^{\text{rw}}_s$ (i.e., the row vector corresponding to node $s$). 
%The proposed method utilizes deterministic graph traversal to generate a  of $\mathbf{\Pi}^{\text{rw}}_u$, then employs Monte-Carlo random walk to refine the resulting vector. 

\begin{algorithm}[h]
	\SetKwInOut{Input}{Input}
	\SetKwInOut{Output}{Output}
	\SetKwComment{Comment}{/* }{*/}
% 	\SetKwComment{Com}{//}{.}
\footnotesize
\caption{GFPush}
\label{alg:GFPush}
\Input{Self-loop augmented graph $\widetilde{G}$, propagation step $N$, node $s$, threshold $r_{max}$, weight coefficients $w_n$, $0 \leq n \leq N$.}
\Output{An approximation $\widetilde{\mathbf{\Pi}}_s$ of transition vector $\mathbf{\Pi}_s$ of node $s$.}
$\mathbf{r}^{(n)} \leftarrow \vec{0}$ for $n=1,...,N$; $\mathbf{r}^{(0)} \leftarrow \mathbf{e}^{(s)}$  ($\mathbf{e}^{(s)}_s = 1$, $\mathbf{e}^{(s)}_v = 0$ for $v\neq s$).\\
$\mathbf{q}^{(n)} \leftarrow \vec{0}$ for $n=1,...,N$; $\mathbf{q}^{(0)} \leftarrow \mathbf{e}^{(s)}$.\\
\For{$n=1:N$}{
%	$\mathbf{Q}^{(t)}_s \leftarrow \mathbf{R}^{(t)}_s$.\\
	\While{there exists node $v$ with $\mathbf{r}^{(n-1)}_v> \mathbf{d}_v \cdot r_{max}$}{
		\For{each $u \in \mathcal{N}_v$}{
		\Comment{$\mathcal{N}_v$ is the neighborhood set of $v$ in graph $\widetilde{G}$. }
 			$\mathbf{r}^{(n)}_u\leftarrow \mathbf{r}^{(n)}_u + \mathbf{r}^{(n-1)}_v/\mathbf{d}_v$. 
 			
 			$\mathbf{q}^{(n)}_u \leftarrow \mathbf{r}^{(n)}_u$.
		}
 		$\mathbf{r}^{(n-1)}_v \leftarrow 0$.\\
 		\Comment{Perform a push operation on $v$.}
	}
}
%$\mathbf{Q}^{(T)}_s \leftarrow \mathbf{R}^{(T)}_s$.\\
$\widetilde{\mathbf{\Pi}}_s \leftarrow \sum_{n=0}^N w_n \cdot \mathbf{q}^{(n)}$.\\
\Return{$\widetilde{\mathbf{\Pi}}_s$.}
\end{algorithm}
% \vspace{-0.1in}

% For each node $u$, a straightforward method to approximate 
\vpara{Top-$k$ Sparsification.} To further reduce training cost, we perform top-$k$ sparsification for $\widetilde{\mathbf{\Pi}}_{s}$. In this procedure, only the top-$k$ largest elements of $\widetilde{\mathbf{\Pi}}_{s}$ are preserved and other entries are set to $0$.
Hence the resultant sparsified transition vector $\widetilde{\mathbf{\Pi}}^{(k)}_s$ has at most $k$ non-zero elements. 
%Recall that $\mathbf{P}=\widetilde{\mathbf{D}}^{-1}\widetilde{\mathbf{A}}$ is defined on the self-loop augmented graph $\widetilde{G}$, and the random walk on $\widetilde{G}$ can be seen as a variation of lazy random walk on G, with $(\mathbf{I} +\mathbf{D}^{-1}\mathbf{A})/2$ as the reverse transition probability matrix. According to the theory of Escaping Mass of lazy random walk (Proposition 2.5  in~\cite{spielman2013local}), which says the probability that an $N$-hop lazy random walk starting from $s$ will concentrate around a local cluster of $s$, the sparsified transition vector $\widetilde{\mathbf{\Pi}}^{(k)}_s$ is still expected to be effective for (random) feature propagation by preserving most of ``local neighborhood nodes'' for node $s$. 
In this way, the model only considers the $k$ most important neighborhoods for each node in random propagation, which is still expected to be effective based on the clustering assumption~\cite{chapelle2009semi}.
Similar technique was also adopted by PPRGo~\cite{bojchevski2020scaling}.
We will empirically examine the effects of $k$ in Section~\ref{sec:param}.

\vpara{Parallelization.} In \model, we need to approximate row vectors for all nodes in $L\cup U'$. %, and different approximation vectors are are calculated independently with each other. 
It could be easily checked that different row approximations are calculated independently with each other in GFPush. Thus we can launch multiple workers to approximate multiple vectors simultaneously. This procedure is implemented with multi-thread programming in our implementation. %The obtained sparsified sub-matrix is used to serve random propagation.

\subsection{Mini-Batch Random Propagation}
\label{sec:rand_prop}
\model\ adopts the sparsified row approximations of $\mathbf{\Pi}$ to perform random propagation in a mini-batch manner. 
%Before that, we first perform top-k sparsification on each approximated transition vector to further reduce computational costs.
%\vpara{Top-$k$ Sparsification.} For vector $\widetilde{\mathbf{\Pi}}_s$,  its top-$k$ sparsification  $\widetilde{\mathbf{\Pi}}_s^{(k)}$ is conducted by only preserving the top-$k$ largest elements and truncating the other entries to zero. Hence  $\widetilde{\mathbf{\Pi}}^{(k)}_s$ has at most $k$ non-zeros. Recall that $\mathbf{P}=\widetilde{\mathbf{D}}^{-1}\widetilde{\mathbf{A}}$ is defined on the self-loop augmented graph $\widetilde{G}$, which can be seen as a variation of lazy random walk---with $(\mathbf{I} +\mathbf{D}^{-1}\mathbf{A})/2$ as reverse transition matrix---on G. According to the theory of Escaping Mass of lazy random walk (Proposition 2.5  in~\cite{spielman2013local}), which says that the probability that a $T-$hop lazy random walk starting from $s$ still concentrates around a local cluster of $s$. The top-$k$ sparsified transition vector $\widetilde{\mathbf{\Pi}}^{(k)}_s$ 
%is still effective for (random) feature propagation by preserving most of important neighborhood nodes for node $v$.
%\vpara{Mini-Batch Propagation.}
Specifically, at the $t$-th training step, we randomly sample a batch of labeled nodes $L_t$ from $L$, and a batch of unlabeled nodes $U_t$ from $U'$. Then we calculate augmented feature vector $\overline{\mathbf{X}}_s$ for node $s \in L_t \cup U_t$ by:
%then calculate the corresponding augmented feature vector $\overline{\mathbf{X}}_{u}$ of each node $u \in V_i$ is obtained via:
% Analogous to Equation~\ref{alg:2}, we obtain the augmented feature matrix via:
\begin{equation}
\small
\label{equ:rpbatch}
\overline{\mathbf{X}}_{s} = \sum_{v \in \mathcal{N}^{(k)}_s} \mathbf{z}_v \cdot \widetilde{\mathbf{\Pi}}^{(k)}{(s,v)} \cdot \mathbf{X}_v, \quad \mathbf{z}_v \sim \text{Bernoulli}(1-\delta),
\end{equation}
where $\mathcal{N}^{(k)}_s$ denotes the non-zero indices of $\widetilde{\mathbf{\Pi}}^{(k)}_s$, $\mathbf{X}_v\in \mathbb{R}^{d_f}$ is feature vector of node $v$. %, which can be seen as the top-$k$ neighborhoods of $v$ with $\widetilde{\mathbf{\Pi}}$ as adjacency matrix. 
At each training step, we generate $M$ augmented feature vectors $\{\overline{\mathbf{X}}^{(m)}_s|1 \leq m \leq M\}$ by repeating this procedure for $M$ times. Let $b = |L_t|+|U_t|$ denote the batch size. Then the time complexity of each batch is bounded by $\mathcal{O}(k \cdot b \cdot d_f)$.%, which is independent of the graph size.

\vpara{Random Propagation for Learnable Representations.} In Equation~\ref{equ:rpbatch}, the augmented feature vector $\overline{\mathbf{X}}_s$ is calculated with raw features $\mathbf{X}$. However, in some real applications (e.g., image or text classification), the dimension of $\mathbf{X}$ might be extremely large, which will incur a huge cost for calculation.
%$\overline{\mathbf{X}}_s$. 
To mitigate this issue, we can employ a linear layer to transform each $\mathbf{X}_v$ to a low-dimensional hidden representation $\mathbf{H}_v\in \mathbb{R}^{d_h}$ firstly, and then perform random propagation with $\mathbf{H}$:
%and obtain $\overline{\mathbf{X}}_s$ by performing random propagation with $\mathbf{H}_v$:
\begin{equation}
\small
\label{equ:high}
    \overline{\mathbf{X}}_{s} = \sum_{v \in \mathcal{N}^{(k)}_s} \mathbf{z}_v \cdot \widetilde{\mathbf{\Pi}}^{(k)}{(s,v)} \cdot   \mathbf{H}_v, \quad \mathbf{H}_v = \mathbf{X}_v\cdot \mathbf{W}^{(0)},
    %\sim \text{Bernoulli}(1-\delta)
\end{equation}
where $\mathbf{W}^{(0)}\in \mathbb{R}^{d_f \times d_h}$ denotes learnable transformation matrix. In this way, the computational complexity of this procedure is reduced to $\mathcal{O}(k \cdot b \cdot d_h)$, where $d_h \ll d_f$ denotes the dimension of $\mathbf{H}_v$. %Then the resulting augmented hidden representation $\overline{\mathbf{H}}_v$ is input into the MLP model for prediction.

%\wz{complexity}

%Note that the random propagation procedure in \model\ could also be performed on the hidden representation vector $\mathbf{H}^{(l)}_v$ as the substitution of raw feature $\mathbf{X}_v$, which ensures the computation complexity not dependent with the dimension of raw features.

\vpara{Prediction.} During training, the augmented feature vector $\overline{\mathbf{X}}^{(m)}_s$ is fed into an MLP model to get the corresponding outputs:
\begin{equation}
\small
\hat{\mathbf{Y}}^{(m)}_s = \text{MLP}(\overline{\mathbf{X}}^{(m)}_s, \mathbf{\Theta}),
\end{equation}
%In the $l^{th}$ layer, the corresponding representation vector $\mathbf{H}^{(l,s)}_v\in \mathbb{R}^{d_l}$ are obtained with
%\begin{equation}
%\mathbf{H}^{(l,s)}_v = \sigma(\mathbf{W}^{(l)}\mathbf{H}^{(l-1,s)}_v),
%\end{equation}
where $\hat{\mathbf{Y}}^{(m)}_s\in [0,1]^C$ denotes the prediction probabilities of $s$. $\mathbf{\Theta}$ represents MLP's parameters.

%\vpara{Addressing High-dimensional Features.}

%as the substitution of  $\overline{\mathbf{X}}_v$.

%As we stated in Section~\ref{sec:overview}, the approximation and we use power iteration to perform 
%To further improve efficiency, we only preserve top-$k$ neighborhoods with the largest proximity score for each node, i.e., $|\mathcal{N}^{\text{rw}}_u|\leq k$. 
%After that, the augmented features of the batch of nodes are fed into a neural network model for semi-supervised learning or self-supervised learning. In doing so, the time complexity and memory complexity of the model is bounded by $\mathcal{O}(bk)$, where $b\ll n$ and $k\ll n$. In practice, we could control the trade-off of performance and efficiency by changing the values of $b$ and $k$.

\subsection{Confidence-Aware Consistency Training}
%Similar with \grand, 
\model\ adopts both supervised classification loss and consistency regularization loss to optimize model parameters during training. The supervised loss is defined as the average cross-entropy over multiple augmentations of labeled nodes:
\begin{equation}
\small
\label{equ:suploss}
\mathcal{L}_{sup} = -\frac{1}{|L_t|\cdot M}\sum_{s \in L_t}\sum_{m=1}^{M}\mathbf{Y}_s \cdot \log(\hat{\mathbf{Y}}^{(m)}_s).
\end{equation}

\vpara{Confidence-Aware Consistency Loss.}
%In the semi-supervised setting, supervised information is limited. 
Inspired by recent advances in semi-supervised learning~\cite{berthelot2019mixmatch}, GRAND adopts an additional consistency loss to optimize the prediction consistency of multiple augmentations of unlabeled data, which is shown to be effective in improving generalization capability. \model also follows this idea, while adopts a new confidence-aware consistency loss to further improve effectiveness. 

%Recent advances~\cite{sohn2020fixmatch} on semi-supervised image classification propose to improve model's generalization capability by optimizing the consistency of multiple augmentations of unlabeled data. Borrowing this idea, we design a confidence-aware consistency loss calculated by the unlabeled nodes in $U_n$.
Specifically, for node $s \in U_t$, we first calculate the distribution center by taking the average of its $M$ prediction probabilities, i.e., $\overline{\mathbf{Y}}_s=\sum_{m=1}^M \hat{\mathbf{Y}}^{(m)}_s/M$. Then we apply \textit{sharpening}~\cite{sohn2020fixmatch} trick over $\overline{\mathbf{Y}}_s$ to ``guess'' a pseudo label $\widetilde{\mathbf{Y}}_s$ for node $s$. Formally, the guessed probability on the $j$-th class of node $s$ is obtained via:%\htl{:}
\begin{equation}
\small
\label{equ:sharp}
\widetilde{\mathbf{Y}}(s,j) = \overline{\mathbf{Y}}(s,j)^{\frac{1}{\tau}}/\sum_{c=0}^{C-1}\overline{\mathbf{Y}}(s,c)^{\frac{1}{\tau}}, 
\end{equation}
where $0< \tau \leq 1$ is a hyperparameter to control the sharpness of the guessed pseudo label.  As $\tau$ decreases, $\widetilde{\mathbf{Y}}_s$ is enforced to become sharper and converges to a one-hot distribution eventually. 
%In experiments, we set $\tau=0.1$ and find it works well across most of settings. % \footnote{In our experiments, $\tau$ is simply set to $0.1$ and we find it works well across all datasets}.
Then the confidence-aware consistency loss on unlabeled node batch $U_t$ is defined as: 
\begin{equation}
\small
\label{equ:consis}
\mathcal{L}_{con} = \frac{1}{|U_t|\cdot M}\sum_{s\in U_t} \mathbb{I}(\max(\overline{\mathbf{Y}}_s)\geq \gamma) \sum_{m=1}^M \mathcal{D}(\widetilde{\mathbf{Y}}_s, \hat{\mathbf{Y}}^{(m)}_s),
\end{equation} 
where $\mathbb{I}(\max(\overline{\mathbf{Y}}_s)\geq \gamma)$ is an indicator function which outputs 1 if $\max(\overline{\mathbf{Y}}_s)\geq \gamma$ holds, and outputs 0 otherwise. $0\leq \gamma < 1$ is a predefined threshold. $\mathcal{D}(p, q)$ is a distance function which measures the distribution discrepancy between $p$ and $q$. Here we mainly consider two options for $\mathcal{D}$: \textit{$L_2$ distance} and \textit{KL divergence}.

Compared with the consistency loss used in \grand\ (Cf. Equation~\ref{equ:grand_consis}), the biggest advantage of $\mathcal{L}_{con}$ is that it only considers ``highly confident'' unlabeled nodes determined by threshold $\tau$ in optimization. This mechanism could reduce the potential training noise by filtering out uncertain pseudo-labels, further improving model's performance in practice. Combining $\mathcal{L}_{con}$ and $\mathcal{L}_{sup}$, the final loss for model optimization is defined as:
%$\mathbbm{1}(\max(\overline{\mathbf{Y}}_s)\geq \gamma)$ is an indicator function that outputs $1$ if the maximum probability of $\overline{\mathbf{Y}}_s$ falls above the threshold $\gamma$,  and outputs 0 otherwise. 
%Compared with the consistency regularization loss used in \grand, the biggest advantage of

%In each training step, we update model's parameters by optimizing the combination of $\mathcal{L}_{sup}$ and $\mathcal{L}_{con}$:
\begin{equation}
\label{equ:total_loss}
\small
\mathcal{L} = \mathcal{L}_{sup} + \lambda(t) \mathcal{L}_{con},
\end{equation}
where $\lambda(t)$ is a linear warmup function~\cite{goyal2017accurate} which increases linearly from 0 to the maximum value $\lambda_{max}$ as training step $t$ increases. 

\hide{
\vpara{Dynamic Loss Weight Scheduling.} We adopt a dynamic scheduling strategy to let $\lambda$ linearly increase with the number of training steps. Formally, at the $n$-th training step, $\lambda$ is obtained via 
\begin{equation}
\small
\label{equ:sch}
\lambda = \min(\lambda_{max},  \lambda_{max}\cdot n / n_{max}).
\end{equation}
In the first $n_{max}$ training steps, $\lambda$ increases from $\lambda_{min}$ to $\lambda_{max}$, and remains constant in the following training steps. This strategy limits $\lambda$ to a small value in the early stage of training when the generated pseudo-labels $\widetilde{\mathbf{Y}}$ are not much reliable, which will help model converge.
}

\vpara{Model Inference.}
%Under the semi-supervised setting, labeled nodes are usually scarce, most of nodes are unlabeled and need to be predicted during inference. 
After training, we need to infer the predictions for  unlabeled nodes. %, which are usually massive under semi-supervised setting.
\model adopts power iteration to calculate the exact prediction results for unlabeled nodes during inference:
\begin{equation}
\small
\hat{\mathbf{Y}}^{(inf)} = \text{MLP}( \sum_{n=0}^Nw_n(\widetilde{\mathbf{D}}^{-1}\widetilde{\mathbf{A}})^n \cdot (1-\delta) \cdot \mathbf{X}, \mathbf{\Theta}),
\end{equation}
where we rescale $\mathbf{X}$ with $(1 - \delta)$  to make it identical with the expectation of the DropNode perturbed features used in training. 
Note that unlike \grand, the above power iteration process only needs to be performed once in \model, and the computational cost is acceptable in practice. 
Compared with obtaining predictions with GFPush as done in training, this inference strategy could provide more accurate predictions in theory.
Algorithm~\ref{alg:grand+} shows the entire training and inference procedure of \model.

%adding a confidence indicator` $\mathbbm{1}(\max(\overline{\mathbf{Y}}_u \geq \gamma))$. $I$

%To generate effective graph data augmentations, \grand\ employs multi-order propagation in 

%we still use power iteration to perform feature propagation 
%This inspires us to adopt different strategies for training and inference. In model training,  we only conduct approximation for a small proportion of rows of $\mathbf{\Pi}$, corresponding to the limited nodes used in training. 
%As for inference, approximating rows for the large amount of unlabeled nodes will cause tremendous cost when processing large graphs. Thus we still adopt power iteration to perform feature propagation for inference. Different from training, the augmented features only needs to be calculated once with deterministic masks in that case, and the computation cost is acceptable in practice. Similar idea has been adopted by PPRGo~\cite{bojchevski2020scaling}.

%As expressed in Section~\ref{sec:overview}, the approximate random propagation strategy is only conducted in the period of training. 
%As for inference, we utilize power iteration to perform deterministic mixed-order feature propagation with $\mathbf{X}$. Formally, the final predictions for all nodes in $V$ are inferred by
\hide{
\begin{equation}
\hat{\mathbf{Y}}^{(inf)} = \text{MLP}(\mathbf{\Pi}\cdot (1-\delta) \cdot \mathbf{X}, \mathbf{\Theta}),
\end{equation}
}

\begin{algorithm}[t!]

	\SetKwInOut{Input}{Input}
	\SetKwInOut{Output}{Output}
	\SetKwComment{Comment}{/* }{*/}
	\caption{\model}
	\footnotesize
	\label{alg:grand+}
	\Input{Graph $G$, feature matrix $\mathbf{X} \in \mathbb{R}^{|V| \times d_f}$, labeled node set $L$, unlabeled node set $U$ and observed labels $\mathbf{Y}_L \in \mathbb{R}^{|L|\times C}$.}
	\Output{Classification probabilities $\hat{\mathbf{Y}}^{(inf)}$.} %\jie{what is this?}
		%\STATE $\pi^{(0)} \leftarrow \emptyset$, $\pi^{(1)} \leftarrow \emptyset$
		%\STATE $\Omega\leftarrow U$
    Sample a subset of unlabeled nodes $U'$ from $U$.\\

		\For{$s \in L \cup U'$}{     
		 $\widetilde{\mathbf{\Pi}}_s \leftarrow \text{GFPush}(G, s)$.\\
		 Obtain $\widetilde{\mathbf{\Pi}}^{(k)}_s$ by applying top-$k$ sparsification on $\widetilde{\mathbf{\Pi}}_s.$ \\
		 \Comment{Approximating row vectors with pallalization.}
		}
		\For{$t=0:T$}{
			Sample a batch of labeled nodes $L_t \subseteq L$ and a batch of unlabeled nodes $U_t \subseteq U'$.\\
			\For{$s \in L_t \cup U_t$}{
				\For{$m=1:M$}{
					Generate augmented feature vector $\overline{\mathbf{X}}^{(m)}_{s}$ with Equation~\ref{equ:rpbatch}.\\
					%$\overline{\mathbf{X}}^{(m)}_{v} = \sum_{j \in \mathcal{N}^{\pi}_v}  \widetilde{\mathbf{\Pi}}^{(k)}{(v,j)} \cdot  \mathbf{z}_j \cdot \mathbf{X}_j, \quad \mathbf{z}_j \sim \text{Bernoulli}(1-\delta).$\\
					Predict class distribution with %$\widetilde{Z}^s = \text{GCN}(\widetilde{X}_k, A)$
					$\hat{\mathbf{Y}}^{(m)}_s = \text{MLP}(\mathbf{\overline{X}}^{(m)}_s, \mathbf{\Theta}).$
				}
			}%p(\mathbf{Y}|\overline{\mathbf{X}}^{(s)};\Theta)$.
			Compute $\mathcal{L}_{sup}$ via Equation~\ref{equ:suploss} and $\mathcal{L}_{con}$ via Equation~\ref{equ:consis}.\\
			Update the parameters $\Theta$ by mini-batch gradients descending: $\mathbf{\Theta} = \mathbf{\Theta} - \eta \nabla_\mathbf{\Theta} (\mathcal{L}_{sup} + \lambda \mathcal{L}_{con}).$\\
			\Comment{Stop training with early-stopping.}
		}
		% \STATE $Z^j_{(t)} =\text{GCN}^j(\widetilde{X}^j_{(t)})$
		%\STATE $Z^{(j)}_{(t)} =\text{GCN}^{(j)}(S^{(j)}[t])$
		%\STATE Perform SGD on $Loss^{(j)}(Y^{(j)},Z^{(j)}_{(t)},L^{(j)})$ %to update GCN$^{j}$
		%\STATE Feed $\widetilde{X}^{(j)}_{(t)}$ into $\text{GCN}^{(j)}$

		%\STATE 
	
		Infer classification probabilities $\hat{\mathbf{Y}}^{(inf)} =  \text{MLP}( \mathbf{\Pi}\cdot (1-\delta) \cdot \mathbf{X}, \mathbf{\Theta})$.
		
		%P\left(\mathbf{Y}~\bigg|~\frac{1}{K+1}\sum_{k=0}^K\hat{\mathbf{A}}^k \mathbf{X};\hat{\Theta}\right)$.
		\Return{$\hat{\mathbf{Y}}^{(inf)}.$}
\end{algorithm}
\subsection{Model Analysis} 
\label{sec:model_analysis}

\vpara{Complexity Analysis.}  We provide detailed analyses for the time complexities of \model's different learning stages.
 According to Theorem~\ref{thm1}, the complexity of approximation stage (Line 2--5 of Algorithm~\ref{alg:grand+}) is $\mathcal{O}((|U'|+|L|)\cdot N/r_{max})$. As for the training stage (Line 6--15 of Algorithm~\ref{alg:grand+}), the total complexity of $T$ training steps is $\mathcal{O}(k\cdot b \cdot M \cdot T)$, which is  practically efficient for large graphs since $b$ and $k$ are usually much smaller than the graph size.
  The complexity of inference stage (Line 17 of Algorithm~\ref{alg:grand+}) is $\mathcal{O}((|V|+|E|) \cdot N)$, which is linear with the sum of node and edge counts.
 
\vpara{\model\ vs. PPRGo and GBP.}
Similar with \model,  PPRGo~\cite{bojchevski2020scaling} and GBP~\cite{chen2020scalable} also adopt matrix approximation methods to scale GNNs. However, \model\ differs from the two methods in several key aspects. PPRGo scales up APPNP by using Forward Push~\cite{andersen2006local} to approximate the ppr matrix. Compared with PPRGo, \model\ is more flexible in real applications thanks to the adopted generalized propagation matrix $\mathbf{\Pi}$ and GFPush algorithm. GBP also owns this merit by using the generalized PageRank matrix~\cite{li2019optimizing} for feature propagation. However, it directly approximates the propagation results of raw features through bidirectional propagation~\cite{banerjee2015fast}, whose computational complexity is linear with the raw feature dimension, rendering it difficult to handle datasets with high-dimensional features.
%which makes the running cost highly dependent on the raw feature dimension, thereby causes much difficulties in processing datasets with high-dimensional features.
%and is difficult to handle high-dimensional features. 
Moreover, different from PPRGo and GBP designed for the general supervised classification problem,  \model\ makes significant improvements for semi-supervised setting by adopting random propagation and consistency regularization to enhance generalization capability.
%during training, which brings substantial advantage in improving model's generalization performance in semi-supervised setting.

\hide{
In this section, we present the \full\ (\model) for semi-supervised learning on graphs. 
Its idea is to enable each node to randomly propagate with different subsets of neighbors in different training epochs. 
This random propagation strategy is demonstrated as an economic way for stochastic graph data augmentation, based on which we  design a consistency regularized training for improving \model's generalization capacity. 

Figure \ref{fig:arch2} illustrates the full architecture of \model.
%with the consistency regularized training. 
Given an input graph, \model\ generates multiple data augmentations by performing random propagation (DropNode + propagation) multiple times at each epoch. 
In addition to the classification loss, \model\ also leverages a consistency regularization loss to enforce the models to give similar predictions across different augmentations. 
}
%by encouraging predictions invariant to different augmentations of the same node. 

\hide{
\subsection{Random Propagation}
%\jt{looks not like ``operations''. seems to me ``dropnode'' is the only operation, followed by a discussion and a prediction.}\\
%\jt{the current structure of section 3 and 4 is a bit messy}\\
%\jt{first here the goal is to design grand to address the three challenges, so the most important thing is to explain how we address the three challenges and the foundamental ideas behind the design. the current explanation ``This enables \model\ to reduce the risks of the overfitting and over-smoothing issues. '' is too weak.}\\
%\jt{the analysis of dropnode and dropout can be moved to section 4.}\\

%The basic \model\ operations are illustrated in Figure~\ref{fig:arch1}. 
Given an input graph with its associated feature matrix, 
\model\ first conducts the random propagation process and then makes the prediction by using the simple multilayer perceptron (MLP) model. 

The motivation for random propagation is to address the non-robustness issue faced by existing GNNs~\cite{dai2018adversarial,zugner2019adversarial,zugner2018adversarial}. 
This process is coupled with the DropNode and propagation steps. 
In doing so, \model\  naturally separates the feature propagation and non-linear transformation operations in standard GNNs, enabling \model\ to reduce the risk of the overfitting and over-smoothing issues.

\vpara{DropNode.}
\label{sec:randpro}
In random propagation, we aim to perform message passing in a random way during model training such that each node is not sensitive to specific neighborhoods. 
To achieve this, we design a simple yet effective node sampling operation---DropNode---before the propagation layer.

%we randomly ``drop'' some nodes before propagation. This strategy is called dropnode.

DropNode is designed to randomly remove some nodes' all features. 
In specific,  at each training epoch, the entire feature vector of each node is randomly discarded with a pre-defined probability, i.e., some rows of $\mathbf{X}$ are set to $\vec{0}$.  
The resultant perturbed feature matrix $\widetilde{\mathbf{X}}$ is then fed into the propagation layer. 

The formal DropNode operation is shown in Algorithm \ref{alg:dropnode}. 
First, we randomly sample a binary mask $\epsilon_i \sim Bernoulli(1-\delta)$ for each node $v_i$. 
Second, we obtain the perturbed feature matrix $\widetilde{\mathbf{X}}$ by multiplying each node's feature vector with its corresponding mask, i.e., $\widetilde{\mathbf{X}}_i=\epsilon_i \cdot \mathbf{X}_i$. 
Finally,  we scale $\widetilde{\mathbf{X}}$ with the factor of $\frac{1}{1-\delta}$ to guarantee the perturbed feature matrix is in expectation equal to $\mathbf{X}$. 
Note that the sampling procedure is only performed during training. 
During inference, we directly set $\widetilde{\mathbf{X}}$ with the original feature matrix $\mathbf{X}$.

After DropNode, the perturbed feature matrix $\widetilde{\mathbf{X}}$ is fed into the propagation layer to perform message passing. 
Here we adopt mixed-order propagation, i.e., $\overline{\mathbf{X}} = \overline{\mathbf{A}} \widetilde{\mathbf{X}}$, 
%\begin{equation}
%\label{equ:kAX}
%	\overline{X} = \overline{A} \widetilde{X}.
%\end{equation}
where  $\overline{\mathbf{A}} =  \sum_{k=0}^K\frac{1}{K+1}\hat{\mathbf{A}}^k$---the average of the power series of $\hat{\mathbf{A}}$ from order 0 to order $K$. 
%\yd{define $\hat{A}$}
This kind of propagation rule enables the model to incorporate the multi-order neighborhood information, reducing the risk of over-smoothing when compared with using $\hat{\mathbf{A}}^K$ only. % when $K$ is large. 
Similar ideas have been adopted in recent GNN studies~\cite{abu2019mixhop,abu2018n}.
  
\vpara{Prediction Module.}
After the random propagation module, the augmented feature matrix $\overline{\mathbf{X}}$ can be then fed into any neural networks for predicting nodes labels. 
In \model, we employ a two-layer MLP as the classifier, that is:
\begin{equation}
\small
\label{equ:mlp}
    P(\mathbf{Y}|\overline{\mathbf{X}};\Theta) = \sigma_2(\sigma_1(\overline{\mathbf{X}}\mathbf{W}^{(1)})\mathbf{W}^{(2)})
\end{equation}
where $\sigma_1(.)$ is the ReLU function, $\sigma_2(.)$ is the softmax function, and $\Theta=\{\mathbf{W}^{(1)} \in \mathbb{R}^{d \times d_h}, \mathbf{W}^{(2)} \in \mathbb{R}^{d_h \times C}\}$ is the model parameters.

The MLP classification model can be also replaced with more complex and advanced GNN models, including GCN and GAT. 
The experimental results show that the replacements result in consistent performance drop across different datasets due to GNNs' over-smoothing problem (Cf. Appendix~~\ref{sec:oversmoothing_grand} for details).

With this data flow, it can be realized that \model\ actually separates the feature propagation (i.e., $\overline{\mathbf{X}} = \overline{\mathbf{A}} \widetilde{\mathbf{X}}$ in random propagation) and transformation (i.e., $\sigma(\overline{\mathbf{X}} \mathbf{W})$ in prediction) steps, which are coupled with each other in standard GNNs (i.e., $\sigma(\mathbf{AX} \mathbf{W})$). 
This allows us to perform the high-order feature propagation $\overline{\mathbf{A}} =  \frac{1}{K+1}\sum_{k=0}^K\hat{\mathbf{A}}^k$ without increasing the complexity of neural networks, reducing the risk of overfitting and over-smoothing.

\subsection{Consistency Regularized Training}
We show that random propagation can be seen as an efficient method for stochastic data augmentation. 
As such, it is natural to design a consistency regularized training algorithm for \model.

\vpara{Random Propagation as Stochastic Data Augmentation.}
%Feature propagation in existing GNNs and Weisfeiler-Lehman Isomorphism test~\cite{shervashidze2011weisfeiler} has been proven to be an effective method for enhancing node representation by aggregating information from neighborhoods. 
%We discuss the additional implications that \model's random propagation brings into feature propagation. 
Random propagation randomly drops some nodes' entire features before propagation. 
As a result, each node only aggregates information from a random subset of its (multi-hop) neighborhood. 
In doing so, we are able to stochastically generate different representations for each node, which can be considered as a stochastic graph augmentation method. 
In addition, random propagation can be seen as injecting random noise %---dropping a portion of nodes---
into the propagation procedure.

To empirically examine this data augmentation idea, we generate a set of augmented node representations $\overline{\mathbf{X}}$ with different drop rates in random propagation and use each $\overline{\mathbf{X}}$ to train a GCN for node classification on commonly used datasets---Cora, Citeseer, and Pubmed.  %\yd{ to wz: why use GCN, not MLP (Grand)}
The results show that the decrease in GCN's classification accuracy is less than $3\%$ even when the drop rate is set to $0.5$. 
In other words, with half of rows in the input $\mathbf{X}$ removed (set to $\vec{0}$), random propagation is capable of generating augmented node representations that are sufficient for prediction.

Though one single $\overline{\mathbf{X}}$ is relatively inferior to the original $\mathbf{X}$ in performance, in practice, multiple augmentations---each per epoch---are utilized for training the \model\ model. 
Similar to bagging~\cite{breiman1996bagging}, \model's random data augmentation scheme makes the final prediction model implicitly assemble models on exponentially many augmentations, yielding much better performance than the deterministic propagation used in GCN and GAT.

 \begin{algorithm}[tb]
\caption{Consistency Regularized Training for \model}
\small
\label{alg:2}
\begin{algorithmic}[1] %[1] enables line numbers
\REQUIRE ~~\\
 Adjacency matrix $\hat{\mathbf{A}}$,
%labeled node set $V^L$,
% unlabeled node set $V^U$,
feature matrix $\mathbf{X} \in \mathbb{R}^{n \times d}$, 
%$T$: maximum number of learning epochs,\\
% $K$: times of random propagation each epoch, \\
times of augmentations in each epoch $S$, DropNode probability $\delta$.\\
\ENSURE ~~\\
Prediction $\mathbf{Z}$. %\jie{what is this?}
%\STATE $\pi^{(0)} \leftarrow \emptyset$, $\pi^{(1)} \leftarrow \emptyset$
%\STATE $\Omega\leftarrow U$
\WHILE{not convergence}
\FOR{$s=1:S$} 
%\STATE Obtaining the deformity feature matrix $X^{'}$ by applying graph dropout on $X$ via Eq.\ref{equ:nodedropout} or Eq.\ref{equ:featuredropout}.
\STATE Apply DropNode via Algorithm \ref{alg:dropnode}: $
\widetilde{\mathbf{X}}^{(s)} \sim \text{DropNode}(\mathbf{X},\delta)$. 
\STATE Perform propagation: $\overline{\mathbf{X}}^{(s)} = \frac{1}{K+1}\sum_{k=0}^K\hat{\mathbf{A}}^k \widetilde{\mathbf{X}}^{(s)}$.
\STATE Predict class distribution using MLP: %$\widetilde{Z}^s = \text{GCN}(\widetilde{X}_k, A)$
$\widetilde{\mathbf{Z}}^{(s)} = P(\mathbf{Y}|\overline{\mathbf{X}}^{(s)};\Theta)$.
\ENDFOR
% \STATE $Z^j_{(t)} =\text{GCN}^j(\widetilde{X}^j_{(t)})$
%\STATE $Z^{(j)}_{(t)} =\text{GCN}^{(j)}(S^{(j)}[t])$
%\STATE Perform SGD on $Loss^{(j)}(Y^{(j)},Z^{(j)}_{(t)},L^{(j)})$ %to update GCN$^{j}$
%\STATE Feed $\widetilde{X}^{(j)}_{(t)}$ into $\text{GCN}^{(j)}$
\STATE Compute supervised classification loss $\mathcal{L}_{sup}$ via Eq. \ref{equ:loss} and consistency regularization loss via Eq. \ref{equ:consistency}.
\STATE Update the parameters $\Theta$ by gradients descending:
$$\nabla_\Theta \mathcal{L}_{sup} + \lambda \mathcal{L}_{con}$$
%\STATE 
\ENDWHILE
\STATE Output prediction $\mathbf{Z}$ via Eq. \ref{equ:inference}.
\end{algorithmic}
\end{algorithm}

%\vpara{$S-$augmentation Prediction.}
\vpara{$S-$augmentation.} Inspired by the above observation, we propose to generate $S$ different data augmentations for the input graph data $\mathbf{X}$. 
%To achieve that, we adopt $S-$sampling dropnode strategy in random propagation module. 
In specific, we perform the random propagation operation for $S$ times to generate $S$ augmented feature matrices $\{\overline{\mathbf{X}}^{(s)}|1\leq s \leq S\}$. 
%Then we propagate these features according to Equation \ref{equ:kAX} respectively, and hence obtain $S$ data augmentations $\{\overline{X}^{(s)}|1 \leq s \leq S\}$. 
Each of these augmented feature matrices is fed into the MLP prediction module to get the corresponding output:

%The prediction probability on the $s^{th}$ augmented data is denoted as:
 \begin{equation}
 \small
     \widetilde{\mathbf{Z}}^{(s)} = P(\mathbf{Y}|\overline{\mathbf{X}}^{(s)}; \Theta), 
 \end{equation}
where $\widetilde{\mathbf{Z}}^{(s)} \in [0,1]^{n\times C}$ denotes the classification probabilities on the $s$-th augmented data $\overline{\mathbf{X}}^{(s)}$. 
%Then we feed each $\widetilde{X}^{(s)}$ into a prediction model to get the corresponding prediction probabilities:

\vpara{Classification Loss.}
With $m$ labeled nodes among $n$ nodes, the supervised objective of the graph node classification task in each epoch is the average cross-entropy loss over $S$ augmentations:
\begin{equation}
\small
\label{equ:loss}
	\mathcal{L}_{sup} = -\frac{1}{S}\sum_{s=1}^{S}\sum_{i=0}^{m-1}\mathbf{Y}_{i} \cdot \log \widetilde{\mathbf{Z}}_{i}^{(s)} ,
\end{equation}

\noindent where $\widetilde{\mathbf{Z}}_i^{(s)}$ is the $i$-th row vector of  $\widetilde{\mathbf{Z}}^{(s)}$. Optimizing this loss enforces the model to output the same predictions for (only) labeled nodes on different augmentations.
However, labeled data is often very rare in the semi-supervised setting, in which we would like to also make full use of unlabeled data. 
%, we also employ consistency regularization loss in \model. 
 
 %Let $\widetilde{Z}^s = p_{model}(\mathcal{Y}|\widetilde{X}^s, \hat{A})$ denote model's prediction probabilities for the $s$-th augmented data $\widetilde{X}^s$. 

\vpara{Consistency Regularization Loss.} 
In the semi-supervised setting, we propose to optimize the consistency among $S$ augmentations for unlabeled data. Considering a simple case of $S=2$, we can
%Let us first consider a simple case, where the random propagation procedure is performed twice in each epoch, i.e., $S=2$.
%A straightforward method is to 
minimize the distributional distance between the two outputs, i.e.,
 %\begin{equation}
 %\small
 %\label{equ:2d}
 $
\min \sum_{i=0}^{n-1} \mathcal{D}(\widetilde{\mathbf{Z}}^{(1)}_i, \widetilde{\mathbf{Z}}^{(2)}_i)$,
%\end{equation}
where $ \mathcal{D}(\cdot,\cdot)$ is the distance function. 
%Next we show how to extend it into multiple-augmentation situation.
To extend this idea into multiple-augmentation situation, we first 
%First, we 
calculate the label distribution center by taking the average of all distributions, i.e., 
 $ \overline{\mathbf{Z}}_i = \frac{1}{S}\sum_{s=1}^{S} \widetilde{\mathbf{Z}}_i^{(s)}$.
%which could be followed by the minimization of 
Then we minimize the distributional distance between $\widetilde{\mathbf{Z}}_i^{(s)}$ and $\overline{\mathbf{Z}}_i$, i.e., $\min \sum_{s=1}^{S}\sum_{i=0}^{n-1} \mathcal{D}(\overline{\mathbf{Z}}_i, \widetilde{\mathbf{Z}}^{(s)}_i)$.
%\begin{equation}
%    \min \sum_{i=1}^{K}\sum_{j=1}^n \mathcal{D}(\bar{Z}_i, \widetilde{Z}^i_j)
%\end{equation}
\hide{
\begin{equation}
    \min \sum_{s=1}^{S}\sum_{i=1}^n \mathcal{D}(\overline{\mathbf{Z}}_i, \widetilde{\mathbf{Z}}^{(s)}_i).
\end{equation}
}
%\reminder{modified, please check it}
However, the distribution center calculated in this way is always inclined to have higher entropy values, indicating greater ``uncertainty''. 
%Thus this method will enforce model to give uncertain predictions to unlabeled data. 
Consequently, it will bring extra uncertainty into the model's predictions. 
To avoid this problem, we utilize the \textit{label sharpening} trick here.
%for \model. 
Specifically, we apply a sharpening function onto the averaged label distribution to reduce its entropy~\cite{berthelot2019mixmatch}, i.e.,
\begin{equation}
\small
\label{equ:sharpen}
% \overline{\mathbf{Z}}^{'}_{ik} = \frac{\overline{\mathbf{Z}}_{ik}^{\frac{1}{T}}}{\sum_{j=0}^{C-1}\overline{\mathbf{Z}}_{ij}^{\frac{1}{T}}},
\overline{\mathbf{Z}}^{'}_{ik} = \overline{\mathbf{Z}}_{ik}^{\frac{1}{T}} ~\bigg/~\sum_{j=0}^{C-1}\overline{\mathbf{Z}}_{ij}^{\frac{1}{T}},
\end{equation}
where $0< T\leq 1$ acts as the ``temperature'' that controls the sharpness of the categorical distribution. 
As $T \to 0$, the sharpened label distribution will approach a one-hot distribution. 
To substitute $\overline{\mathbf{Z}}_i$ with $\overline{\mathbf{Z}}^{'}_i$, we minimize the distance between  $\widetilde{\mathbf{Z}}_i$ and $\overline{\mathbf{Z}}^{'}_i$ in \model:
\begin{equation}
\small
\label{equ:consistency}
    \mathcal{L}_{con} =   \frac{1}{S}\sum_{s=1}^{S}\sum_{i=0}^{n-1} \mathcal{D}(\overline{\mathbf{Z}}^{'}_i, \widetilde{\mathbf{Z}}^{(s)}_i).
\end{equation}

Therefore, by setting $T$ as a small value, we can enforce the model to output low-entropy predictions. 
This can be viewed as adding an extra entropy minimization regularization into the model, which assumes that the classifier's decision boundary should not pass through high-density regions of the marginal data distribution~\cite{grandvalet2005semi}. 

As for the distance function $\mathcal{D}(\cdot, \cdot)$, we adopt the squared $L_2$ loss $\mathcal{D}(\mathbf{a}, \mathbf{b})=\|\mathbf{a}-\mathbf{b}\|^2$ in our model. 
Recent studies have demonstrated that it is less sensitive to incorrect predictions~\cite{berthelot2019mixmatch} and thus is more suitable for the semi-supervised setting than cross-entropy.

\vpara{Semi-supervised Training and Inference.}
In each epoch, we employ both the supervised classification loss in Eq. \ref{equ:loss} and the consistency regularization loss in Eq. \ref{equ:consistency} on $S$ augmentations. 
Hence, the final loss of \model\ is:
\begin{equation}
\small
\label{equ:inf}
	\mathcal{L} = \mathcal{L}_{sup} + \lambda \mathcal{L}_{con},
\end{equation}
where $\lambda$ is a hyper-parameter that controls the balance between the supervised classification and consistency regularization losses.
%The model parameters are updated by gradient descent with Adam updating rule\cite{kingma2014adam}. 

During inference, as mentioned in Section \ref{sec:randpro}, we directly use the original feature $\mathbf{X}$ without DropNode for propagation. 
This is justified because we scaled  the perturbed feature matrix $\widetilde{\mathbf{X}}$ during training to guarantee its expectation to match $\mathbf{X}$. 
Hence the inference formula is:
\begin{equation}
\small
\label{equ:inference}
% 	\widetilde{Z} = \text{submodel}\left(\frac{1}{k}\sum_i^k\hat{A}^i X\right)
% \widetilde{Z}
\mathbf{Z}= P\left(\mathbf{Y}~\bigg|~\frac{1}{K+1}\sum_{k=0}^K\hat{\mathbf{A}}^k \mathbf{X};\hat{\Theta}\right),
\end{equation}
where $\hat{\Theta}$ denotes the optimized parameters after training. 
Algorithm \ref{alg:2} outlines \model's training process.

\subsection{Complexity Analysis} 
\model\ comprises of random propagation and consistency regularized training. 
For random propagation, we compute $\overline{\mathbf{X}}$ by iteratively calculating the product of $\hat{\mathbf{A}}^k$ and $\widetilde{\mathbf{X}}$, and its time complexity is $\mathcal{O}(Kd(n+|E|))$. 
The complexity of its prediction module (2-layer MLP) is $\mathcal{O}(nd_h(d+ C))$, where $d_h$ denotes its hidden size. 
By applying consistency regularized training, the total computational complexity of \model\ is $\mathcal{O}(S(Kd(n + |E|)+ nd_h(d + C))$, \textbf{which is linear with the sum of node and edge counts.
}
%\yd{to change time complexity structure}

}

\hide{

\section{Graph Random Networks}
\hide{
\begin{figure*}
  		\centering
  		\includegraphics[width=0.8 \linewidth,height =  0.25\linewidth]{grand2.pdf}
 	\caption{Diagram of the training process of \model.  
 	As a general semi-supervised learning framework on graphs, \model\ provides two mechanisms---random propagation and consistency regularization----to enhance the prediction model's robustness and generalization. In each epoch, given the input graph, \model\ first generates $S$ graph data augmentations $\{\widetilde{X}^{(s)}| 1\leq s \leq S\}$ via random propagation layer. Then, each $\widetilde{X}^{(s)}$ is fed into a prediction model, leading to the corresponding prediction distribution $\widetilde{Z}^{(s)}$. In optimization, except for optimizing the supervised classification loss $\mathcal{L}^l$ with the given labels $Y^L$, we also minimize the prediction distance among different augmentations of unlabeled nodes via unsupervised consistency loss $\mathcal{L}^u$. }
%  	The Architecture of \model. In \model, we will first introduce an efficient graph disentanglement method by simple random graph sampling, coupled with graph propagation to generate disentangled data, which have little loss of performance compared with the original feature data in GCNs. Hence random propagation layer also offers a general way for graph data augmentation. To further exploit unlabeled data, we introduce random consistency loss to consider the prediction consistency of unlabeled data between different random augmentation rounds.\reminder{*}} 
  	\label{fig:model}
\end{figure*}
}
To achieve a better model for semi-supervised learning on graphs, we propose Graph Random Networks (\model). In \model , different with other GNNs, each node is allowed to randomly interact with different subsets of neighborhoods in different training epochs. This random propagation mechanism is shown to be an economic way for stochastic graph data augmentation. Based on that, we also design a consistency regularized training method to improve model's generalization capacity by encouraging predictions invariant to different augmentations of the same node. %That makes  improve model's generalization capacity under semi-supervised setting.
%This will further improve 
%data into consideration and dramatically improve model's generalization capacity. 
%these aims 
%at once
 %via \textit{random propagation layer} and \textit{consistency regularization}. 

\hide{
Graph is complex with highly tangled nodes and edges, while previous graph models, e.g., GCN and GAT, take the node neighborhood as a whole and follow a determined aggregation process recursively. These determined models, where nodes and edges interact with each other in a fixed way, suffer from overfitting and risk of being misguiding by small amount of potential noise, mistakes, and malevolent attacks. To address it, we explore graph-structured data in a random way. In the proposed framework, we train the graph model on massive augmented graph data, generated by random sampling and propagation. In different training epochs, nodes interact with different random subsets of graph, and thus mitigate the risk of being misguiding by specific nodes and edges. On the other hand, the random data augmentation alleviates the overfitting and the implicit ensemble style behind the process improves the capacity of the representation. As the model should generalize well and have a similar prediction on unlabeled data on different data augmentations, despite the potential divergence from sampling randomness, we can further introduce an unsupervised graph-based regularization to the framework.
}

\hide{
In a graph neural model, if any given node's  representation (including unlabeled nodes') has a sufficient and consistent performance in any random subgraphs, that is, any node representation disentangles with specific nodes or edge links, which can be swapped by others, the graph neural model can effectively alleviate the  overfitting to specific graph structure and being misguided by noise, mistakes, and malevolent attacks. One the other hand, the graph neural model trained on exponentially many random subgraphs will also explore graph structures in different levels as sufficiently as possible, while previous models, e.g., GCN and GAT, model the neighborhood as a whole. Based on it, the proposed model, \model, utilizes a series of random sampling strategies to explore graph structures. We found that sampling strategy coupled with graph propagation is an efficient way of graph data augmentation. By leveraging the consistency of abundant unlabeled data on different random augmentations, we can further lower the generalization loss.
}

%have a expressive capacity of capturing complex graphs. Hence, \model\ features graph random sampling for data disentanglement and leverage of unlabeled data.
%
\hide{
\subsection{Over-smoothing Problem }
The over-smoothing issue of GCNs was first studied in \cite{li2018deeper}, which indicates that node features will converge to a fixed vector as the network depth increases. This undesired convergence heavily restricts the expressive power of deep GCNs. Formally speaking, suppose $G$ has $P$ connected components $\{C_i\}^{P}_{i=1}$. Let $\mathbbm{1}^{(i)} \in \mathbb{R}^n$ denote the indication vectors for the $i$-th component $C_i$, which indicates whether a node belongs to $C_i$  i.e.,
	\begin{equation*}
\mathbbm{1}_j^{(i)} = \left\{ \begin{aligned}
1, &v_j \in C_i \\
0, &v_j \notin C_i  \end{aligned}\right..
\end{equation*}
 The over-smoothing phenomenon of GCNs is formulated as the following theorem:
\begin{theorem}
	Given a graph $G$ which has $P$ connected components $\{C_i\}^{P}_{i=1}$, for any $\mathbf{x} \in \mathbb{R}^n$, we have:
	\begin{equation*}
		\lim_{k\rightarrow +\infty} \hat{A}^k x = \widetilde{D}^{\frac{1}{2}}[\mathbbm{1}^{(1)}, \mathbbm{1}^{(2)}, ..., \mathbbm{1}^{(P)}] \hat{x}
	\end{equation*}
	where $\hat{x} \in \mathbb{R}^P$ is a vector associated with $x$.  %$\mathbbm{1}^{(i)} \in \mathbb{R}^n$ is the indication vector of the $i^{th}$ component $C_i$, which denotes whether a node belongs to $C_i$, i.e., 

\end{theorem}
From this theorem, we could notice that after repeatedly performing the propagation rule of GCNs many times, node features will converge to a linear combination of $\{D^{\frac{1}{2}}\mathbbm{1}^{(i)}\}$. And the nodes in the same connected component only distinct by their degrees. In the above analyses, the activation function used in GCNs is assumed to be a linear transformation, but the conclusion can also be generalized to non-linear case\cite{oono2019graph}.
%收敛后同一个component的feature只跟degree 有关
}

\begin{figure*}
    \centering
    \includegraphics[width= \linewidth]{grand6.pdf}
    \caption{Architecture of \model.}
    \label{fig:arch1}
\end{figure*}

\subsection{Architecture}
The architecture of \model\ is illustrated in Figure~\ref{fig:arch1}. Overall, the model includes two components, i.e., Random propagation module and classification module.

\subsubsection{Random Propagation Module.}
\label{sec:randpro}
\begin{figure}
    \centering
    \includegraphics[width=0.8
    \linewidth]{dropnode_vs_dropout2.pdf}
    \caption{Difference between DropNode and dropout. Dropout drops  elements of $X$ independently. While dropnode drops feature vectors of nodes (row vectors of $X$) randomly.}
    \label{fig:dropnode_vs_dropout}
\end{figure}

\begin{algorithm}[tb]
\caption{Dropnode}
\label{alg:dropnode}
\begin{algorithmic}[1] 
\REQUIRE ~~\\
Original feature matrix $X \in R^{n \times d}$, DropNode probability 
$\delta \in (0,1)$. \\
\ENSURE ~~\\
Perturbed feature matrix  $X\in R^{n \times d}$.
\IF{mode == Inference}
\STATE $\widetilde{X} = X$.
\ELSE
\STATE Randomly sample $n$ masks: $\{\epsilon_i \sim Bernoulli(1-\delta)\}_{i=1}^n$.
%$\mathcal{L}^l + \lambda \mathcal{L}^u$
%\STATE 
\STATE Obtain deformity feature matrix by  multiplying each node's feature vector with the corresponding  mask: $\widetilde{X}_{i} = \epsilon_i \cdot X_{i} $.
\STATE Scale the deformity features: $\widetilde{X} = \frac{\widetilde{X}}{1-\delta}$.
\ENDIF
\end{algorithmic}
\end{algorithm}

\label{sec:randpro}
%In this section, we will first introduce our motivation to adopt random sampling on graph-structured data. The divergence from randomness of sampling also inspires us to constrain unlabeled data with consistency regularization. Then we detail the two modules in \model, 

%The proposed model, \model, provides random propagation layer and consistency regularization for a common graph model supervised with scarce labels.
%The training process of \model\ in a certain epoch is illustrated in Figure~\ref{fig:model}. 
%The basic idea of random propagation is to let each node rand
%we random sample a fixed proportion of neighborhoods for each node $v_i \in \mathcal{V}$, and let $v_i$ only interact with the sampled neighborhoods in propagation.
In random propagation, we aim to perform message passing in a random way during model training. To achieve that, we add an extra node sampling operation called ``DropNode'' in front of the propagation layer.
%we randomly ``drop'' some nodes before propagation. This strategy is called dropnode.

\vpara{DropNode.} In DropNode, feature vector of each node is randomly removed (rows of $X$ are randomly set to $\vec{0}$) with a pre-defined probability $\delta \in (0,1)$ at each training epoch. The resultant perturbed feature matrix $\widetilde{X}$ is fed into the propagation layer later on.
More formally, we first randomly sample a binary mask $\epsilon_i \sim Bernoulli(1-\delta)$ for each node $v_i$, and obtain the perturbed feature matrix $\widetilde{X}$ by multiplying each node's feature vector with the corresponding mask, i.e., $\widetilde{X}_i=\epsilon_i \cdot X_i$. Furthermore, we scale $\widetilde{X}$ with a factor of $\frac{1}{1-\delta}$ to guarantee the perturbed feature matrix is equal to $X$ in expectation. Please note that the sampling procedure is only performed during training. In inference time, we directly let $\widetilde{X}$ equal to the original feature matrix $X$. The algorithm of DropNode is shown in Algorithm \ref{alg:dropnode}. 

DropNode is similar to dropout\cite{srivastava2014dropout}, a more general regularization method used in deep learning to prevent overfitting. 
In dropout, the elements of $X$ are randomly dropped out independently. While DropNode drops a node's whole features together, serving as a node sampling method. 
Figure \ref{fig:dropnode_vs_dropout} illustrates their differences.
Recent study suggests that a more structured form of dropout, e.g., dropping a contiguous region in DropBlock\cite{ghiasi2018dropblock}, is required for better regularizing CNNs on images. From this point of view, dropnode can be seen as a structured form of dropout on graph data.
%Similar idea has been proposed 
We also demonstrate that dropnode is more suitable for graph data than dropout both theoretically (Cf. Section \ref{sec:theory}) and experimentally (Cf. Section \ref{sec:ablation}). 
%Similar idea, i.e., DropBlock\cite{ghiasi2018dropblock} has been proposed as a more effective regularization for CNNs, which shows more effective than original dropout. 

%On the other hand, dropnode can be seen as a structured from of dropout for graph data. 

After dropnode, the perturbed feature matrix $\widetilde{X}$ is fed into a propagation layer to perform message passing. Here we adopt mixed-order propagation, i.e.,

\begin{equation}
\label{equ:kAX}
	\overline{X} = \overline{A} \widetilde{X}.
\end{equation}
Here we define the propagation matrix as $\overline{A} =  \frac{1}{K+1}\sum_{k=0}^K\hat{A}^k$, that is, the average of the power series of $\hat{A}$ from order 0 to order $K$. This kind of propagation rule enables model to incorporate multi-order neighborhoods information, and have a lower risk of over-smoothing compared with using $\hat{A}^K$ when $K$ is large. Similar ideas have been adopted in previous works~\cite{abu2019mixhop,abu2018n}. We compute the Equation \ref{equ:kAX} by iteratively calculating the product of sparse  matrix $\hat{A}^k$ and $\widetilde{X}$. The corresponding time complexity is $\mathcal{O}(Kd(n+|E|))$.

\hide{
In other words, the original features of about $\delta |V|$ nodes are removed (set as $\vec{0}$).
% , and $\delta$ controls the extent of disentanglement ($\delta \sim 0.5$). 
Obviously, this random sampling strategy may destroy the information carried in nodes and the resultant corrupted feature matrix is insufficient for prediction. To compensate it, we try to recover the information in a graph signal propagation process, and get the recovery feature matrix $\widetilde{X}$. The propagation recovery process is:
\begin{equation}
\label{equ:kAX}
	\widetilde{X} = \frac{1}{K+1}\sum_k^K\hat{A}^k X^{'}.
\end{equation}

%To guarantee the expectation of $X'$ equal to original features $X$, we normalize $X'$ as $\frac{X'}{1-\delta}$.
%Please note that the mask for each node is sampled independently from all of others.
It also offers a general way for implicit model ensemble on exponentially many augmented data.
}
 %Consistency regularization further leverages unlabeled data to prevent the potential divergence of predictions from different augmentations.
\hide{
In each epoch, given the input graph, \model\ first generates $S$ graph data augmentations $\{\widetilde{X}^{(s)}| 1\leq s \leq S\}$ via random propagation layer. Then, each $\widetilde{X}^{(s)}$ is fed into a prediction model, leading to the corresponding prediction distribution $\widetilde{Z}^{(s)}$. In optimization, except for optimizing the supervised classification loss $\mathcal{L}^l$ with the given labels $Y^L$, we also minimize the prediction distance among different augmentations of unlabeled nodes via unsupervised consistency loss $\mathcal{L}^u$.}

%In random propagation layer, we augment the graph data by random sampling and propagation $S$ times. With the set of graph data augmentations 
%and train a combination of prediction models on the augmented graph data. 
%Briefly, random propagation layer consists of graph dropout and graph propagation. 
%It also offers a general way for implicit model ensemble on exponentially many augmented data. Unsupervised consistency regularization further leverages unlabeled data to prevent the potential divergence of predictions from different augmentations.
%To be specific, in each epoch we drop partial nodes or edges and recover the graph information by propagation multiple times and minimize the prediction distance of all nodes on different augmented data via consistency regularization. In the inference phase, we just input the original $X$ without sampling.
% Finally, we analyze graph dropout and the consistency constraint on graphs theoretically. The training process of \model\ in a certain epoch is illustrated in Figure~\ref{fig:model} and we will detail the model below.

%To further exploit unlabeled data, we introduce random consistency loss to consider the prediction consistency of unlabeled data between different random augmentation rounds.

\hide{
\subsection{Motivation}
Our motivation is that GCNs can not sufficiently leverage unlabeled data in this task. Specifically,

\subsection{Overview of Graph Random Network}
    In this section, we provide a brief introduction of the proposed graph random network. The basic idea of \model is to promote GCNs' generalization and robustness by taking the full advantage of unlabeled data in the graph.

%\subsection{Random Propagation Layer}
%
% Randomly sampling part of graph nodes may be one of the most straightforward method for graph disentanglement.
Random propagation layer consists of graph dropout and propagation.
We first introduce how to generate a subgraph by random sampling nodes.
Formally, we randomly sample the nodes without replacement at a probability $1-\delta$  and drop the rest.
%, as set $\mathcal{C}$. % and drop out the rest. We denote the chosen node set as $\mathcal{C}$. 
The deformity feature matrix $X^{'}$ is formed in the following way,
\begin{align}
\label{equ:samplingX1}
\left\{
\begin{aligned}
	& Pr(X^{'}_{i}=\vec{0}) = \delta, \\%, & i \notin \mathcal{C}.
	&Pr(X^{'}_{i}=X_i) = 1-\delta. %, & i \in \mathcal{C}.\\
\end{aligned}
\right.
\end{align}

In other words, the original features of about $\delta |V|$ nodes are removed (set as $\vec{0}$).
% , and $\delta$ controls the extent of disentanglement ($\delta \sim 0.5$). 
Obviously, this random sampling strategy may destroy the information carried in nodes and the resultant corrupted feature matrix is insufficient for prediction. To compensate it, we try to recover the information in a graph signal propagation process, and get the recovery feature matrix $\widetilde{X}$. The propagation recovery process is:
\begin{equation}
\label{equ:kAX}
	\widetilde{X} = \frac{1}{K}\sum_k^K\hat{A}^k X^{'}.
\end{equation}

\hide{
\begin{equation}
\label{Equ:denoise}
\widetilde{X} = \arg\min_{\widetilde{X}} \frac{1}{2}\left\| \widetilde{X}- X^{'}\right\|^2_2 + \alpha \frac{1}{2}\left \| \widetilde{X} - D^{-\frac{1}{2}}AD^{-\frac{1}{2}} \widetilde{X}\right\|^2_2
\end{equation}  
\begin{small}
% \widetilde{X}=\arg\min_{\widetilde{X}} Q(\widetilde{X}) =
\begin{equation}
\label{Equ:denoise}
\widetilde{X}=\arg\min_{\widetilde{X}} 
   \frac{1}{2}\left(\sum_{i, j=1}^{n} A_{i j}\left\|\frac{1}{\sqrt{D_{i i}}} \widetilde{X}_{i}-\frac{1}{\sqrt{D_{j j}}} \widetilde{X}_{j}\right\|^{2}+\mu \sum_{i=1}^{n}\left\|\widetilde{X}_{i}-X^{'}_{i}\right\|^{2}\right)
\end{equation}  
\end{small}
% \reminder{
% https://papers.nips.cc/paper/2506-learning-with-local-and-global-consistency.pdf}
where the first term keeps the denoised signal similar to the measurement and the second term enforces the smoothing of the solution. By setting the derivate of $\widetilde{X}$ to zero, we can obtain the solution to the Equation \ref{Equ:denoise}:
\begin{equation}
\label{Equ:denoise2}
\widetilde{X} = (I+\alpha \tilde{L})^{-1}X^{'}
\end{equation}

To avoid the inversion operation in Equation \ref{Equ:denoise2}, we derive an approximate solution of Equation \ref{Equ:denoise2} using Taylor Extension:
\begin{equation}
\label{Equ:denoise3}
\widetilde{X} = (I + \alpha \tilde{L}^2)^{-1}X^{'} \approx \sum_{i=0}^K (-\alpha\tilde{L}^2)^i X^{'} 
\end{equation}
}

Combining Equation \ref{equ:samplingX1} and \ref{equ:kAX}, the original feature matrix $X$ is firstly \hide{heavily} corrupted by sampling and then smoothed by graph propagation. In fact, the whole procedure consisting of two opposite operations provides a new representation $\widetilde{X}$ for the original $X$. Later we will verify that $\widetilde{X}$ is still a good representation substitute with a sufficient prediction ability, while the randomness inside the process helps the model avoid the over-dependence on specific graph structure and 
explore different graph structure in a more robust way.
%as sufficiently as possible. 
We treat these opposite operations as a whole procedure and call them \textit{random propagation layer} in the proposed model \model. 
\hide{After filtered by random sampling operation in random propagation layer, the subsequent model parts only involve interaction among nearly half of the original node features and thus graph-structure data are disentangled somehow.}

Here, we generalize the random propagation layer with another sampling strategies. Instead of removing the feature of a node entirely by sampling once, we drop the element of the feature in each dimension one by one by multiple sampling. This is a multi-channel version of node sampling. Furthermore, we scale the remaining elements with a factor of $\frac{1}{1-\delta}$ to 
 guarantees the deformity feature matrix or adjacency matrix are the same as the original in expectation. Similar idea of dropping and rescaling has been also used in Dropout~\cite{srivastava2014dropout}. Hence we called our three graph sampling strategy graph dropout (\textit{dropnode} and \textit{dropout}), and formulate them following:
 
\vpara{Dropnode.} In dropnode, the feature vector of each node is randomly dropped with a pre-defined probability $\delta \in (0,1)$ during the propagation. More formally, we first form a deformity feature matrix $X^{'}$ in the following way:
\begin{align}
\label{equ:nodedropout}
\left\{
\begin{aligned}
& Pr(X^{'}_{i}=\vec{0}) = \delta,& \\
&Pr(X^{'}_{i}= \frac{X_i}{1-\delta} ) = 1- \delta. &
\end{aligned}
\right.
\end{align}
Then let $X^{'}$ propagate in the graph via Equation~\ref{equ:kAX}. 

\vpara{Dropout.} In dropout, the element of the feature matrix is randomly dropped with a probability $\delta \in (0,1)$ during the propagation. Formally, we have:
\begin{align}
\label{equ:featuredropout}
\left\{
\begin{aligned}
& Pr(X^{'}_{ij}=0) = \delta,& \\
&Pr(X^{'}_{ij}= \frac{X_{ij}}{1-\delta} ) = 1- \delta. &
\end{aligned}
\right.
\end{align}
Then we  propagate $X^{'}$ in the graph via Equation~\ref{equ:kAX}. 

%Then we use $\hat{A}^{'}$ as the substitute of $\hat{A}$ in the propagation Equation~\ref{equ:kAX}.  

Note that dropout in graph dropout shares the same definition of Dropout~\cite{srivastava2014dropout}, a more general method widely used in deep learning to prevent parameters from overfitting. However, dropout in graph dropout, as a multi-channel version of dropnode, is mainly developed to explore graph-structured data in the semi-supervised learning framework. The graph dropout strategy, directly applied to graph objects, is also an efficient way of graph data augmentation and  model ensemble on exponentially many subgraphs. As for the common dropout applied to the input units, the optimal probability of drop is usually closer to 0 to prevent the loss of information~\cite{srivastava2014dropout}, which is not the case in our graph dropout. In this paper we will further analyze theoretically that graph dropout help \model\ leverage unlabeled data, and dropnode and dropout play different roles. 
}
%For the convenience of narration, in the rest of the paper, we mainly discuss dropnode and similar conclusion holds for dropout and dropedge. 

\hide{
Note that dropout is applied to neural network activation to prevent overfitting of parameters while our node dropout and edge dropout are coupled with graph propagation, performing on graph objects without parameters. We will reveal that our graph dropout strategy is an efficient way of graph data augmentation and graph models ensemble on exponentially many subgraphs. We will further analyse theoretically that graph dropout help \model\ leverage unlabeled data. \reminder{}
}

\vpara{Stochastic Graph Data Augmentation.}
%Here, we will reveal that random propagation layer is an efficient method of stochastic graph data augmentation.
 Feature propagation have been proven to be an effective method for enhancing node representation by aggregating information from neighborhoods, and becomes the key step of Weisfeiler-Lehman Isomorphism test~\cite{shervashidze2011weisfeiler} and GNNs. In random propagation, some nodes are first randomly dropped before propagation. That is, each node only aggregates information from a subset of multi-hop neighborhoods randomly. In this way, random propagation module stochastically generates different representations for a node, suggesting an efficient stochastic augmentation method for graph data. From another point of view, random propagation can also be seen as injecting random noise to the propagation procedure by dropping a portion of nodes . 

We also conduct an interesting experiment to examine the validation of this data augmentation method. In this setting, we first sample a set of augmented node representations $\overline{X}$ from random propagation module with different drop rates, then use $\overline{X}$ to train a GCN for node classification. 
By examining the classification accuracy on $\overline{X}$, we can empirically analyze the influence of the injected noise on model performance. The results with different drop rates have been shown in Figure \ref{fig:redundancy}. It can be seen that the performance only decreases by less than $3\%$ as the drop rate increases from $0$ to $0.5$, indicating the augmented node representations is still sufficient for prediction\footnote{In practice, the drop rate is always set to $0.5$}. 
 Though $\overline{X}$  is still inferior to $X$ in performance, in practice, multiple augmentations will be utilized for training prediction model since the sampling will be performed in every epoch. From the view of bagging~\cite{breiman1996bagging}, this random data augmentation scheme makes the final prediction model implicitly ensemble models on exponentially many augmentations, yielding much better performances than using deterministic propagation.

%training graph models on $\widetilde{X}$ with multiple sampling augmentations utilizes all the feature information in statistics. 

%The augmented data should preserve a sufficient ability in prediction.
%We empirically show the recovery feature matrix $\widetilde{X}$ with nearly half of the information removed is still sufficient for prediction. Figure \ref{fig:redundancy} shows the node classification  accuracy of GCN on three public datasets, i.e., Cora, Citeseer and Pubmed,  with the feature matrix $\widetilde{X}$ in both training and testing phases. 
\begin{figure}
  		\centering
  		\includegraphics[width = 0.8 \linewidth, height =  0.58\linewidth]{drop_rate.pdf}
  	\caption{Classification results of GCN with $\overline{X}$ as input.} 
  	\label{fig:redundancy}
\end{figure}
  
%The sufficient prediction ability of $\widetilde{X}$ inspires us to consider random propagation layer as an economic strategy for graph data augmentation, as multiple random sampling behaviors generate exponentially many $\widetilde{X}$s.
% and we can achieve exponentially many $\widetilde{X}$s. 

% generated with only nearly half of the feature information 
 %is still inferior to $X$ in performance, training graph models on $\widetilde{X}$ with multiple sampling augmentations utilizes all the feature information in statistics. 

%More importantly, this graph data augmentation also guarantees the robustness of the graph model, avoid being misguided by data noises and attacks. Graph model trained on arbitrary combination of nodes and edges has a consistent performance, hence the model reduces the dependency on specific graph structure, as nodes and edges are replaceable. Besides, 

\subsubsection{Prediction Module.}
In prediction module, the augmented feature matrix $\overline{X}$ is fed into a neural network to predict nodes labels. We employ a two-layer MLP as the classifier:
\begin{equation}
\label{equ:mlp}
    P(\mathcal{Y}|\overline{X};\Theta) = \sigma_2(\sigma_1(\overline{X}W^{(1)})W^{(2)})
\end{equation}
where $\sigma_1$ is ReLU function, $\sigma_2$ is softmax function, $\Theta=\{W^{(1)}, W^{(2)}\}$ refers to model parameters. Here the classification model can also adopt more complex GNN based node classification models, e.g., GCN, GAT.  But we find the performance decreases when we replace MLP with GNNs because of the over-smoothing problem. We will explain this phenomenon in Section \ref{sec:oversmoothing}.

%with different drop ratios . We found, surprisingly, that the performance only decreases by less than $3\%$ when we drop $50\%$ nodes.

% that is, only half of the nodes interact with each other, at a specific sampling epoch. 
\hide{
From another perspective, Equation \ref{equ:samplingX1} and \ref{equ:kAX} in the random propagation layer perform a data augmentation procedure by linear interpolation~\cite{devries2017dataset} in a random subset of the multi-hop neighborhood.
}

\hide{
In a graph, we random sample a fixed proportion of neighborhoods for each node $v_i \in \mathcal{V}$, and let $v_i$ only interact with the sampled neighborhoods in propagation. Using this method, we are equivalent to let each node randomly perform linear interpolation with its neighbors. 

However, generating neighborhood samples for each node always require a time-consuming preprocessing. For example, the sampling method used in GraphSAGE~\cite{hamilton2017inductive} requires $k \times n$ times of sampling operations, where $k$ denotes the size of sampled neighbor set. To solve this problem, here we propose an efficient sampling method to perform random propagation --- graph dropout. Graph dropout is inspired from Dropout\cite{srivastava2014dropout}, a widely used regularization method in deep learning. \reminder{graph dropout is an efficient sampling method}
In graph dropout, we randomly drop a set of nodes or edges in propagation, which is introduced separately. \\

\vpara{Edge Dropout.} The basic idea of edge dropout is randomly dropping a fix proportion of edges during each propagation. 
Specifically, we construct a deformity feature matrix $\hat{A}^{'}$ following:
\begin{align}
\label{equ:nodedropout}
\left\{
\begin{aligned}
& Pr(\hat{A}^{'}_{ij}=\vec{0}) = \delta.& \\
&Pr(\hat{A}^{'}_{ij}= \frac{\hat{A}^{'}_{ij}}{1-\delta} ) = 1- \delta. &
\end{aligned}
\right.
\end{align}
Then we use $\hat{A}^{'}$ as the replacement of $\hat{A}$ in propagation.

\vpara{Node Dropout.} In node dropout, the feature vector of each node is randomly dropped with a pre-defined probability $\delta \in (0,1)$ during propagation. More formally, we first form a deformity feature matrix $X^{'}$ in the following way:
\begin{align}
\label{equ:nodedropout}
\left\{
\begin{aligned}
& Pr(X^{'}_{i}=\vec{0}) = \delta.& \\
&Pr(X^{'}_{i}= \frac{X_i}{1-\delta} ) = 1- \delta. &
\end{aligned}
\right.
\end{align}
Then let $X^{'}$ propagate in graph as the substitute of $X$, i.e., $\tilde{X} = \hat{A}X$. 

Actually, node dropout can be seen as a special form of edge dropout:
}
%\reminder{Dropping a node is equivalent to drop all the edges start from the node.} \\
%
%\reminder{connection between graph dropout and feature dropout. dropblock..}

\subsection{Consistency Regularized Training}
\label{sec:consis}
\begin{figure*}
	\centering
	\includegraphics[width= \linewidth]{grand_consis.pdf}
	\caption{Illustration of consistency regularized training for \model. \yd{how about making the random progagation module into two as well?}}
	\label{fig:arch2}
\end{figure*}
 %By training model on different augmentations of labeled examples, we are equivalent to add an implicit consistency regularization on labeled examples.
  As mentioned in previous subsection, random propagation module can be seen as an efficient method of stochastic data augmentation. That inspires us to design a consistency regularized training algorithm for optimizing parameters. 
  In the training algorithm, we generate multiple data augmentations at each epoch by performing dropnode multiple times. Besides the supervised classification loss, we also add a consistency regularization loss to enforce model to give similar predictions across different augmentations. This algorithm is illustrated in Figure \ref{fig:arch2}.
  %Algorithm \ref{alg:2} summarizes the training algorithm.
  
  %In the training phase, we aim to leverage consistency regularization to facilitate model's generalization by enforcing model to give similar predictions among diffident augmentations.  
 \begin{algorithm}[tb]
\caption{Consistency Regularized Training for \model}
\label{alg:2}
\begin{algorithmic}[1] %[1] enables line numbers
\REQUIRE ~~\\
 Adjacency matrix $\hat{A}$,
labeled node set $V^L$,
 unlabeled node set $V^U$,
feature matrix $X \in R^{n \times d}$, 
%$T$: maximum number of learning epochs,\\
% $K$: times of random propagation each epoch, \\
times of dropnode in each epoch $S$, dropnode probability $\delta$.\\
\ENSURE ~~\\
Prediction $Z$.
%\STATE $\pi^{(0)} \leftarrow \emptyset$, $\pi^{(1)} \leftarrow \emptyset$
%\STATE $\Omega\leftarrow U$
\WHILE{not convergence}
\FOR{$s=1:S$} 
%\STATE Obtaining the deformity feature matrix $X^{'}$ by applying graph dropout on $X$ via Eq.\ref{equ:nodedropout} or Eq.\ref{equ:featuredropout}.
\STATE Apply dropnode via Algorithm \ref{alg:dropnode}: $
\widetilde{X}^{(s)} \sim \text{dropnode}(X,\delta)$. 
\STATE Perform propagation: $\overline{X}^{(s)} = \frac{1}{K}\sum_{k=0}^K\hat{A}^k \widetilde{X}^{(s)}$.
\STATE Predict class distribution using MLP: %$\widetilde{Z}^s = \text{GCN}(\widetilde{X}_k, A)$
$\widetilde{Z}^{(s)} = P(\mathcal{Y}|\overline{X}^{(s)};\Theta)$.
\ENDFOR
% \STATE $Z^j_{(t)} =\text{GCN}^j(\widetilde{X}^j_{(t)})$
%\STATE $Z^{(j)}_{(t)} =\text{GCN}^{(j)}(S^{(j)}[t])$
%\STATE Perform SGD on $Loss^{(j)}(Y^{(j)},Z^{(j)}_{(t)},L^{(j)})$ %to update GCN$^{j}$
%\STATE Feed $\widetilde{X}^{(j)}_{(t)}$ into $\text{GCN}^{(j)}$
\STATE Compute supervised classification loss $\mathcal{L}_{sup}$ via Equation \ref{equ:loss} and consistency regularization loss via Equation \ref{equ:consistency}.
\STATE Update the parameters $\Theta$ by gradients descending:
$$\nabla_\Theta \mathcal{L}_{sup} + \lambda \mathcal{L}_{con}$$
%\STATE 
\ENDWHILE
\STATE Output prediction $Z$ via Equation \ref{equ:inference}
\end{algorithmic}
\end{algorithm}

%The algorithm is shown in Algorithm \ref{alg:2}.
 \subsubsection{$S-$augmentation Prediction}
At each epoch, we aim to generate $S$ different data augmentations for graph data $X$. To achieve that, we adopt $S-$sampling dropnode strategy in random propagation module. That is, we performing $S$ times of random sampling in dropnode to generate $S$ perturbed feature matrices $\{\widetilde{X}^{(s)}|1\leq s \leq S\}$. Then we propagate these features according to Equation \ref{equ:kAX} respectively, and hence obtain $S$ data augmentations $\{\overline{X}^{(s)}|1 \leq s \leq S\}$. Each of these augmented feature matrice are fed into the prediction module to get the corresponding output:

%The prediction probability on the $s^{th}$ augmented data is denoted as:
 \begin{equation}
     \widetilde{Z}^{(s)} = P(\mathcal{Y}|\overline{X}^{(s)}; \Theta).
 \end{equation}
Where $\widetilde{Z}^{(s)} \in (0,1)^{n\times C}$ denotes the classification probabilities on the $s$-th augmented data. 
%Then we feed each $\widetilde{X}^{(s)}$ into a prediction model to get the corresponding prediction probabilities:
\subsubsection{Supervised Classification Loss.}
The supervised objective of graph node classification in an epoch is the averaged cross-entropy loss over $S$ times sampling:
\begin{equation}
\label{equ:loss}
	\mathcal{L}_{sup} = -\frac{1}{S}\sum_{s=1}^{S}\sum_{i=1}^m \sum_{j=1}^{C}Y_{i,j} \ln \widetilde{Z}_{i,j}^{(s)} ,
\end{equation}

\noindent where $Y_{i,l}$ is binary, indicating whether node $i$ has the label $l$. By optimizing this loss, we can also enforce the model to output the same predictions on multiple augmentations of a labeled node.
 However, in the semi-supervised setting, labeled data is always very rare. In order to make full use of unlabeled data, we also employ consistency regularization loss in \model. 
 
 %Let $\widetilde{Z}^s = p_{model}(\mathcal{Y}|\widetilde{X}^s, \hat{A})$ denote model's prediction probabilities for the $s$-th augmented data $\widetilde{X}^s$. 
 \subsubsection{Consistency Regularization Loss.} How to optimize the consistency among $S$ augmentations of unlabeled data? 
 Let's first consider a simple case, where the random propagation procedure is performed only twice in each epoch, i.e., $S=2$. Then a straightforward method is to minimize the distributional distance between two outputs:
 \begin{equation}
 \label{equ:2d}
     \min \sum_{i=1}^n \mathcal{D}(\widetilde{Z}^{(1)}_i, \widetilde{Z}^{(2)}_i),
 \end{equation}
where $ \mathcal{D}(\cdot,\cdot)$ is the distance function. Then we extend it into multiple augmentation situation. We can first calculate the label distribution center by taking average of all distributions:
\begin{equation}
    \overline{Z} = \frac{1}{S}\sum_{s=1}^{S} \widetilde{Z}^{(s)}.
\end{equation}

And we can minimize the distributional distance between $\widetilde{Z}^{(s)}$ and $\overline{Z}$, i.e., $\min \sum_{s=1}^{S}\sum_{i=1}^n \mathcal{D}(\overline{Z}_i, \widetilde{Z}^{(s)}_i)$.
%\begin{equation}
%    \min \sum_{i=1}^{K}\sum_{j=1}^n \mathcal{D}(\bar{Z}_i, \widetilde{Z}^i_j)
%\end{equation}
\hide{
\begin{equation}
    \min \sum_{s=1}^{S}\sum_{i=1}^n \mathcal{D}(\overline{Z}_i, \widetilde{Z}^{(s)}_i).
\end{equation}
}
%\reminder{modified, please check it}
However, the distribution center calculated in this way is always inclined to have more entropy value, which indicates to be more ``uncertain''. 
%Thus this method will enforce model to give uncertain predictions to unlabeled data. 
Thus this method will bring extra uncertainty into model's predictions. To avoid this problem, We utilize the label sharpening trick in \model. Specifically, we apply a sharpening function onto the averaged label distribution to reduce its entropy:
\begin{equation}
 \overline{Z}^{'}_{i,l} = \frac{\overline{Z}_{i,l}^{\frac{1}{T}}}{\sum_{j=1}^{|\mathcal{Y}|}\overline{Z}_{i,j}^{\frac{1}{T}}},
\end{equation}
where $0< T\leq 1$ acts as the ``temperature'' which controls the sharpness of the categorical distribution. As $T \to 0$, the sharpened label distribution will approach a one-hot distribution. As the substitute of $\overline{Z}$, we minimize the distance between  $\widetilde{Z}^i$ and $\overline{Z}^{'}$ in \model:
\begin{equation}
\label{equ:consistency}
    \mathcal{L}_{con} =   \frac{1}{S}\sum_{s=1}^{S}\sum_{i=1}^n \mathcal{D}(\overline{Z}^{'}_i, \widetilde{Z}^{(s)}_i).
\end{equation}

By setting $T$ as a small value, we could enforce the model to output low-entroy predictions. This can be seen as adding an extra entropy minimization regularization into the model, which assumes that classifier's decision boundary should not pass through high-density regions of the marginal data distribution\cite{grandvalet2005semi}. As for the  distance function $\mathcal{D}(\cdot, \cdot)$, we adopt squared $L_2$ loss, i.e., $\mathcal{D}(x, y)=\|x-y\|^2$, in our model. It has been proved to be less sensitive to incorrect predictions\cite{berthelot2019mixmatch}, and is more suitable to the semi-supervised setting compared to cross-entropy. 
%\reminder{K augmentation at a epoch}

\subsubsection{Training and Inference}
In each epoch, we employ both supervised classification loss (Cf. Equation \ref{equ:loss}) and consistency regularization loss (Cf. Equation \ref{equ:consistency}) on $S$ times of sampling. Hence, the final loss is:
\begin{equation}
\label{equ:inf}
	\mathcal{L} = \mathcal{L}_{sup} + \lambda \mathcal{L}_{con}.
\end{equation}
Here $\lambda$ is a hyper-parameter which controls the balance between supervised classification loss and consistency regularization loss.
%The model parameters are updated by gradient descent with Adam updating rule\cite{kingma2014adam}. 
In the inference phase, as mentioned in Section \ref{sec:randpro}, we directly use original feature $X$ as the output of dropnode instead of sampling. Hence the inference formula is:
\begin{equation}
\label{equ:inference}
% 	\widetilde{Z} = \text{submodel}\left(\frac{1}{k}\sum_i^k\hat{A}^i X\right)
% \widetilde{Z}
Z= P\left(\mathcal{Y}~\bigg|~\frac{1}{K+1}\sum_{k=0}^K\hat{A}^k X;\hat{\Theta}\right).
\end{equation}
Here $\hat{\Theta}$ denotes the optimized parameters after training. We summarize our algorithm in Algorithm \ref{alg:2}.

\hide{Note that $X$ is equal to the expectation of $\widetilde{X}$s from data augmentation by multiple sampling. From the view of model bagging~\cite{breiman1996bagging}, the final graph model implicitly aggregates models trained on exponentially many subgraphs, and performs a plurality vote among these models in the inference phase, resulting in a lower generalization loss. }
 
% The whole process is summarized in Algorithm \ref{alg:2}.
%Algorithm \ref{alg:2} summarizes the training and inference procedures of the proposed framework, \model, with random propagation layer and unsupervised consistency regularization. 
%Note that \model\ is a general framework for most of graph models, we borrow a graph node classification model, such as GCN~\cite{kipf2016semi}, MLP, as the ``submodel'' in the framework.
%with the original feature input $X$ and the classification output $Z=\text{submodel}(X) \in R^{n \times |\mathcal{Y}|}$. 
%We focus on corresponding node indices in training and testing phase. 
%\reminder{notation}
%In \model, in a certain training epoch, we perform multiple sampling and generate $S$ data augmentation. Then the data augmentations are fed into a prediction model. The probability output on the $s$-th augmented data is:
\hide{
\begin{equation}
\label{equ:model}
	\widetilde{Z}^{(s)}=p_{\text{model}}\left(\mathcal{Y}~\bigg|~\frac{1}{K}\sum_k^K\hat{A}^k X^{'}\right) \in R^{n \times |\mathcal{Y}|}.
\end{equation}

%\begin{align}
%\label{equ:model}
%\left\{
%\begin{aligned}
%\widetilde{Z}^s=p_{\text{submodel}}\left(\mathcal{Y}~\bigg|~\frac{1}{k}\sum_i^k\hat{A}^i X^{'}\right) \in R^{n \times |\mathcal{Y}|}
%\end{aligned}
%\right.
%\end{align}

Hence the supervised objective of graph node classification in an epoch is the averaged cross-entropy loss over $S$ times sampling:
\begin{equation}
\label{equ:loss}
	\mathcal{L}^l = -\frac{1}{S}\sum_{s}^{S}\sum_{i\in V^L } \sum_l^{|\mathcal{Y}|}Y_{i,l} \ln \widetilde{Z}_{i,l}^s ,
\end{equation}

\noindent where $Y_{i,l}$ is binary, indicating whether node $i$ has the label $l$. 
%The model perform a step of SGD on each augmented data.

In each epoch, we employ both supervised cross-entropy loss (Cf. Eq.\ref{equ:loss}) and unsupervised consistency loss (Cf. Eq.\ref{equ:consistency}) on $S$ times of sampling. Hence, the final loss is:
\begin{equation}
\label{equ:inf}
	\mathcal{L} = \mathcal{L}^l + \lambda \mathcal{L}^u.
\end{equation}
Here $\lambda$ is a hyper-parameter which controls the balance between supervised loss and unsupervised consistency loss.
In the inference phase, the output is achieved by averaging  over the results on exponentially many augmented test data. This can be economically realized by inputting $X$ without sampling:
\begin{equation}
\label{equ:inference}
% 	\widetilde{Z} = \text{submodel}\left(\frac{1}{k}\sum_i^k\hat{A}^i X\right)
% \widetilde{Z}
Z= p_{\text{model}}\left(\mathcal{Y}~\bigg|~\frac{1}{K}\sum_k^K\hat{A}^k X\right).
\end{equation}

Note that $X$ is the average of $X^{'}$s from data augmentation by multiple sampling. From the view of model bagging~\cite{breiman1996bagging}, the final graph model implicitly aggregates models trained on exponentially many subgraphs, and performs a plurality vote among these models in the inference phase, resulting in a lower generalization loss.  
}
% The whole process is summarized in Algorithm \ref{alg:2}.

%\input{graph_dropout_proof.tex}

\hide{
\subsection{graph propagation for smoothing}

In this section, we introduce the random propagation layer, an efficient method to perform stochastic data augmentation on the graph. We first demonstrate the propagation layer in GCNs is actually a special form of data augmentation on graph-structured data. Based on this discovery, we propose random propagation strategy, which augments each node with a part of randomized selected neighborhoods \reminder{}. Finally we show that this strategy can be efficiently instantiated with a stochastic sampling method --- graph dropout.

\subsection{Feature Propagation as Data Augmentation}
\label{sec:aug}
An elegant data augmentation technique used in supervised learning is linear interpolation \cite{devries2017dataset}. Specifically, for each example on training dataset, we first find its $K-$nearest neighboring samples in feature space which share the same class label. Then for each pair of neighboring feature vectors, we get the augmented example using interpolation:
\begin{equation}
    x^{'}_i = \lambda x_i + (1-\lambda)x_j
\end{equation}
where $x_j$ and $x_i$ are neighboring pairs with the same label $y$, $x^{'}_i$ is augmented feature vector, and $\lambda \in [0,1]$ controls degree of interpolation\reminder{}. Then the example $(x^{'}_i, y)$ is used as extra training sample to facilitate the learning task. In this way, the model is encouraged to more smooth in input space and more robust against small perturbations. \reminder{why?}
%Please note that $x_i$ and $x_j$ have the same class label. And we will directly use the same label  

As for the problem of semi-supervised learning on graphs, we assume the graph signals are smooth, i.e., \textit{neighborhoods have similar feature vectors and similar class labels}. Thus a straightforward idea of augmenting graph-structured data is interpolating node features with one of its neighborhoods' features. However, in the real network, there exist small amounts of neighboring nodes with different labels and we can't identify that for samples with non-observable labels \reminder{?}. Thus simple interpretation with one neighbor might bring uncontrollable noise into the learning framework. 
An alternative solution is interpolating samples with multiple neighborhoods, which leads to Laplacian Smoothing:

\begin{equation}
\label{equ:lap}
    %x^{'}_i = \lambda x_i + (1-\lambda) x_j \quad (x_j \in \mathcal{N}(x_i))
    x^{'}_i = (1-\lambda) x_i + \lambda \sum_j \frac{\widetilde{A}_{ij}}{\widetilde{D}_{ii}} x_j
\end{equation}.
Here we follow the definition of GCNs which adding a self-loop for each node in graph\reminder{}. Rewriting Eq.~\ref{equ:lap} in matrix form, we have:
\begin{equation}
\label{equ:lap2}
    X^{'}=  (I-\lambda I)X + \lambda \widetilde{D}^{-1} \widetilde{A}X
\end{equation}.
As pointed out by Li et.al.\cite{li2018deeper}\reminder{}, when $\lambda = 1$ and replacing $\Tilde{D}^{-1} \Tilde{A}$ with the symmetric normalized adjacency matrix $\hat{A}$, Eq.~\ref{equ:lap2} is identical to the propagation layer in GCN, i.e., $X^{'}= \hat{A}X$.

%where $\mathbf{N}(x_i)$ denotes neighborhoods set of the $i^{th}$ node.  
%In this subsection, we careful analyze propagation layer in GCN and give some insightful observations from the view of data augmentation. 
%}
}

\hide{
\subsection{Random Propagation with Graph Dropout}
With the insight from last Section~\ref{sec:aug}, we develop a random propagation strategy for stochastic data augmentation on graph.

During each time of random propagation, we random sample a fixed proportion of neighborhoods for each node $v_i \in \mathcal{V}$, and let $v_i$ only interact with the sampled neighborhoods in propagation. Using this method, we are equivalent to let each node randomly perform linear interpolation with its neighbors. 

However, generating neighborhood samples for each node always require a time-consuming preprocessing. For example, the sampling method used in GraphSAGE~\cite{hamilton2017inductive} requires $k \times n$ times of sampling operations, where $k$ denotes the size of sampled neighbor set. To solve this problem, here we propose an efficient sampling method to perform random propagation --- graph dropout. Graph dropout is inspired from Dropout\cite{srivastava2014dropout}, a widely used regularization method in deep learning. \reminder{graph dropout is an efficient sampling method}
In graph dropout, we randomly drop a set of nodes or edges in propagation, which is introduced separately. \\

\vpara{Edge Dropout.} The basic idea of edge dropout is randomly dropping a fix proportion of edges during each propagation. 
Specifically, we construct a deformity feature matrix $\hat{A}^{'}$ following:
\begin{align}
\label{equ:nodedropout}
\left\{
\begin{aligned}
& Pr(\hat{A}^{'}_{ij}=\vec{0}) = \delta.& \\
&Pr(\hat{A}^{'}_{ij}= \frac{\hat{A}^{'}_{ij}}{1-\delta} ) = 1- \delta. &
\end{aligned}
\right.
\end{align}
Then we use $\hat{A}^{'}$ as the replacement of $\hat{A}$ in propagation.

\vpara{Node Dropout.} In node dropout, the feature vector of each node is randomly dropped with a pre-defined probability $\delta \in (0,1)$ during propagation. More formally, we first form a deformity feature matrix $X^{'}$ in the following way:
\begin{align}
\label{equ:nodedropout}
\left\{
\begin{aligned}
& Pr(X^{'}_{i}=\vec{0}) = \delta.& \\
&Pr(X^{'}_{i}= \frac{X_i}{1-\delta} ) = 1- \delta. &
\end{aligned}
\right.
\end{align}
Then let $X^{'}$ propagate in graph as the substitute of $X$, i.e., $\tilde{X} = \hat{A}X$. 

Actually, node dropout can be seen as a special form of edge dropout: \reminder{Dropping a node is equivalent to drop all the edges start from the node.} \\

\reminder{connection between graph dropout and feature dropout. dropblock..}

%In stochastic propagation strategy, we randomly select a part of neighborhoods to pass message for a node. 

}

\hide{
\section{Consistency Optimization}
\subsection{\drop: Exponential Ensemble of GCNs}
Benefiting from the information redundancy, the input of \shalf gives a relatively sufficient view for node classification.
% evident from the previous section. 
It is straightforward to generalize \shalf to ensemble multiple \shalf s, as each set $\mathcal{C}$ provides a unique network view. % recovery feature matrices in different \shalf are different, thus provide different data views. 
In doing so, we can use $|\mathcal{C}_i \bigcap \mathcal{C}_j|$ to measure the correlation between two network data view $i$ and $j$. In particular, $\mathcal{C}$ and $V-\mathcal{C}$ can be considered independent. However, one thing that obstructs the direct application of traditional ensemble methods is that a network can generate exponential data views when sampling $n/2$ nodes from its $n$ nodes.

\begin{figure}{
		\centering
		\includegraphics[width = 1\linewidth]{fig_NSGCN.pdf}
		\caption{\sdrop and \dm. \drop: In each epoch we randomly drop half of the nodes and do the process of \half. The input in the inference phase is the original feature matrix without dropout, but with a discount factor of $0.5$. This method can be treated as the exponential ensemble of \half.
			\dm: In each epoch of \drop, we consider the half of the nodes that are dropped are independent to the remaining nodes and input them to another \drop. We minimize the disagreement of these two \drop s via the Jensen-Shannon divergence. This method can be treated as an economical co-training of two independent \drop s.}
		\label{fig:cotrain}
	}
\end{figure}

%Inspired by the dropout mechanism~\cite{srivastava2014dropout} that implicitly ensembles exponential number of neural networks, w
We propose a dropout-style ensemble model \drop.
In \drop, we develop a new network-sampling ensemble method---graph dropout, which samples the node set $\mathcal{C}$ from a network's node set and do the \shalf process in each training epoch. 
In the inference phase, we use the entire feature matrix with a discount factor $0.5$ as input~\cite{srivastava2014dropout}. 
%Since the network sampling is dropout-like, we call the model \textbf{\drop}. 
The data flow in \drop\ is shown in %the top part of
Figure~\ref{fig:cotrain}.

Specifically, given the recovery feature matrix $\widetilde{X}$ in a certain epoch, the softmax output of GCN in the training phase is $Z = GCN(\widetilde{X}) \in R^{n \times |\mathcal{Y}|}$. Hence the  objective of node classification in the network is

\begin{equation}
\label{equ:loss}
	Loss_{c1} = -\sum_{i\in V^L } \sum_l^{|\mathcal{Y}|}Y_{i,l} \ln Z_{i,l}
\end{equation}

\noindent where $Y_{i,l}$ is binary, indicating whether node $i$ has the label $l$.  %\todo{to jz: plz add notation definition here.; Done}

To make the model practical, in the inference phase, the output is achieved by using the entire feature matrix $X$ and a discount factor $0.5$~\cite{srivastava2014dropout}, that is, 

\begin{equation}
\label{equ:inf}
	\hat{Z} = GCN\left(\frac{0.5}{k}\sum_i^k\hat{A}^i X\right)
\end{equation}

The discount factor guarantees the input of GCNs in the inference phase is the same as the expected input in the training phase. Similar idea has been also used in~\cite{srivastava2014dropout}. 
%\todo{add the same reference above; Done}

\subsection{\dm: Disagreement Minimization}%between two \drop s }
In this section, we present the \dm\ model, which leverages the independency between the chosen node set $\mathcal{C}$ and the dropout node set $V-\mathcal{C}$ in each epoch of \drop. 
%Our intuition is to minimize the disagreement of 

In \drop, the network data view provided by the chosen set $\mathcal{C}$ is trained and the view provided by the remaining nodes $V-\mathcal{C}$ at each epoch is completely ignored. 
%On the other hand, t
The idea behind \drop\ aims to train over as many as network views as possible. 
Therefore, in \dm, we propose to simultaneously train two \drop s with complementary inputs---$\mathcal{C}$ and $V-\mathcal{C}$---at every epoch. 
In the  training phase, we improve the two GCNs' prediction ability separately by minimizing the disagreement of these two models. %We call the model \dm.

% half of the information is wasted at each epoch, but 
%However, the samples that would normally be wasted could still provide rich and independent information. Hence we simultaneously train two GCNs with their input information complementary at every epoch. In the  training phase, we improve GCNs prediction ability separately while minimizing the disagreement of these two models. We call the model \dm.

The objective to minimize the disagreement of the two complementary \sdrop at a certain epoch can be obtained by Jensen-Shannon-divergence, that is, 

\begin{equation}
\label{equ:jsloss}
	Loss_{JS} = \frac{1}{2}\sum_{i\in V} \sum_l^{|\mathcal{Y}|} \left(Z'_{i,l} \ln \frac{Z'_{i,l}}{Z''_{i,l}} + Z''_{i,l} \ln \frac{Z''_{i,l}}{Z'_{i,l}}\right)
\end{equation}

\noindent where $Z'$ and $Z''$ are the outputs of these two \drop s in the training phase, respectively. %\todo{xxx, to jz: add notation defintions here; Done}

The final objective of \sdm is

\begin{equation}
\label{equ:loss}
	Loss_{dm} = (Loss_{c1} + Loss_{c2})/2 + \lambda Loss_{JS}
\end{equation}

In the inference phase, the final prediction is made by the two GCNs together, i.e., 

\begin{equation}
\label{equ:inf}
	\hat{Z} = (\hat{Z}^{'} + \hat{Z}{''})/2
\end{equation}

Practically, the correlation of two \drop s and their individual prediction ability jointly affect the performance of \dm. 
The inputs to the two models are complementary  and usually uncorrelated, % (in fact, the relevance is also affected by the co-dependency of features), \sdm outperforms \sdrop in general.
empowering \dm\ with more representation capacity than \drop, which is also evident from the empirical results in Section \ref{sec:exp}.

\subsection{ Graph Dropout for Network Sampling}
%In previous sampling-based GCNs, such as 
In GraphSAGE~\cite{hamilton2017inductive} and FastGCN~\cite{chen2018fastgcn}, 
%graph elements (e.g., 
nodes or edges are sampled in order to accelerate training, avoid overfitting, and make the network data regularly like grids. However in the inference phase, these models may miss information when the sampling result is not representative. 

Different from previous works, % that attempt to converge to GCN, 
our \drop\ and \dm\ models are based on the network redundancy observation and the proposed dropout-like ensemble enable the (implicit) training on exponential sufficient network data views, making them outperform not only existing sampling-based methods, but also the original GCN with the full single-view~\cite{kipf2016semi}. 
This provides 
%result generates 
insights into the network sampling-based GCNs and offers a new way to sample networks for GCNs.
%Supposedly this idea may help previous network sampling-based methods have a new prospective of the effects of network sampling and explore the network information more sufficiently.

In addition, there exist connections between the original dropout mechanism in CNNs and our dropout-like network sampling technique. 
Similarly, both techniques use a discount factor in the inference phase. 
Differently, the dropout operation in the fully connected layers in CNNs results in the missing values of the activation (hidden representation) matrix $H^{(l)}$, while our graph dropout sampling method is applied in network nodes, which removes network information in a structured way. 

Notice that very recently, DropBlock~\cite{ghiasi2018dropblock} suggested that a more structured form of dropout is required for better regularizing CNNs, and from this point of view, our graph dropout based network sampling is more in line with the idea of DropBlock. In \drop\ and \dm, dropping network information, instead of removing values from the activation functions, brings a structured way of removing information from the network. 
}
%Thus it may be more suitable to call our dropout-like sampling graph dropout, though our original motivation is to use network sampling to generate exponential data views, and dropout focuses more on regularization.

%There is also similarity between the ordinary dropout mechanism and our network sampling methods, where in the inference phase both use a discount factor. The application of dropout in previous graph convolutional networks is similar to that in fully connected layers in CNNs, resulting in the missing values of the activation (hidden representation) matrix $H^{(l)}$, while our graph dropout-like sampling method is applied in network nodes, which will remove the network information in a more structured way. Very recently, DropBlock~\cite{ghiasi2018dropblock} suggests that a more structured form of dropout is needed to better regularize CNNs, and from the view of dropout, our dropout-like sampling is more like this kind of structured dropout. In graphs, Removing network nodes, instead of removing values from the activation functions, results in a more structured way of removing information from the network. Thus it may be more suitable to call our dropout-like sampling graph dropout, though our original motivation is to use network sampling to generate exponential data views, and dropout focuses more on regularization.

}%3nd of hide

\hide{
We first introduce the following Lemma with the proof in Appendix A.1. %\wz{appendix}
\begin{lemma}
\label{lemma1}

For any reserve vector $\mathbf{q}^{(n)}$, residue vector $\mathbf{r}^{(n)}$ and random walk transition vector $\mathbf{P}^{n}_s$ ($0 \leq n \leq N$), we have:
\begin{equation}
\small
\mathbf{P}^{n}_s =  (\widetilde{\mathbf{D}}^{-1}\widetilde{\mathbf{A}})^n_s = \mathbf{q}^{(n)} + \sum_{i=1}^{n}  (\mathbf{P}^i)^\mathsf{T} \cdot \mathbf{r}^{(n-i)}
\end{equation}
\hide{For any reserve vector $\mathbf{Q}^{(n)}_s$, residue vector $\mathbf{R}^{(n)}_s$ and random walk transition vector $\mathbf{P}^{n}_s$, $0 \leq n \leq N$:
\begin{equation}
\small
    \mathbf{P}^{n}_s = \mathbf{Q}^{(n)}_s + \sum_{i=0}^{n-1}  (\mathbf{P}^\mathsf{N})^i\mathbf{R}_s^{(n-1-i)}
\end{equation}
}
\end{lemma}
From Lemma~\ref{lemma1} we can derive that 
$
\mathbf{q}^{(n)}_v =  \mathbf{P}^{n}(s,v) - \sum_{i=1}^{n} (\mathbf{P}^{i})^{\mathsf{T}}_v \cdot \mathbf{r}^{(n-i)}%\leq  \mathbf{Q}^{(t)}(s, v) \leq  \mathbf{P}^t(s,v)
$. When the algorithm terminates, the elements of $\mathbf{r}^{(n-i)}$ are usually very small, thus
%\end{equation}
 %is an underestimation of $\mathbf{P}^{t}_s$, 
$\mathbf{q}^{(n)}$ can be seen as an approximation of $\mathbf{P}^n_s$. And consequently $\widetilde{\mathbf{\Pi}}_s=\sum_{n=0}^N w_n\mathbf{q}^{(n)}$ is considered as an approximation of $\mathbf{\Pi}_s$ as returned by Algorithm~\ref{alg:GFPush}.
}

\section{Experiments}
\label{sec:exp}

%In this section, we evaluate \model\  on semi-supervised graph learning benchmark datasets from various perspectives, including performance demonstration, ablation study, and the analyses of robustness, over-smoothing, and generalization. 

\subsection{Experimental Setup}
 \vpara{Baselines.}% To evaluate the effectiveness and scalability of \model, 
 In our experiments, we compare \model\ with five state-of-the-art full-batch GNNs---GCN~\cite{kipf2016semi}, GAT~\cite{Velickovic:17GAT}, APPNP~\cite{klicpera2018predict}, GCNII~\cite{chen2020simple} and GRAND~\cite{feng2020grand}, as well as five representative scalable GNNs---FastGCN~\cite{FastGCN}, GraphSAINT~\cite{zeng2020graphsaint},  SGC~\cite{wu2019simplifying}, GBP~\cite{chen2020scalable} and PPRGo~\cite{bojchevski2020scaling}.
For \model, we implement three variants with different settings for propagation matrix $\mathbf{\Pi}$ (Cf. Equation~\ref{equ:prop}):
\begin{itemize}
\item \model\ (P):  \textit{Truncated ppr matrix} $\mathbf{\Pi}^{\text{ppr}} = \sum_{n=0}^{N}\alpha (1-\alpha)^n \mathbf{P}^{n}$.
\item \model\ (A): \textit{Average pooling matrix} $\mathbf{\Pi}^{\text{avg}} = \sum_{n=0}^{N} \mathbf{P}^{n}/(N+1)$.
\item \model\ (S): \textit{Single order matrix} $ \mathbf{\Pi}^{\text{single}} = \mathbf{P}^{N}$.
\end{itemize}
\vpara{Datasets.} The experimnents are conducted on seven public datasets of different scales, including three widely adopted benchmark graphs---Cora, Citeseer and Pubmed~\cite{yang2016revisiting}, and four relatively large graphs---AMiner-CS~\cite{feng2020grand}, Reddit~\cite{hamilton2017inductive}, Amazon2M~\cite{Chiang2019ClusterGCN} and MAG-scholar-C~\cite{bojchevski2020scaling}. For Cora, Citeseer and Pubmed, we use public data splits~\cite{yang2016revisiting, kipf2016semi, Velickovic:17GAT}. For AMiner-CS, Reddit, Amazon2M and MAG-Scholar-C, we use 20$\times$\#classes nodes for training, 30$\times$\#classes nodes for validation and the remaining nodes for test.
The corresponding statistics are summarized in Table~\ref{tab:dataset}.  More details for the setup and reproducibility can be found in Appendix ~\ref{sec:imp}. 
%For AMiner-CS, Reddit and  MAG-Scholar-C we randomly sample the same number of nodes for each class---20 nodes per class for training and 30 nodes per class for validation. For Amazon2M, we uniformly sample all the training and validation nodes from the whole datasets, as the node counts of some classses are less than 20. 

\begin{table}[t]
	\caption{Dataset statistics.}
	\vspace{-0.05in}
	\small
	\label{tab:dataset}
\begin{tabular}{l|rrrrrr}
		\toprule
		Dataset &  Nodes &  Edges  & Classes & Features \\
		\midrule
		%    Task  & transductive & transductive & transductive & inductive\\
		Cora & 2,708 & 5,429 & 7&1,433 \\ 
		Citeseer & 3,327 & 4,732 & 6 & 3,703 \\
		Pubmed & 19,717 & 44,338& 3 & 500 \\
		AMiner-CS & 593,486 & 6,217,004 & 18 & 100 \\
		Reddit & 232,965 & 11,606,919 & 41& 602 \\
		Amazon2M & 2,449,029 & 61,859,140 & 47 & 100 \\
		MAG-Scholar-C & 10,541,560 & 265,219,994& 8 & 2,784,240 \\
		\bottomrule
	\end{tabular}
	\vspace{-0.1in}
\end{table}

%\vpara{Evaluation Protocol.} For Cora, Citeseer and Pubmed, we use public data splits~\cite{yang2016revisiting, kipf2016semi, Velickovic:17GAT}.
%---20 nodes per class for training, 500 nodes for validation and 1,000 nodes for testing---as adopted by most of literature of GNNs
%For AMiner-CS, Reddit, Amazon2M and MAG-Scholar-C, we use 20$\times$\#classes nodes for training, 30$\times$\#classes nodes for validation and the remaining nodes for test. %For fair comparison, we make careful hyperparameter selections for all the methods on the four datasets (Cf. Appendix~\ref{sec:hyper_selection})\footnote{We will release all codes after the period of double-blind review.}.

%Following previous works~\cite{shchur2018pitfalls,bojchevski2020scaling}, we report the average results of 100 random runs for each method.

%For reproducibility, we provide detailed implementation notes in Appendix \ref{sec:imp}\footnote{We will release all codes upon the end of double-blind review.}.

\subsection{Results on Benchmark Datasets}
\label{sec:overall}
%\vpara{Results on Benchmark Graphs.}
To evaluate the effectiveness of  \model, we compare it with 10 GNN baselines on Cora, Citeseer and Pubmed.
%, the three most popular benchmarks for graph-based semi-supervised learning.
%Table~\ref{tab:small_graph} summarizes the prediction accuracy on 
% Cora, Citeseer and Pubmed, the most 
Following the community convention, the results of baseline models on the three benchmarks are taken from the previous works~\cite{Velickovic:17GAT, chen2020simple,feng2020grand}. 
For \model, we conduct $100$ trials with random seeds and report the average accuracy and the corresponding standard deviation over the trials. The results are demonstrated in Table~\ref{tab:small_graph}.
%As the three datasets are small, all the unlabeled nodes are used for consistency regularization in \model, i.e., $U'=U$. 
It can be observed that the best \model\ variant consistently outperforms all baselines across the three datasets. Notably, \model\ (A) improves upon \grand\ by a margin of 2.3\% (absolute difference) on Pubmed. The improvements of \model\ (P) over \grand\ on Cora  (85.8$\pm$0.4 vs. 85.4$\pm$0.4) and Citeseer (75.6$\pm$0.4 vs. 75.4$\pm$0.4) are also statistically significant (p-value $\ll$ 0.01 by a t-test). These results suggest the strong \textit{generalization performance} achieved by \model.
%When compared to the previous state-of-the-art full-batch GNN---GCNII, the proposed model also achieves 0.2\%, 2.2\% and 4.7\% improvements respectively. More importantly, \model\ is the only scalable GNN method that consistently outperforms GCNII across all datasets.
%which is the state-of-the-art outperformance  Compared with scalable GNN baselines, \model\ achieves 
%The reported accuracies  of \model\ in
%On Cora, Citeseer and Pubmed, we compare \model\ with 15 representative GNNs of different categories, that is:

%\begin{}

\begin{table}
	\small
	\caption{Classification Accuracy (\%) on Benchmarks.}
	\vspace{-0.1in}
      \label{tab:small_graph}
\begin{tabular}{c|l|cccc}
		\toprule%\toprule
		Category &Method &Cora & Citeseer & Pubmed \\
		\midrule
	   	\multirow{5}{*}{{\tabincell{c}{ Full-batch \\ GNNs}} } &GCN   & 81.5 $\pm$ 0.6 & 71.3 $\pm$ 0.4 & 79.1 $\pm$ 0.4  \\

		& GAT & 83.0 $\pm$ 0.7 & 72.5 $\pm$ 0.7 & 79.0 $\pm$ 0.3 \\
		& APPNP & 84.1 $\pm$ 0.3 & 71.6 $\pm$ 0.5 & 79.7 $\pm$ 0.3 \\ 
		%Graph U-Net~\cite{gao2019graph} & 84.4$\pm$0.6 & 73.2$\pm$0.5 & 79.6$\pm$0.2 \\
		%SGC~\cite{wu2019simplifying} & 81.0 $\pm$0.0 & 71.9 $\pm$ 0.1 &78.9 $\pm$ 0.0 \\
		%GMNN~\cite{qu2019gmnn} & 83.7  & 72.9 & 81.8  \\
		%GraphNAS~\cite{gao2019graphnas} & 84.2$\pm$1.0 & 73.1$\pm$0.9 & 79.6$\pm$0.4 \\
		& GCNII &85.5 $\pm$ 0.5 & 73.4 $\pm$ 0.6 & 80.3 $\pm$ 0.4\\
  	    &\grand & 85.4 $\pm$ 0.4 & 75.4 $\pm$ 0.4 & 82.7 $\pm$ 0.6 \\		
		 \midrule
		 
	 	\multirow{5}{*}{{\tabincell{c}{ Scalable \\ GNNs}} } %&GraphSAGE~\cite{hamilton2017inductive}& 78.9 $\pm$ 0.8 & 67.4 $\pm$ 0.7 & 77.8 $\pm$ 0.6 \\
		&FastGCN & 81.4 $\pm$ 0.5 & 68.8 $\pm$ 0.9 & 77.6 $\pm$ 0.5  \\
		& GraphSAINT & 81.3 $\pm$ 0.4 & 70.5 $\pm$ 0.4 & 78.2 $\pm$ 0.8 \\
		& SGC  &81.0 $\pm$ 0.1 & 71.8 $\pm$ 0.1 & 79.0 $\pm$ 0.1\\
		& GBP & 83.9 $\pm$ 0.7  & 72.9 $\pm$ 0.5 & 80.6 $\pm$ 0.4\\
		& PPRGo &82.4 $\pm$ 0.2 & 71.3 $\pm$ 0.3 &80.0 $\pm$ 0.4 \\
		  \midrule
  	\multirow{2}{*}{{\tabincell{c}{ Our \\ Methods}} }  & \model\ (P) & \textbf{85.8 $\pm$ 0.4} & \textbf{75.6 $\pm$ 0.4} & 84.5 $\pm$ 1.1\\
    & \model\ (A) & 85.5 $\pm$ 0.4 & 75.5 $\pm$ 0.4 & \textbf{85.0 $\pm$ 0.6}\\
    & \model\ (S) & 85.0 $\pm$ 0.5 & 74.4 $\pm$ 0.5 & 84.2 $\pm$ 0.6 \\
    
    \bottomrule  %\bottomrule

	\end{tabular}
\end{table}

\hide{
\begin{itemize}
    \item \textbf{GCN}~\cite{kipf2016semi}  uses the propagation rule described in Eq.~\ref{equ:gcn_layer}. %based on a first-order approximation of spectral convolutions on graphs (Cf. Eq.~\ref{equ:gcn_layer}).
    \item \textbf{GAT}~\cite{Velickovic:17GAT} propagates information based on  self-attention.
    \item \textbf{APPNP}~\cite{klicpera2018predict} propagates information with personalized PageRank matrix.% for GNNs.
    \item \textbf{VBAT}~\cite{deng2019batch} applies virtual adversarial training~\cite{miyato2015distributional} into GCNs.% to encourage the model invariant to adversarial attacks.  
%$\text{G}^3$NN~\cite{ma2019flexible} 
    \item \textbf{G$^3$NN}~\cite{ma2019flexible} %proposes a generative framework for semi-supervised learning on graphs.%, which 
    regularizes GNNs with an extra link prediction task.% to perform regularization. 
    %utilizes an additional link prediction task to perform regularization. 
    \item \textbf{GraphMix}~\cite{verma2019graphmix} adopts MixUp~\cite{zhang2017mixup} for regularizing GNNs.
    \item \textbf{DropEdge}~\cite{YuDropedge} randomly drops some edges in GNNs training.
    \item \textbf{GraphSAGE}~\cite{hamilton2017inductive} proposes node-wise neighborhoods sampling.

    \item \textbf{FastGCN}~\cite{FastGCN} using importance sampling for fast GCNs training.
    \item \textbf{\model\_GCN}. Note that the prediction module in \model\ is the simple MLP model. In \model\_GCN, we replace MLP with GCN. 
    \item \textbf{\model\_GAT} replaces the MLP component in \model\ with GAT.
    \item \textbf{\model\_dropout} substitutes our DropNode technique with the dropout operation in \model's random propagation.
\end{itemize}
}

\begin{table}
 	\small
	\caption{Accuracy (\%) and Running Time (s) on Large Graphs.}
	\vspace{-0.05in}
	\label{tab:large_graph2}
\resizebox{1.02\linewidth}{!}{
	 \setlength{\tabcolsep}{0.82mm}{ 
\begin{tabular}{l|cc|cc|cc|cc}
		\toprule
        \multirow{2}{*}{Method}  & \multicolumn{2}{c|}{AMiner-CS} & \multicolumn{2}{c|}{Reddit} & \multicolumn{2}{c|}{Amazon2M} & \multicolumn{2}{c}{MAG.} \\
		& Acc & RT & Acc  & RT & Acc  & RT  & Acc & RT \\
		\midrule

		 %GraphSAGE~\cite{hamilton2017inductive}& 46.1 $\pm$ 2.1 & 88.7 $\pm$ 0.5  &  72.3 $\pm$ 1.0  \\
		 GRAND & 53.1$\pm$1.1 & 750 & OOM& --& OOM& --& OOM& -- \\
		 \midrule
		FastGCN      & 48.9$\pm$1.6       & 69     & 89.6$\pm$0.6        & 158   & 72.9$\pm$1.0 & 239    & 64.3$\pm$5.6 & 4220 \\
		GraphSAINT   & 51.8$\pm$1.3        & 39     & 92.1$\pm$0.5        & 39    & 75.9$\pm$1.3 & 189    & 75.0$\pm$1.7  & 6009 \\
		SGC          & 50.2$\pm$1.2        & 9    & 92.5$\pm$0.2        & 31    & 74.9$\pm$0.5 & 69   & -- & >24h \\
		% sgc 54.3 1.0 16.9
		%& SIGN~\cite{FastGCN}                   & 50.7 $\pm$ 1.6  & 35.0 & 92.8 $\pm$ 0.2 & 80.2 & 76.9 $\pm$ 0.4& 169.0  & \\
		GBP          & 52.7$\pm$1.7        & 21     & 88.7$\pm$1.1        & 370   & 70.1$\pm$0.9 & 280  & -- & >24h \\
		PPRGo        & 51.2$\pm$1.4        & 11   & 91.3$\pm$0.2        & 233   & 67.6$\pm$0.5 & 160   & 72.9$\pm$1.1 & 434 \\
		\midrule

		\model (P) & 53.9$\pm$1.8          & 17   & 93.3$\pm$0.2         & 183 & 75.6$\pm$0.7 & 188 & 77.6$\pm$1.2 &  653 \\
		\model (A) & 54.2$\pm$1.7          & 14  &\textbf{93.5$\pm$0.2} & 174 & 75.9$\pm$0.7 & 136 & \textbf{80.0$\pm$1.1} &  737  \\
		\model (S) & \textbf{54.2$\pm$1.6} & 10  & 92.8$\pm$0.2         & 62  & \textbf{76.2$\pm$0.6} & 80 & 77.8$\pm$0.9 & 483 \\
		\bottomrule
	\end{tabular}
	}
	}
	\vspace{-0.05in}
\end{table}

\hide{
 \begin{table*}
 	\small
	\caption{Results on Large Datasets.}
	\vspace{-0.05in}
	\label{tab:large_graph2}
\begin{tabular}{l|cl|cc|cc|cc|cc}
		\toprule
        \multirow{2}{*}{Method}  & \multicolumn{2}{c|}{AMiner-CS} & \multicolumn{2}{c|}{Reddit} & \multicolumn{2}{c}{Amazon2M} & \multicolumn{2}{c|}{MAG} & \multicolumn{2}{c}{Papers100M}\\
		& Acc (\%) & RT (s) & Acc (\%) & RT (s) & Acc (\%) & RT (s) & Acc (\%) & RT (s) & Acc (\%) & RT (s) \\
		\midrule

		 %GraphSAGE~\cite{hamilton2017inductive}& 46.1 $\pm$ 2.1 & 88.7 $\pm$ 0.5  &  72.3 $\pm$ 1.0  \\
		FastGCN      & 48.9 $\pm$ 1.6        & 9$\times$ (68.7)     & 89.6 $\pm$ 0.6        & 157.6   & 72.9 $\pm$ 1.0 & 239.4    & 64.3 $\pm$ 5.6 & 4219.9  &  &\\
		GraphSAINT   & 51.8 $\pm$ 1.3        & 4$\times$ (39.4)     & 92.1 $\pm$ 0.5        & 39.4    & 75.9 $\pm$ 1.3 & 188.5    & 75.0 $\pm$ 1.7  & 6009.1 & & \\
		SGC          & 50.2 $\pm$ 1.2        & 0.9$\times$ (8.9)    & 92.5 $\pm$ 0.2        & 30.9    & 74.9 $\pm$ 0.5 & 68.5   & -- & -- & 52.6 $\pm$ 0.3 & 2246.1 \\
		% sgc 54.3 1.0 16.9
		%& SIGN~\cite{FastGCN}                   & 50.7 $\pm$ 1.6  & 35.0 & 92.8 $\pm$ 0.2 & 80.2 & 76.9 $\pm$ 0.4& 169.0  & \\
		GBP          & 52.7 $\pm$ 1.7        & 2$\times$ (21.2)     & 88.7 $\pm$ 1.1        & 370.0   & 70.1 $\pm$ 0.9 & 279.8   & -- & -- & &\\
		PPRGo        & 51.2 $\pm$ 1.4        & 1.2$\times$ (11.2)   & 91.3 $\pm$ 0.2        & 233.1   & 67.6 $\pm$ 0.5 & 160.2   & 72.9 $\pm$ 1.1 & 434.3 & 51.4 $\pm$ 0.2 & 631.2\\
		\midrule

		\model (P) & 53.9 $\pm$ 1.8          & 1.7$\times$ (16.7)   & 93.2 $\pm$ 0.1         & 153.4 & 74.7 $\pm$ 0.7 & 207.5 & 77.6 $\pm$ 1.2 &  596.2 & 52.7 $\pm$ 0.3 & 788.1\\
		\model (A) & 54.2 $\pm$ 1.7          & 1.5$\times$ (14.4)   &\textbf{93.4 $\pm$ 0.2} & 144.6 & 75.3 $\pm$ 0.7 & 193.0 & 77.8 $\pm$ 1.1 &  653.4 & 53.1 $\pm$ 0.3 & 760.6 \\
		\model (S) & \textbf{54.2 $\pm$ 1.6} & 1$\times$ (9.6)      & 92.9 $\pm$ 0.2         & 48.4  & \textbf{76.1 $\pm$ 0.6} & 68.7 & 77.8 $\pm$ 0.9 & 482.9 & 52.7 $\pm$ 0.5 & 754.8\\
		\bottomrule
	\end{tabular}
	\vspace{-0.05in}
\end{table*}
}
\hide{
 \begin{table}
 	\scriptsize
	\caption{Results on AMiner-CS, Reddit and Amazon2M$^4$.}
	\vspace{-0.1in}
	\label{tab:large_graph2}
\begin{tabular}{c|cc|cc|cc}
		\toprule
		\multirow{2}{*}{Method}  & \multicolumn{2}{c|}{AMiner-CS} & \multicolumn{2}{c|}{Reddit} & \multicolumn{2}{c}{Amazon2M} \\
		& Acc (\%) & RT (s) & Acc (\%) & RT (s) & Acc (\%) & RT (s) \\
		\midrule

		 GraphSAGE~\cite{hamilton2017inductive}& 46.1 $\pm$ 2.1  &     129.0 & 88.7 $\pm$ 0.5 & 439.7  &  72.3 $\pm$ 1.0 & 2078.2 \\
		FastGCN~\cite{FastGCN}                & 48.9 $\pm$ 1.6  &     68.7 & 89.6 $\pm$ 0.6 & 157.6 & 72.9 $\pm$ 1.0 & 239.4\\
		GraphSAINT~\cite{zeng2019graphsaint}                    & 51.8 $\pm$ 1.3  &39.4      & 92.1 $\pm$ 0.5 & 49.3 & 75.9 $\pm$ 1.3 & 188.5 \\
		SGC~\cite{wu2019simplifying}          & 50.2 $\pm$ 1.2  & 9.6 & 92.5 $\pm$ 0.2 & 30.9 & 74.9 $\pm$ 0.5 & 68.5\\
		% sgc 54.3 1.0 16.9
		%& SIGN~\cite{FastGCN}                   & 50.7 $\pm$ 1.6  & 35.0 & 92.8 $\pm$ 0.2 & 80.2 & 76.9 $\pm$ 0.4& 169.0  & \\
		GBP~\cite{chen2020scalable}           & 52.7 $\pm$ 1.7  & 21.2 & 88.7 $\pm$ 1.1 & 370.0 &  70.1 $\pm$ 0.9 & 279.8\\
		PPRGo~\cite{bojchevski2020scaling}    & 51.2 $\pm$ 1.4  & 11.2 & 91.3 $\pm$ 0.2 & 233.1 & 67.6 $\pm$ 0.5 & 160.2 \\
		\midrule

		\model\ (P) & 53.9 $\pm$ 1.8 & 16.7 & 93.2 $\pm$ 0.1 & 153.4 & 74.7 $\pm$ 0.7  & 207.5 \\
		\model\ (A) & 54.2 $\pm$ 1.7 & 14.4 & \textbf{93.4 $\pm$ 0.2}  & 144.6 &75.3 $\pm$ 0.7& 193.0\\
		\model\ (S) &  \textbf{54.2 $\pm$ 1.6} & 9.6 & 92.9 $\pm$ 0.2  &48.4 & \textbf{76.1 $\pm$ 0.6}  & 68.7 \\
		\bottomrule
	\end{tabular}
	\vspace{-0.1in}
\end{table}
}

\hide{
 \begin{table}
	\small
	\caption{Classification accuracy (\%) on and total running time (s) medium datasets$^4$.}
	\label{tab:large_graph2}
	\setlength{\tabcolsep}{1.mm}\begin{tabular}{c|ccccccc}
		\toprule
		\toprule
		Method &  GraphSAGE~\cite{hamilton2017inductive} &	FastGCN~\cite{FastGCN}  & GraphSAINT~\cite{zeng2019graphsaint}   & SGC~\cite{wu2019simplifying}  & 	GBP~\cite{chen2020scalable} &  PPRGo~\cite{bojchevski2020scaling}  & 	\model (P)\\

		\bottomrule
\bottomrule
\end{tabular}
\vspace{-0.1in}
\end{table}
}
\subsection{Results on Large Graphs}
To justify the scalability of \model, we further compare it with five scalable GNN baselines on four large graphs, i.e., AMiner-CS, Reddit, Amazon2M and MAG-Scholar-C. 
%Note that we do not compare full-batch GNNs and regularization based GNNs, as most of them cannot be executed on the datasets due to the large memory costs. 
%For each baseline, we conduct detailed hyper-parameter tuning on validation set. As for \model\, we fix the size of unlabeled subset $U'$ as 10000 and tune other hyperparameters on validation set. 
%For fair comparison, we conduct careful hyperparameter selection for all methods (Cf. Appendix~\ref{sec:imp}).
Note that the feature dimension of MAG-Scholar-C is huge (i.e., 2.8M features per node). To enable \model to deal with it, a learnable linear layer is added before random propagation to transform the high-dimensional node features to low-dimensional hidden vectors (Cf. Equation~\ref{equ:high}). 
For a fair comparison, we conduct careful hyperparameter selection for all methods (Cf. Appendix~\ref{sec:imp}).
We run each model for 10 trails with random splits, and report its average accuracy and average running time (including preprocessing time, training time and inference time) over the trials. 
%For each method,  and use the hyperparameter configuration with the best  report the average classification accuracy ( and total running time (RT)---including preprocessing time, training time and inference time---of 100 runs for each method. 
The results are summarized in Table~\ref{tab:large_graph2}.
 %The detailed hyper-parameters and running environments are described in Appendix A.1.
%report both classification accuracy and total running time (including preprocessing time, training time and inference time).
%We conduct 100 trails (10 random splits $\times$ 10 runs per split) for each method and report the average classification accuracy (Acc)  and total running time (RT)---including approximation time, training time and inference time---in Table~\ref{tab:large_graph2}.
%The results are summarized in Table~\ref{tab:large_graph2}.
%the parameter configuration with the best classification performance for comparison. The detailed hyperparameter search strategy for each method is described in Appendix~\ref{}. 
%Following~\cite{shchur2018pitfalls}, we conduct 10 random splits and run 10 trials per split (100 trails in total) for each method. 
%In Table~\ref{tab:large_graph}, we report model's average classification accuracy and running time---including preprocessing time, training time and inference time---on AMiner-CS, Reddit, Amazon2M and MAG-Scholar-C.  %(10 random splits and 10 random initializations for each split)%with the best parameter configuration.
%Table~\ref{tab:large_graph} summarizes model average result of 100 runs (10 random splits, 10 random initializations for each split) on each dataset.
%From Table~\ref{tab:}

We interpret the results of Table~\ref{tab:large_graph2} from two perspectives. First, combining with the results in Table~\ref{tab:small_graph}, we notice that the three variants of \model\ exhibit big differences %matrix is an important factor for \model's performance
across these datasets: On Cora and Citeseer,  \model\ (P) achieves better results than \model\ (A) and \model\ (S); On Pubmed,  Reddit and MAG-Scholar-C, \model\ (A) surpasses the other two variants; On AMiner-CS and Amazon2M, \model\ (S) gets the best classification results.
%\model\ (S) surpasses \model\ (P) and \model\ (A) and  on AMiner-CS an Amazon2M which exhibits big differences from the results on Cora, Citeseer, Pubmed and Reddit.
%achieves the best classification accuracy on the three datasets by adopting the optimal propagation matrix on different dataset.
%From Table~\ref{tab:large_graph2}, we notice that with the optimal propagation matrix \model\ achieves the best classification accuracy across all the three datasets. 
%Interestingly, we observe that \model\ (S) surpasses \model\ (P) and \model\ (A) w.r.t both accuracy and  on AMiner-CS an Amazon2M which exhibits big differences from the results on Cora, Citeseer, Pubmed and Reddit. We conjecture the reason is because the original features of AMiner-CS and Amazon2M contain more noise information, so that local information is not helpful to the model's prediction. 
This indicates that the propagation matrix plays a critical role in this task, and further suggests that \model\ could \textit{flexibly} deal with different graphs by adjusting the generalized mixed-order matrix $\mathbf{\Pi}$. 

Second, we observe \model\ consistently surpasses all baseline methods in accuracy and gets efficient running time on the four datasets. Importantly, on the largest graph MAG-Scholar-C, \model\ could succeed in training and making predictions in around 10 minutes, while SGC and GBP require more than 24 hours to finish, because the two methods are designed to directly propagate the high-dimensional raw features in pre-processing step. Compared with FastGCN and GraphSAINT, \model\ (S) achieves 8$\times$ and 12$\times$ acceleration respectively. 
 When compared with PPRGo, the fastest model on this dataset in the past, \model (S) gets 4.9\% improvement in accuracy while with a comparable running time. These results indicate \model\ \textit{scales well} on large graphs and further emphasize its \textit{excellent performance}.

\hide{
\begin{table}
	\small
	\caption{\model\ vs. full-batch GNNs on AMiner-CS.}
		\vspace{-0.1in}
    \label{tab:fullbatch}
  \begin{tabular}{c|ccc}
		\toprule%\toprule
		%\multirow{2}{*}{Method}	& \multicolumn{4}{c|}{Pubmed}  &  \multicolumn{4}{c}{AMiner-CS} \\
	  Method &  Accucay (\%) & Running Time (s) & GPU Memory (MB) \\
		\midrule
		GCN & 49.9$\pm$2.0 & 90 & 3021 \\
		GAT & 49.6$\pm$1.7 & 308  & 5621 \\
		GRAND&  53.1$\pm$1.1 & 750 & 3473 \\
		\midrule
		\model\ (S) & 54.2$\pm$1.6 & 10 & 1059 \\
		\bottomrule
		%\bottomrule
	\end{tabular}
\end{table}
}

We also report the accuracy and running time of GRAND on AMiner-CS. Note that it can not be executed on the other three large datasets due to the out-of-memory error. %As we observed, 
% comparisons between \model\ (S) and three full-batch GNNs (i.e., GCN, GAT and GRAND) on AMiner-CS.  %The experiments are only conducted on AMiner-CS as full-batch GNNs cannot be executed on other larger datasets due to the enormous memory costs. 
%ue to the limited space, 
%We report the  results achieved by \model's three variants.
%Table~\ref{tab:fullbatch} summarizes the results of classification accuracy, running time and occupied GPU memory of each method. 
 %The results of \model\ are produced by the variant with the best classification results on corresponding dataset.
As we can see, \model\ achieves over 40$\times$ acceleration in terms of running time over \grand on AMiner-CS, demonstrating the effectiveness of the proposed approximation techiniques in improving \textit{efficiency}.% \textit{scalability}.
%, and also significantly outperforms GCN and GAT in both efficiency and effectiveness. 
%As for the results on Pubmed,  \model\ has comparable memory cost and per batch running time with GCN and GAT, while takes slightly more running time and batches to converge. %When comparing results on AMiner-CS, we notice that \model\ significantly surpasses GCN and GAT in all indicators.

 %When we compare the results on AMiner-CS, 

%We observe the full-batch GNNs can only preserve good efficiency on Pubmed, when 

\hide{
 \begin{table}
 	\scriptsize
	\caption{Results on MAG-Scholar-C$^4$.}
	\label{tab:mag}
	\vspace{-0.1in}
\begin{tabular}{c|ccccc}
		%\toprule
		\toprule
	 Method  & Acc (\%) & RT (s) & Mem (MB) & BT (ms) & \#Batches \\
		\midrule
		GraphSAGE~\cite{hamilton2017inductive} & - & > 24h & - & - & - \\%3139 & 1113.4 & -\\
		FastGCN~\cite{FastGCN} & 64.3 $\pm$ 5.6 & 4219.9 & 7133 & 1098.2 & 2736\\
		GraphSAINT~\cite{zeng2019graphsaint} & 75.0 $\pm$ 1.7 & 6009.1 & 17897 & 99.4 & 790\\
		 SGC~\cite{wu2019simplifying}   & -  & > 24h & - & - & -\\
		%& SIGN~\cite{FastGCN}                   & 50.7 $\pm$ 1.6  & 35.0 & 92.8 $\pm$ 0.2 & 80.2 & 76.9 $\pm$ 0.4& 169.0  & \\
		GBP~\cite{chen2020scalable}          & - & > 24h & - & - \\
		PPRGo~\cite{bojchevski2020scaling}    &72.9 $\pm$ 1.1 & 434.3& 4037  & 403.6 & 703\\
		\midrule

%		\model (P) & \textbf{76.4 $\pm$ 1.0} & 524.6\\
%		\model (A) & 76.1 $\pm$ 1.0 & 490.3\\
		\model\ (S) \\$T=2, k=32$  &77.8 $\pm$ 0.9 & 489.4 & 2109 &688.9& 496\\
		$T=2, k=16$ & 76.5 $\pm$ 1.5 & 380.3 & 1851 &502.1&448 \\
		\bottomrule
	%	\bottomrule
	\end{tabular}
\vspace{-0.1in}
\end{table}
}

%Compared with PPRGo, GraphSAINT and FastGCN, we observe that \model\ achieves better classification performance, faster running time and lower memory cost when $k=16$. Specifically, \model is 11$\times$ and 15$\times$ faster than FastGCN and GraphSAINT, respectively. When compared with PPRGo, the fastest model on this dataset in the past, \model gets 3.6\% improvement in accuracy while with a comparable running time. 

%As $k$ increases to $32$, the classification accuracy of \model\ increases from 76.5\% to 77.8\% with the running time increasing from 380.3 to 489.4. This suggests that we could control the trade-off between efficiency and effectiveness by adjusting $k$. We will examine this phenomenon in detail in Section~\ref{sec:param}.

%and two settings of maximum neighborhood size: $k$
%order  the results when using single propagation matrix with 
%as the feature propagation is hard to GBP does not support gradients back-propagation. 

%feature propagation procedure in official implementation does not  the propagation results . 

\hide{
\subsection{Ablation Studies.}
We conduct ablation studies to justify the contributions of proposed techniques to \model. Specifically, we examine \model's performance on Cora, Citeseer, Pubmed and MAG-Scholar-C when removing the following designs:
\begin{itemize}
    \item w/o RP and CR: Do not use \textit{random propagation} (RP) and \textit{consistency regularization} (CR) in \model, i.e. $M=1$, $\sigma=0$ and $\lambda=0$, which means \model\ does not perform data augmentation and only uses supervised loss for training like most GNN baselines.
%    \item w/o multiple augmentation: Only conduct once random propagation at each epoch, i.e., M = 1. In this way, consistency loss serves as a self-supervised objective which only enforces model to give highly confident predictions for unlabeled nodes.
%    \item w/o sharpening: Do not adopt sharpening trick (Cf. Equation~\ref{equ:sharp}) in consistency regularization, i.e., $\tau=1$.
    \item w/o confidence: Do not consider prediction confidence in consistency loss (Cf. Equation~\ref{equ:consis}), i.e., $\gamma=0$, meaning that the loss term  is identical to the consistency loss used in \grand.
    \item w/o weight scheduling: Do not use dynamic weight scheduling (Cf. Equation~\ref{equ:sch}) in model training, i.e., $\lambda_{min} = \lambda_{max}$.
\end{itemize}
 Table~\ref{tab:abl} summarizes the results. We observe that all the three techniques contribute to the success of \model. Notably, we observe that the accuracy of \model\ decreases significantly when removing random propagation and consistency regularization, suggesting that the two techniques are critical to the generalization capability of model in semi-supervised setting. 
 We also notice that the newly proposed confidence mechanism and dynamic weight scheduling strategy can consistently improve accuracy, demonstrating their effectiveness in enhancing generalization performance. 

\begin{table}
	\scriptsize
	\caption{Ablation study results (\%).}
	\vspace{-0.1in}
	\label{tab:abl}
\begin{tabular}{lcccc}
			%	\toprule
		\toprule
		Model &Cora & Citeseer & Pubmed & MAG-Scholar-C\\
		\midrule
		%    Task  & transductive & transductive & transductive & inductive\\
		\model  & \textbf{85.7 $\pm$ 0.4} & \textbf{75.6 $\pm$ 0.4} & \textbf{85.0 $\pm$ 0.6}  & \textbf{77.8 $\pm$ 0.9}\\
		\midrule
		\quad w/o RP and CR & 83.7 $\pm$ 0.6  & 70.4 $\pm$ 0.9 & 78.9 $\pm$ 0.9  & 76.3 $\pm$ 0.8\\ %79.7 $\pm$0.3 \\
	%	\quad w/o multiple augmentation &85.1 $\pm$ 0.5 & 74.9 $\pm$ 0.5 & 83.4 $\pm$ 1.4  & 77.6 $\pm$ 1.0\\
	%	\quad w/o sharpening & 83.9 $\pm$ 0.6  & 72.4 $\pm$ 0.8 & 82.1 $\pm$ 0.8 & 77.4 $\pm$ 1.2\\%79.7$\pm$0.4   
		\quad w/o confidence& 85.4 $\pm$ 0.4  & 75.0 $\pm$ 0.5 & 84.2 $\pm$ 0.7 & 77.6 $\pm$ 0.9 \\%79.7$\pm$0.4   \\
		\quad w/o weight scheduling & 85.5 $\pm$ 0.5     & 75.2 $\pm$ 0.8  & 84.4 $\pm$ 0.6   & 71.9 $\pm$ 2.9\\ 
		%\quad replace dropnode with dropout & 84.9 $\pm$ 0.4  & 75.0 $\pm$ 0.3 & 81.7 $\pm$ 1.0 \\
		\bottomrule
		%	\bottomrule
	\end{tabular}
	\vspace{-0.1in}
\end{table}
\begin{figure}[t]
	\centering
	\mbox
	{
			 \hspace{-0.2in}
		\begin{subfigure}[Cora]{
				\centering
				\includegraphics[width = 0.27 \linewidth]{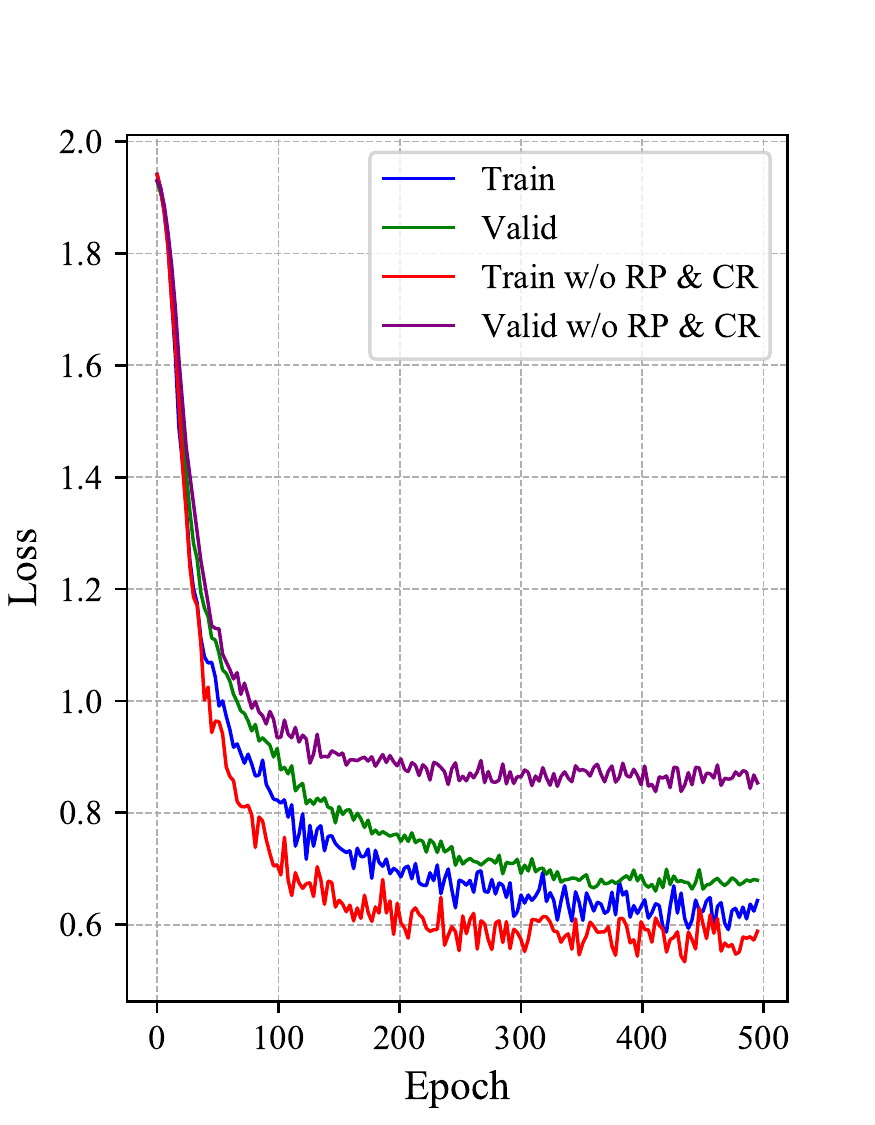}
			}
		\end{subfigure}
		%\hfill
			 %\hspace{-0.15in}
		\begin{subfigure}[Citeseer]{
				\centering
				\includegraphics[width = 0.27 \linewidth]{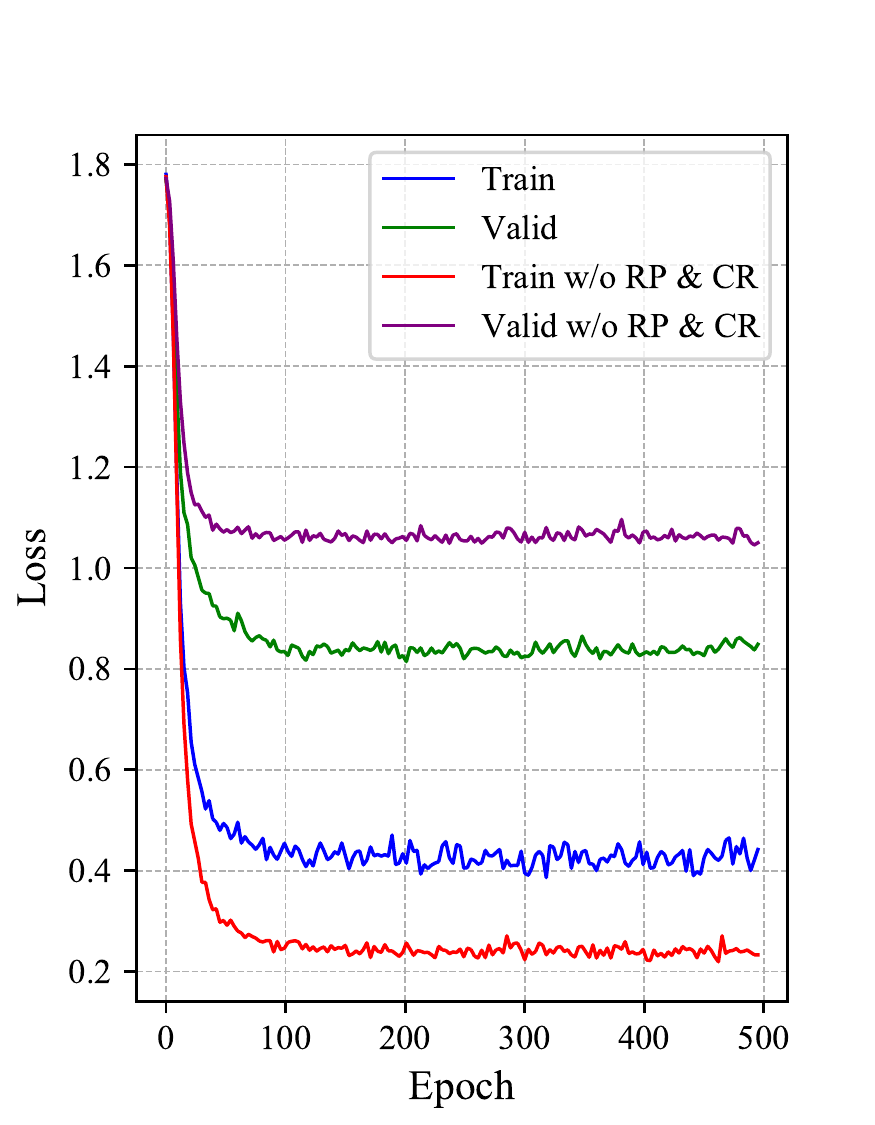}
			}
		\end{subfigure}
	 %\hspace{-0.15in}
		\begin{subfigure}[Pubmed]{
				\centering
				\includegraphics[width = 0.27 \linewidth]{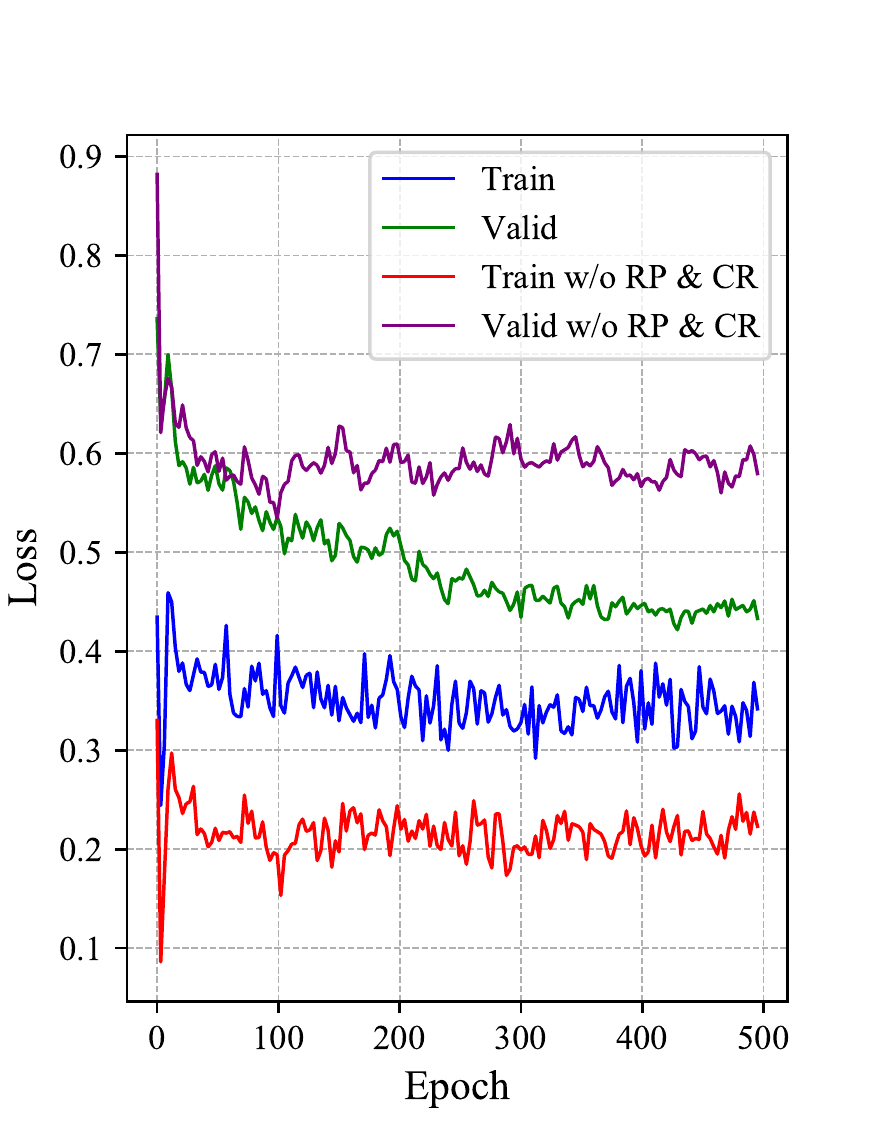}
			}
		\end{subfigure}
	}
	%\hspace{-0.4in}
		\vspace{-0.15in}
	\caption{Generalization: Training and validation losses. 
	%The $y$-axis of each figure has the same range.
	}
	\vspace{-0.2in}
	\label{fig:loss}
\end{figure}
\vpara{Generalization Improvements.} To explore why the random propagation and consistency regularization could improve model's generalization performance, we analyze model's training loss and validation loss when using and not using the two techniques. The results on Cora, Citeseer and Pubmed are shown in Figure~\ref{fig:loss}. As we can see, the adopted random propagation and consistency regularization techniques could significantly reduce the gap between training loss and validation loss across the three datasets, thus alleviate the overfitting problem and improve generalization capability.
}

\subsection{Generalization Improvements}
\begin{figure}[t]
	\centering
	\mbox
	{
    	\hspace{-0.1in}
		%\hfill
		\begin{subfigure}[Accuracy w.r.t. $\lambda_{max}$.]{
				\centering
				\includegraphics[width = 0.48 \linewidth]{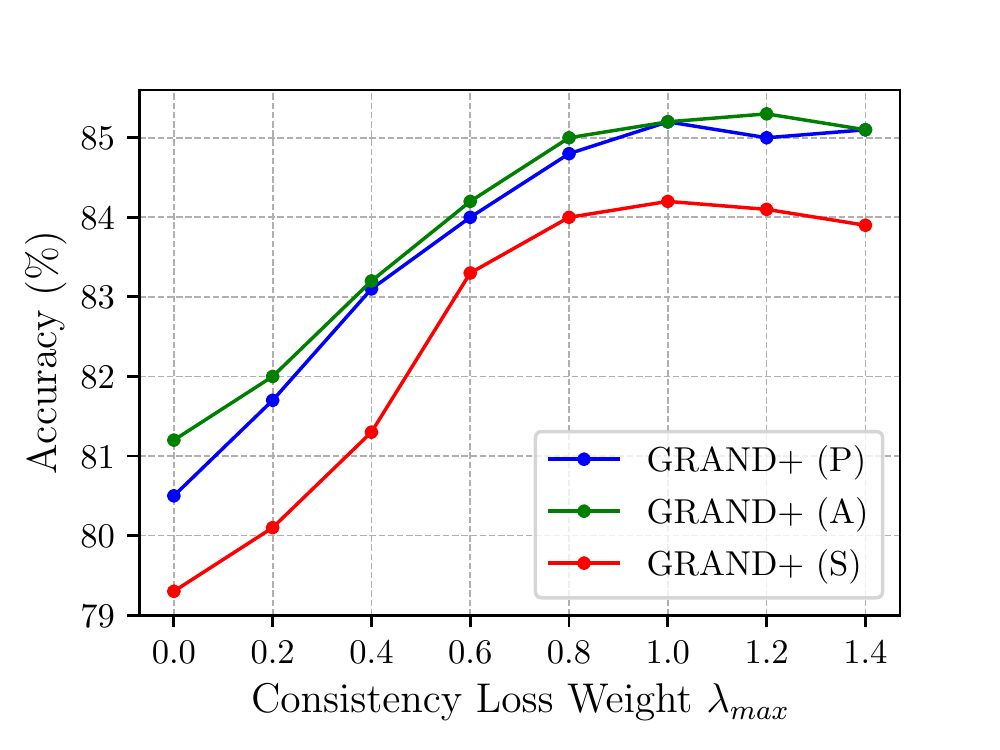}
			}
		\end{subfigure}
		\hspace{-0.1in}
		\begin{subfigure}[Accuracy w.r.t. $\gamma$]{
				\centering
				\includegraphics[width = 0.48 \linewidth]{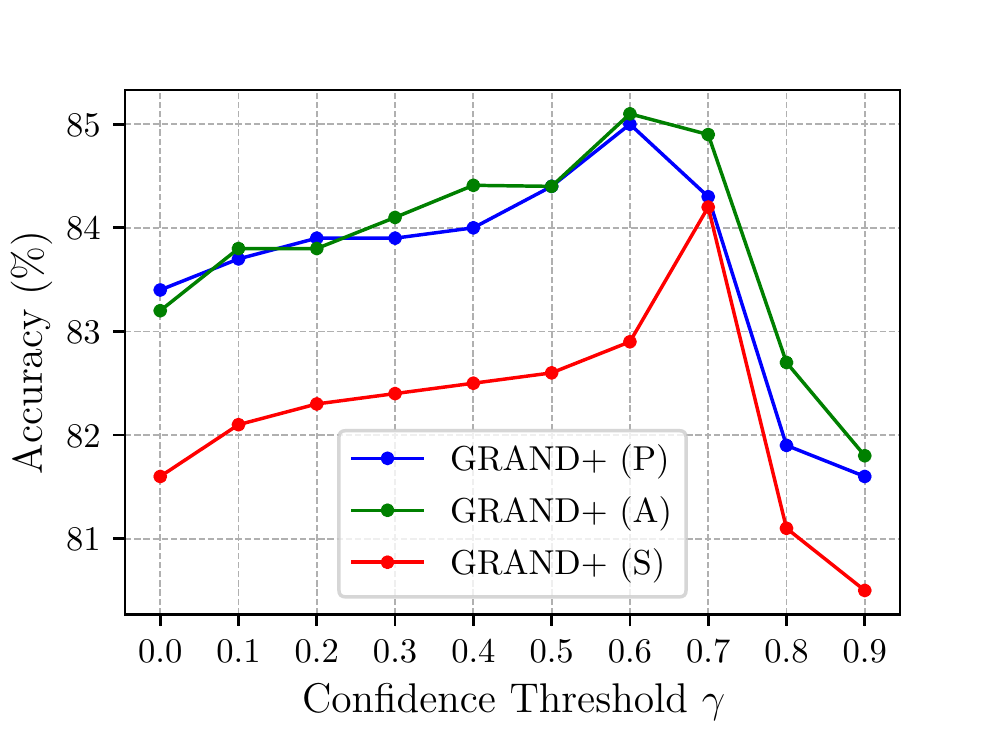}
			}
		\end{subfigure}
	}
	\vspace{-0.1in}
	\caption{Effects of $\lambda_{max}$ and $\gamma$ on Pubmed.}
	\label{fig:lam_conf}
	\vspace{-0.1in}
\end{figure}

In this section, we quantitatively investigate the benefits of the proposed confidence-aware consistency loss $\mathcal{L}_{con}$ to model's generalization capability. In \model, $\mathcal{L}_{con}$ is mainly dominated by two hyperparameters: \textit{confidence threshold} $\gamma$ (Cf. Equation~\ref{equ:consis}) and \textit{maximum consistency loss weight} $\lambda_{max}$ (Cf. Equation~\ref{equ:total_loss}). 

We first analyze the effects of $\gamma$ and $\lambda_{max}$ on \model's classification performance. Specifically, we adjust the values of $\gamma$ and $\lambda_{max}$ separately with other hyperparameters fixed, and observe how \model's accuracy changes on test set. Figure~\ref{fig:lam_conf} illustrates the results on Pubmed dataset. From Figure~\ref{fig:lam_conf} (a), it can be seen that the accuracy is significantly improved as $\lambda_{max}$ increases from 0 to 0.8. When $\lambda_{max}$ is greater than 0.8, the accuracy tends to be stable. This indicates that the consistency loss could really contribute to \model's performance. From Figure~\ref{fig:lam_conf} (b), we observe model's performance benefits from the enlargement of $\gamma$ when $\gamma$ is less than 0.7, which highlights the significance of the confidence mechanism. If $\gamma$ is set too large (i.e., $> 0.7$), the performance will degrade because too much unlabeled samples are ignored in this case, weakening the effects of consistency regularization.

\begin{figure}[t]
	\centering
	\mbox
	{
			 \hspace{-0.05in}
		\begin{subfigure}[$\lambda_{max}=0.0$]{
				\centering
				\includegraphics[width = 0.32 \linewidth]{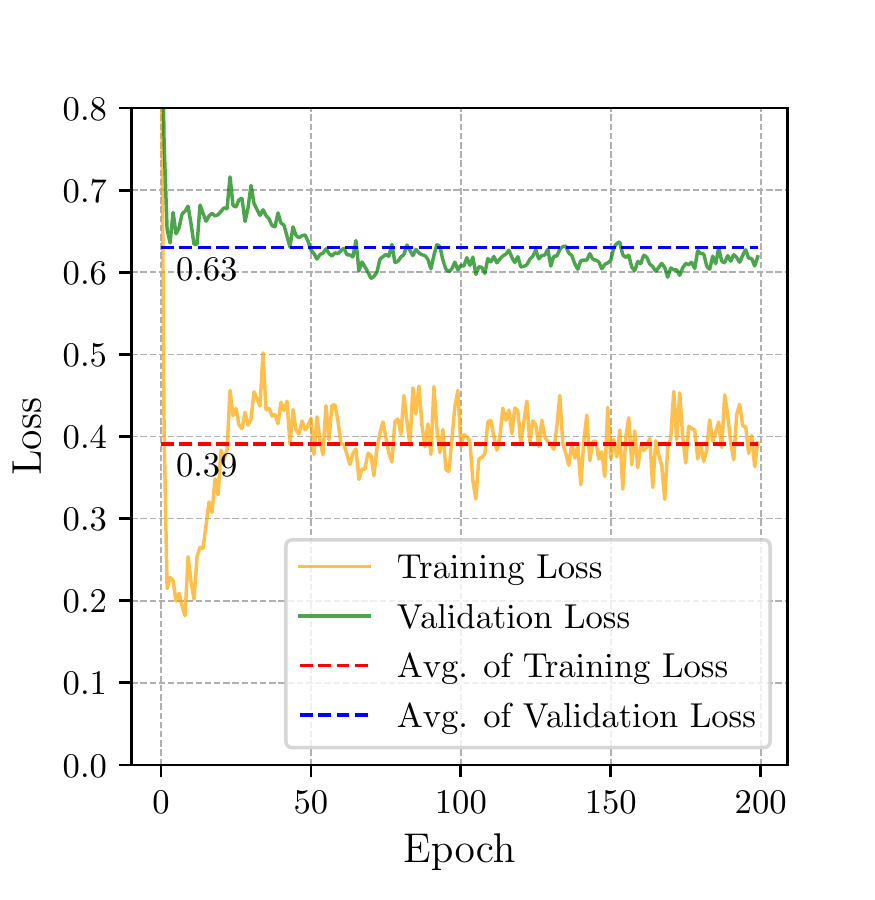}
			}
		\end{subfigure}
		%\hfill
			\hspace{-0.05in}
		\begin{subfigure}[$\lambda_{max}=1.0$, $\gamma=0.0$]{
				\centering
				\includegraphics[width = 0.32 \linewidth]{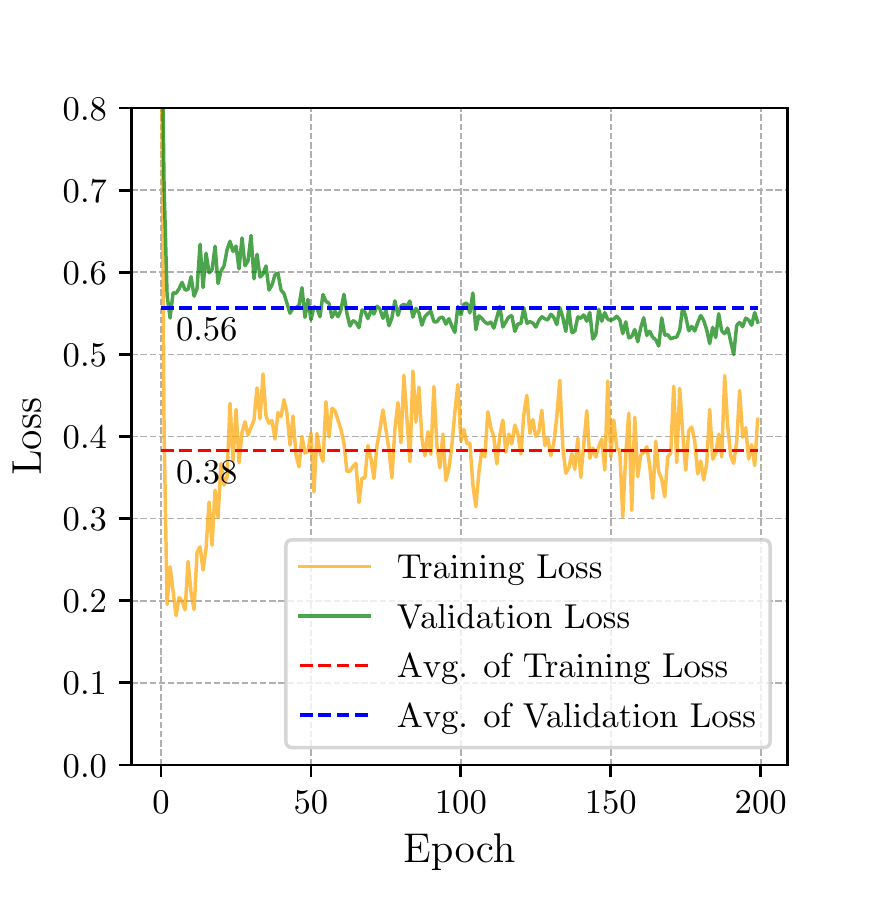}
			}
		\end{subfigure}
	 \hspace{-0.05in}
		\begin{subfigure}[$\lambda_{max}=1.0$, $\gamma=0.6$]{
				\centering
				\includegraphics[width = 0.32 \linewidth]{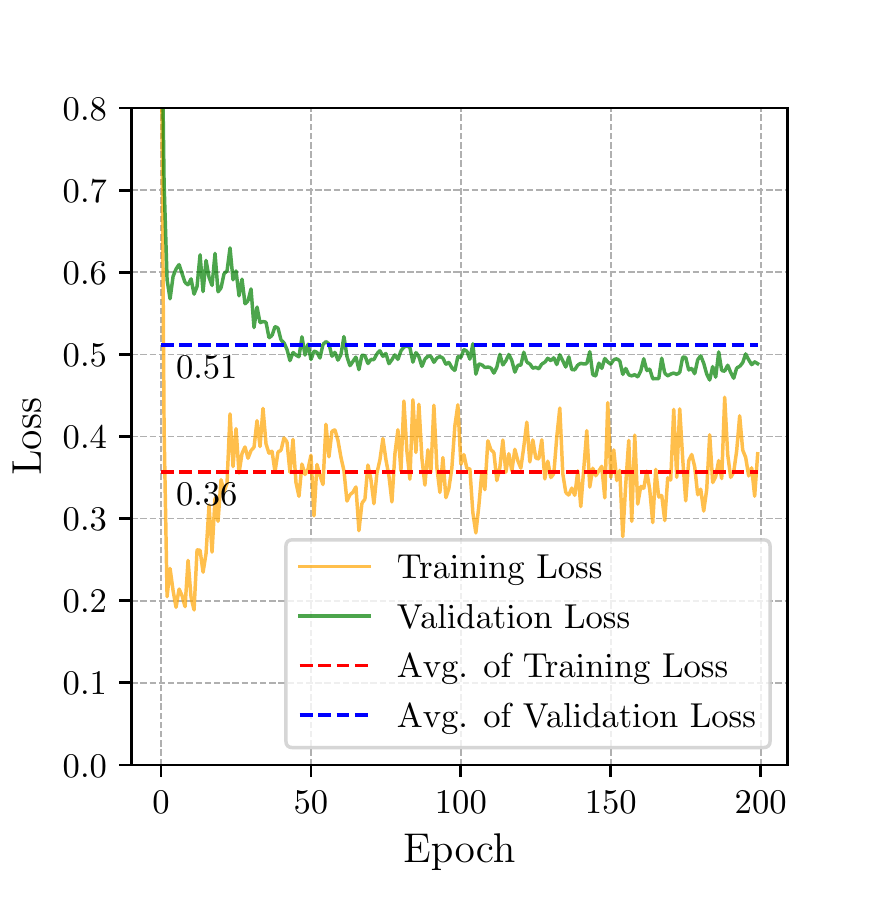}
			}
		\end{subfigure}
	}
	%\hspace{-0.4in}
		\vspace{-0.05in}
	\caption{Training and Validation Losses on Pubmed.
	%The $y$-axis of each figure has the same range.
	}
	\vspace{-0.1in}
	\label{fig:loss}
\end{figure}

Figure~\ref{fig:lam_conf} demonstrates the confidence-aware consistency loss could significantly improve model's performance. We further study its benefits to generalization capability by analysing the cross-entropy losses on training set and validation set. Here we measure model's generalization capability with its \textit{generalization gap}~\cite{keskar2016large}---\textit{the gap between training loss and validation loss}. A smaller generalization gap means the model has a better generalization capability.
Figure~\ref{fig:loss} reports the training and validation losses of \model\ (A) on Pubmed. We can observe the generalization gap is rather large when we do not use consistency loss ($\lambda_{max} = 0$) during training, indicating a severe over-fitting issue. And the gap becomes smaller when we change $\lambda_{max}$ to $1.0$. When we set both $\lambda_{max}$ and  $\gamma$ to proper values (i.e., $\lambda_{max}=1.0$, $\gamma=0.6$), the generalization gap further decreases. These observations demonstrate the proposed consistency training and confidence mechanism indeed contribute to \model's generalization capability. 
%$\lambda_{max}$ and $\gamma$ result in similar trends for performance: The accuracy first increases to a peak as the parameter increases from 0 and then decreases when it is too large. Thus we can conclude that the confidence-aware consistency loss could significantly improve model's performance when $\lambda_{max}$ and $\gamma$ are set to proper values. If $\lambda_{max}$ is set too large

\subsection{Parameter Analysis}
\label{sec:param}

\hide{
 \begin{table}
 	\scriptsize
	\caption{Effects of unlabeled subset size ($|U'|$).}
	\vspace{-0.1in}
	\label{tab:unlabelsize}
	\setlength{\tabcolsep}{1.mm}\begin{tabular}{c|ccc|ccc|ccc}
% 		\toprule
		\toprule
		\multirow{2}{*}{|U'|}  & \multicolumn{3}{c|}{AMiner-CS} & \multicolumn{3}{c|}{Reddit} & \multicolumn{3}{c}{Amazon2M} \\
		& Acc (\%) & RT (s) & PT (ms) & Acc (\%) & RT (s) & AT (ms) & Acc (\%) & RT (s) & AT (ms) \\
		\midrule
		 0 & 51.1 $\pm$ 1.4 & 9.8& 149.2 & 92.3 $\pm$ 0.2 & 53.0 & 717.3 & 75.0 $\pm$ 0.7 & 63.2 & 2356.2 \\
%		 100& 52.2 $\pm$ 1.7 & 8.1 & 151.4 & 92.4 $\pm$ 0.3 & 56.9 & 725.6  &  75.1 $\pm$ 0.7 & 57.5 & 2454.5 \\
         $10^3$ & 53.6 $\pm$ 1.6  & 9.0 & 153.3 & 92.6 $\pm$ 1.2 & 58.0 & 881.6 & 75.2 $\pm$ 0.5 & 61.6 &  2630.4\\       
         $10^4$ & 54.2 $\pm$ 1.6  & 9.6 & 250.4 & 92.9 $\pm$ 0.2 & 59.4 & 2406.6 &  76.1 $\pm$ 0.6 & 68.7 & 3648.6 \\
	    $10^5$ & 54.4 $\pm$ 1.2  & 12.6 & 1121.1 &  92.9 $\pm$ 0.2 & 75.7 & 17670.1 & 76.3 $\pm$ 0.7 & 85.5 & 14249.8 \\
% 		\bottomrule
		\bottomrule
	\end{tabular}
	\vspace{-0.1in}
\end{table}

\vpara{Size of Unlabeled Subset $U'$.} In order to perform consistency regularization, \model\ samples a unlabeled subset $U'$ from  $U$ and conducts approximations for all  row vectors corresponding to nodes in $U'$.
%order to perform consistency regularization, we need to sample a subset of unlabeled nodes $U'$ from $U$ and approximate the transition vectors for nodes in $U'$ before training. 
Here we analyze the effects of the size of $U'$ by setting $|U'|$ as different values from 0 to $10^5$.
The results are presented in able~\ref{tab:unlabelsize}, where we report the classification accuracy (Acc), total running time (RT) and approximation time (AT) of \model\ (S) on AMiner-CS, Reddit and Amazon2M under different settings of $|U'|$.
%The results are presented in Table~\ref{tab:unlabelsize}. 
It can be seen that when the number of unlabeled nodes increases from 0 to $10^3$, the classification results of the model are improved significantly with a comparable total running time, which indicates that the consistency regularization serves as an economic way for improving \model's generalization performance on unlabeled samples. However, when $|U'|$ exceeds $10^4$, the increase rate of classification accuracy will slow down, and the total running time and approximation time increase dramatically. This suggests that we can choose an appropriate value for $|U'|$ to balance the trade-off between effectiveness and efficiency of \model\ in practice.
}
%become larger and larger.

 %The running time and preprocessing time also 
\vpara{Threshold $r_{max}$ and Neighborhood Size $k$.} \model\ uses GFPush and top-$k$ sparsification to approximate multiple row vectors of $\mathbf{\Pi}$ to perform mini-batch random propagation (Cf. Section~\ref{sec:rand_prop}). The approximation error of this process is mainly influenced by two hyperparameters---threshold $r_{max}$ of GFPush and maximum neighborhood size $k$ for sparsification. We conduct detailed experiments to better understand the effects of $k$ and $r_{max}$ on model's accuracy and running time. Figure~\ref{fig:k_r_max} illustrates the corresponding results of \model\ (S) w.r.t. different values of $k$ and $r_{max}$ on MAG-Scholar-C.
%Figure~\ref{fig:param} (a) presents the classification accuracy of \model\ (S) w.r.t. different values of $k$ and $r_{max}$. Table~\ref{tab:k} shows the corresponding total running time and the average of the summation of approximated transition vector's top-$k$ elements,i.e., $\sum_{v\in {U' \cup L}} \sum\widetilde{\mathbf{\Pi}}^{(k)}_v/(|U'|+|L|)$, which reflects the average probability mass preserved in $\widetilde{\Pi}_v^{(k)}$. The results show that $r_{max}$ and $k$ influence much on both running time and performance.
As we can see, both the accuracy and running time increase when $r_{max}$ becomes smaller, which is coincident with the conclusion of Theorem~\ref{thm1}. While $k$ has an opposite effect---the accuracy and running time are enlarged with the increase of $k$.  
%more probability masses will be preserved in $\mathbf{\Pi}^{(k)}$ on average, so as to improve the classification accuracy, and total running time also increases. 
Interestingly, as $k$ decreases from $128$ to $32$, the running time is cut in half with only $\sim2\%$ performance drop in accuracy. This demonstrates the effectiveness of the top-$k$ sparsification strategy, which could achieve significant acceleration at little cost of accuracy.

%which indicates that we are able to improve 

%In the training period of \model, the sparsified transition vector $\widetilde{\Pi}_s^{(k)}$ is used for random propagation, which is mainly influenced by two hyperparameters---threshold $r_{max}$ and maximum neighborhood size $k$.

%In GFPush, the threshold $r_{max}$ is used to control approximation precision the  By adopting the top-$k$ sparsification strategy, \model\ is enabled to use the hyperparameter $k$ to control the maximum neighborhood size of each node used in training. 
%To better understand the effects of $k$ and $r_{max}$, we conduct detailed experiments for \model\ (S) on MAG-Scholar-C. Figure~\ref{fig:param} (a) presents the classification accuracy of \model\ (S) w.r.t. different values of $k$ and $r_{max}$. 

%the impact of $k$,

\vpara{Propagation Order $N$.} We study the influence of propagation order $N$ on \model\ when using different propagation matrices. Figure~\ref{fig:T} presents the classification performance and running time of three \model\ variants on MAG-Scholar-C w.r.t. different values of $N$. As we can see, when $N=2$, \model\ (S) achieves better accuracy and faster running time than \model\ (P) and \model\ (A). However, as $N$ increases, the accuracy of \model\ (S) drops dramatically because of the over-smoothing issue, while \model\ (P) and \model\ (A) do not suffer from this problem and benefit from a larger propagation order. On the other hand, 
increasing $N$ will enlarge models' running time. In real applications, we can flexibly adjust the propagation matrix and the value of $N$ to make desired efficiency and effectiveness.

\begin{figure}[t]
	\centering
	\mbox
	{
    	\hspace{-0.1in}
		%\hfill
		\begin{subfigure}[Classification accuracy.]{
				\centering
				\includegraphics[width = 0.49 \linewidth]{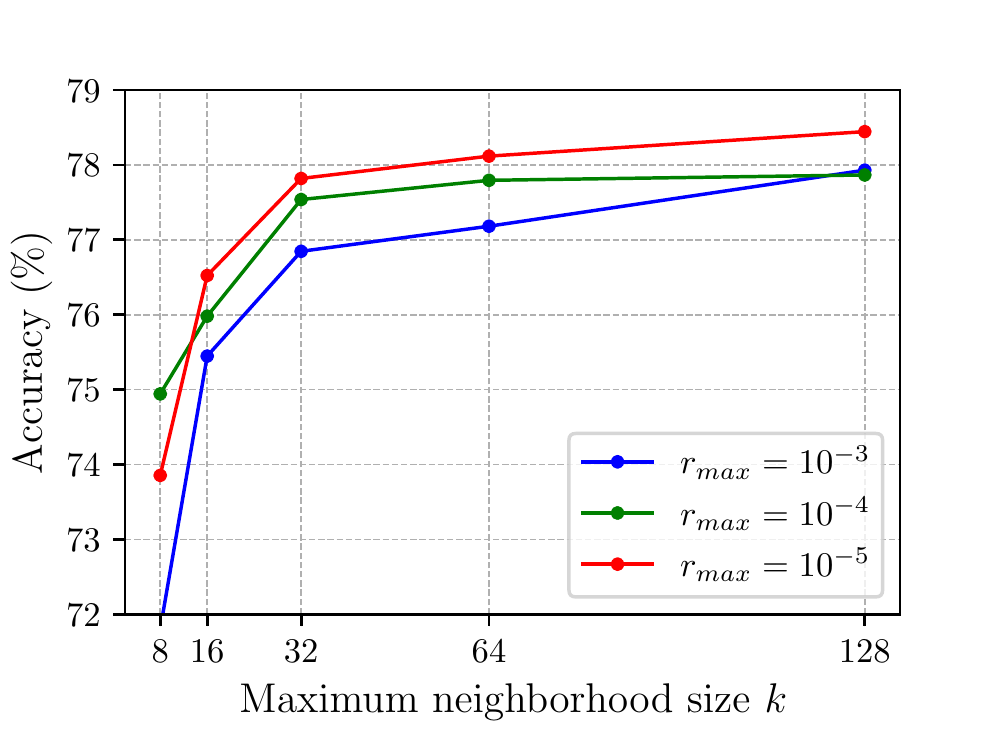}
			}
		\end{subfigure}
		\hspace{-0.1in}
		\begin{subfigure}[Running time.]{
				\centering
				\includegraphics[width = 0.49 \linewidth]{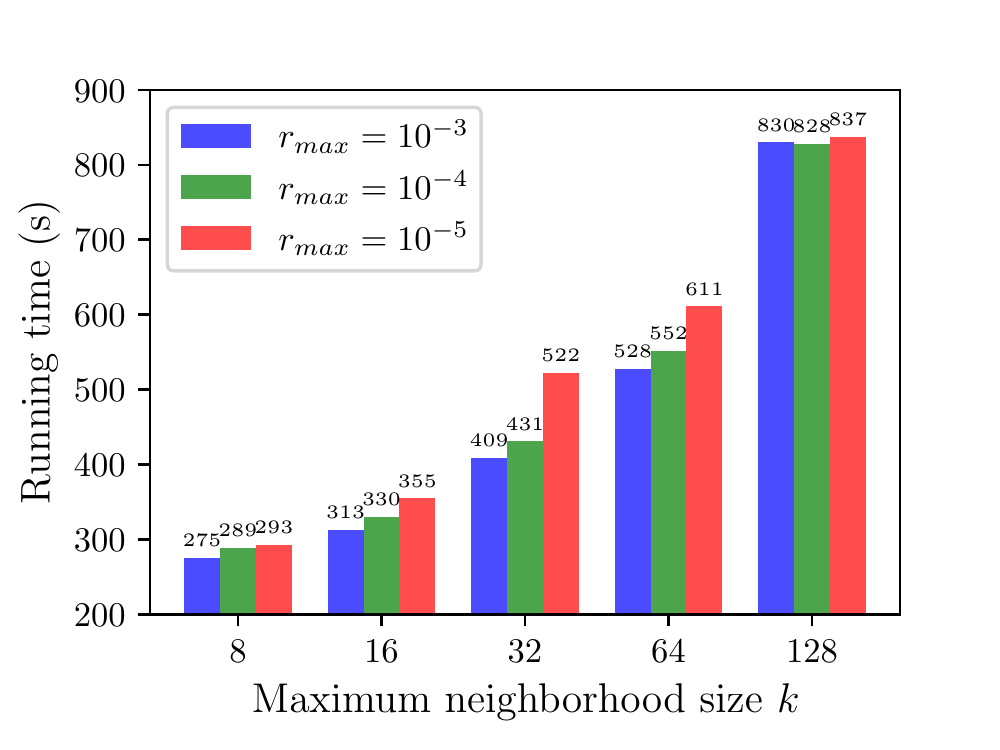}
			}
		\end{subfigure}
	}
	\vspace{-0.1in}
	\caption{\model\ w.r.t. $k$ and $r_{max}$ on MAG-Scholar-C.}
	\label{fig:k_r_max}
	\vspace{-0.1in}
\end{figure}

\begin{figure}[t]
	\centering
	\mbox
	{
    	\hspace{-0.1in}
		%\hfill
		\begin{subfigure}[Classification accuracy.]{
				\centering
				\includegraphics[width = 0.49 \linewidth]{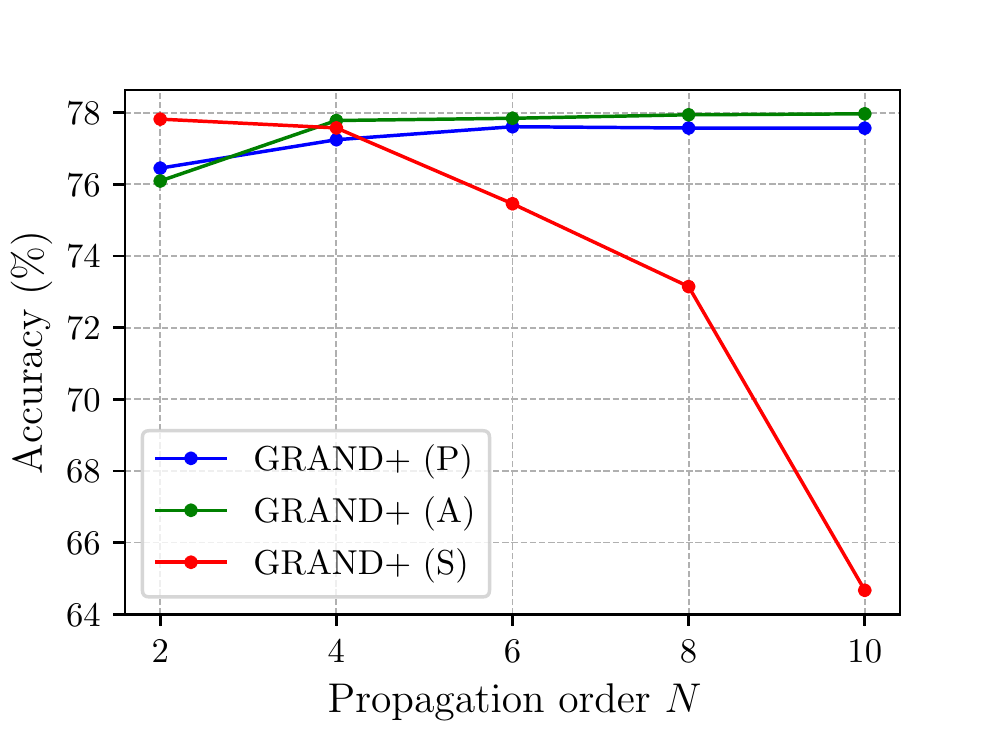}
			}
		\end{subfigure}
		\hspace{-0.1in}
		\begin{subfigure}[Running time.]{
				\centering
				\includegraphics[width = 0.49 \linewidth]{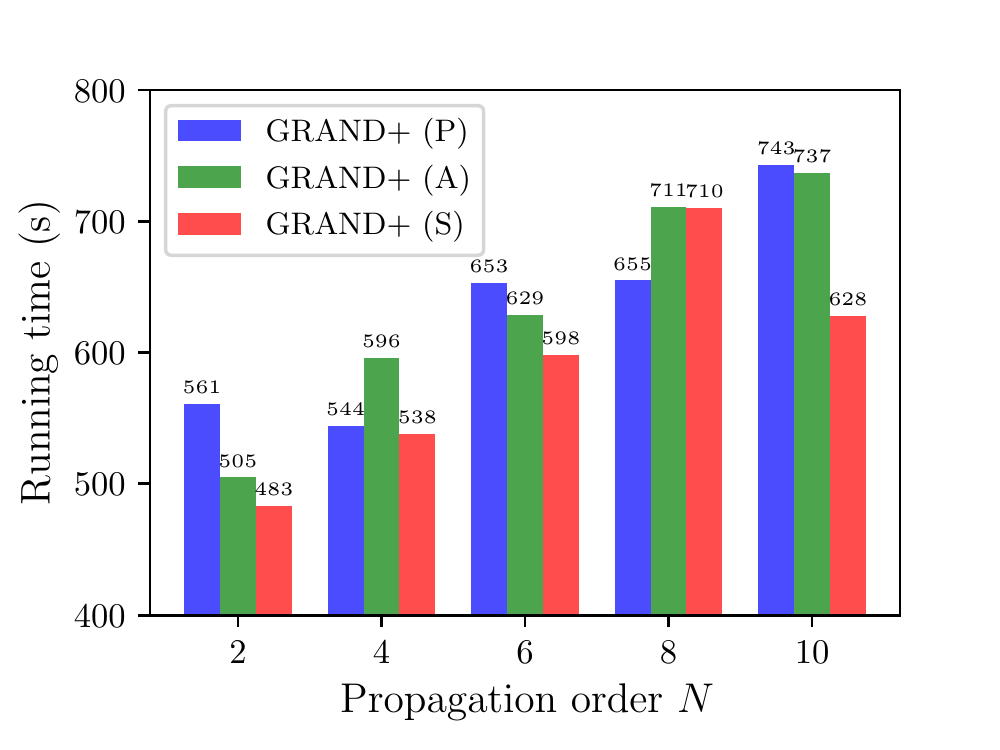}
			}
		\end{subfigure}
	}
	\vspace{-0.1in}
	\caption{Effects of propagation order $N$ on MAG-Scholar-C.}
	\label{fig:T}
	\vspace{-0.1in}
\end{figure}

\hide{

 \begin{table}
 	\scriptsize
	\caption{Efficiency w.r.t. $r_{max}$ and $k$.}
	\vspace{-0.1in}
	\label{tab:k}
	\setlength{\tabcolsep}{1.4mm}\begin{tabular}{c|ccccc|ccccc}
		\toprule
% 		\toprule
		& \multicolumn{5}{c|}{Total running time (s)} & \multicolumn{5}{c}{Sum of top-$k$ elements} \\
	$k$  & 8 & 16 & 32 & 64 & 128 & 8 & 16 & 32 & 64 & 128\\
		\midrule
    $r_{max} = 10^{-3}$ & 275.4 & 313.2 & 408.8 & 528.1 & 829.8 & 0.189 & 0.261 & 0.349 & 0.450 & 0.555 \\
    $r_{max} = 10^{-4}$ & 289.3 & 329.7 & 431.2 & 551.5 & 828.2 & 0.204 & 0.285 & 0.384 & 0.502 & 0.632 \\
    $r_{max} = 10^{-5}$ & 293.1 & 355.3 & 521.7 & 611.2 & 837.2 & 0.205 & 0.285 & 0.385 & 0.503 & 0.634 \\
		\bottomrule
% 		\bottomrule
	\end{tabular}
	\vspace{-0.1in}
\end{table}

 \begin{table}
 	\scriptsize
	\caption{Efficiency w.r.t. propagation order $T$.}
	\vspace{-0.1in}
	\label{tab:prop}
	\setlength{\tabcolsep}{1.mm}\begin{tabular}{c|ccccc|ccccc}
% 		\toprule
		\toprule
		& \multicolumn{5}{c|}{Total running time (s)} & \multicolumn{5}{c}{Preprocessing time (ms)} \\
	$T$  & 2 & 4 & 6 & 8 & 10 & 2 & 4 & 6 & 8 & 10\\
		\midrule
     \model\ (P) & 561.2 & 543.8 & 653.4 & 654.8 & 743.0 & 3733.4 & 3672.5 & 5940.4 & 5791.5 & 6125.2 \\
      \model\ (A) &505.1 & 596.2 & 629.0 & 711.3 & 737.0 & 3946.8 & 3811.0 & 5808.8 & 5921.9 & 6099.8 \\
 \model\ (S) & 482.9 & 538.3 & 597.9 & 710.2 & 628.2 & 3857.5 & 5328.8 & 6258.2 & 7127.3 & 7325.5 \\
		\bottomrule
% 		\bottomrule
	\end{tabular}
	\vspace{-0.1in}
\end{table}

}

\section{Conclusion}
\label{sec:conclusion}
We propose \model, a scalable and high-performance GNN framework for graph-based semi-supervised learning.
The advantages of \model include both the scalability and generalization capability while the existing state-of-the-art solutions typically feature only one of the two. 
To this effect, we follow the consistency regularization principle of \grand in achieving the generalization performance, while significantly extend it to achieve scalability and retain and even exceed the flexibility and generalization capability of \grand. 
To achieve these, \model\ utilizes a generalized mixed-order matrix for feature propagation, and uses our  approximation method generalized forward push (GFPush) to calculate it efficiently. 
In addition, \model adopts a new confidence-aware consistency loss to achieve better consistency training. Extensive experiments show that \model not only gets the best performance on benchmark datasets, but also achieves  performance and efficiency superiority over existing scalable GNNs on datasets with millions of nodes.
In the future, we would like to explore more accurate approximation methods to accelerate GNNs. %We are also interested in exploring the potential of \model\ for active learning on large graphs.
%, and conduct detailed ablation studies and parameter analysis to better understand the effects of proposed techniques.

%\clearpage

\hide{\section{Conclusions}
\label{sec:conclusion}
% In this work, we study graph-based semi-supervised learning with both scalability and generalization capability. 
We propose \model, a scalable and high-performance GNN framework for graph-based semi-supervised learning.
\model\ follows the idea of Graph Random Neural Network (\grand), while significantly improves it in terms of flexibility, scalability and generalization capability.
%while significantly improves its defects with several novel techniques. 
To achieve that, \model\ utilizes a generalized mixed-order matrix for feature propagation, and uses a novel approximation method called GFPush to calculate it efficiently. \model also adopts a new confidence-aware consistency loss to achieve better consistency training. Extensive experiments show that \model not only gets the best performance on benchmarks, but also achieves better performance and efficiency over existing scalable GNNs on realistic dataset with millions of nodes. %We also conduct detailed analyses to better understand the effects of proposed techniques on generalization improvement and .
%We demonstrate the advantages of \model\ on seven public graph datasets of different scales and categories. The results show that \model\ owns flexibility, scalability and better generalization capability over \grand\ and other generic GNNs. 

In the future, we would like to explore more accurate approximation methods to accelerate GNNs. We are also interested in exploring the potential of \model\ for active learning on graphs. 
}

\hide{
In this work, we propose a highly effective method for for semi-supervised learning on graphs. 
%propose a general framework for semi-supervised learning on graphs. 
Unlike previous models following a recursive deterministic aggregation process, we explore graph-structured data in a random way and the random propagation method can generate data augmentations efficiently and achieve a model by implicit ensemble with a much lower generalization loss. In each epoch, we utilize consistency regularization to minimize the differences of prediction distributions among multiple augmented unlabeled nodes. We further reveal the effects of random propagation and consistency regularization in theory. 
\model\ achieves the best results on several benchmark datasets of semi-supervised node classification.
Extensive experiments  also suggest that \model\ has better robustness and generalization,  as well as lower risk of over-smoothing. 
Our work may be able to draw more academic attention to the randomness in graphs, rethinking what makes GNNs work and perform better.

}

\hide{

Importantly, these refinements can be concisely concluded as the diagonal matrix and adaptive encoders, resulting in our \rank\ model.  
Second, we empirically reveal information redundancy  in graph-structured data and propose a graph dropout-based sampling mechanism for the ensemble of exponential subgraph models, resulting in our \nsgcn\ model. 

Excitingly, our node ranking-aware propagation can cover many network attention mechanisms, including node attention, (multiple-hop) edge attention, and path attention, and especially, some existing attention-based GCNs, e.g., GAT, can be reformulated as special cases. Some previous network sampling-based GCNs, e.g., GraphSAGE and FastGCN, may also benefit from our network ensemble and disagreement minimization ideas. 
The graph convolutional modules proposed in this work, including the diagonal matrix module, node-wise non-linear activation encoders, and the training and ensemble framework of sampling-based GCNs, can be reused in other graph neural architectures, and we would like to explore these applications of our model for future work.

Our primary contribution to graph convolutional networks originates from neighborhood aggregation and network sampling. 
There has also been a lot of progress in improving GCNs in various ways. 
For example, taking adversarial attack on graphs into account can make GCNs more robust~\cite{zugner2018adversarial, dai2018adversarial}; %using generative adversarial nets (GANs) to generate fake graph nodes for better node classification~\cite{ding2018semi};
leveraging sub-graph selection algorithms can address the memory and computational resource problems on much larger-scale graph data~\cite{gao2018large}. 
%active learning query strategy can be utilized in GCNs for alleviating the cost of manually labelling network data. 
%We think more powerful GCNs can be designed if these ideas from other researchers are explored when building upon our framework.
} %end of hide =====================

\begin{acks}
The work is supported by the NSFC for Distinguished Young Scholar (61825602) and 
Tsinghua-Bosch Joint ML Center. 
\end{acks}

\clearpage

%\section{Rights Information}
%
%Authors of any work published by ACM will need to complete a rights form. Depending on the kind of work, and the rights management choice made by the author, this may be copyright transfer, permission, license, or an OA (open access) agreement.
%
%Regardless of the rights management choice, the author will receive a copy of the completed rights form once it has been submitted. This form contains \LaTeX\ commands that must be copied into the source document. When the document source is compiled, these commands and their parameters add formatted text to several areas of the final document:
%\begin{itemize}
%\item the ``ACM Reference Format'' text on the first page.
%\item the ``rights management'' text on the first page.
%\item the conference information in the page header(s).
%\end{itemize}
%
%Rights information is unique to the work; if you are preparing several works for an event, make sure to use the correct set of commands with each of the works.

%\balance
%
% The next two lines define the bibliography style to be used, and the bibliography file.
\bibliographystyle{ACM-Reference-Format}
\bibliography{sample-base}

%%% -*-BibTeX-*-
%%% Do NOT edit. File created by BibTeX with style
%%% ACM-Reference-Format-Journals [18-Jan-2012].

\begin{thebibliography}{34}

%%% ====================================================================
%%% NOTE TO THE USER: you can override these defaults by providing
%%% customized versions of any of these macros before the \bibliography
%%% command.  Each of them MUST provide its own final punctuation,
%%% except for \shownote{}, \showDOI{}, and \showURL{}.  The latter two
%%% do not use final punctuation, in order to avoid confusing it with
%%% the Web address.
%%%
%%% To suppress output of a particular field, define its macro to expand
%%% to an empty string, or better, \unskip, like this:
%%%
%%% \newcommand{\showDOI}[1]{\unskip}   % LaTeX syntax
%%%
%%% \def \showDOI #1{\unskip}           % plain TeX syntax
%%%
%%% ====================================================================

\ifx \showCODEN    \undefined \def \showCODEN     #1{\unskip}     \fi
\ifx \showDOI      \undefined \def \showDOI       #1{#1}\fi
\ifx \showISBNx    \undefined \def \showISBNx     #1{\unskip}     \fi
\ifx \showISBNxiii \undefined \def \showISBNxiii  #1{\unskip}     \fi
\ifx \showISSN     \undefined \def \showISSN      #1{\unskip}     \fi
\ifx \showLCCN     \undefined \def \showLCCN      #1{\unskip}     \fi
\ifx \shownote     \undefined \def \shownote      #1{#1}          \fi
\ifx \showarticletitle \undefined \def \showarticletitle #1{#1}   \fi
\ifx \showURL      \undefined \def \showURL       {\relax}        \fi
% The following commands are used for tagged output and should be
% invisible to TeX
\providecommand\bibfield[2]{#2}
\providecommand\bibinfo[2]{#2}
\providecommand\natexlab[1]{#1}
\providecommand\showeprint[2][]{arXiv:#2}

\bibitem[Abu-El-Haija et~al\mbox{.}(2019)]%
        {abu2019mixhop}
\bibfield{author}{\bibinfo{person}{Sami Abu-El-Haija}, \bibinfo{person}{Bryan
  Perozzi}, \bibinfo{person}{Amol Kapoor}, \bibinfo{person}{Hrayr Harutyunyan},
  \bibinfo{person}{Nazanin Alipourfard}, \bibinfo{person}{Kristina Lerman},
  \bibinfo{person}{Greg~Ver Steeg}, {and} \bibinfo{person}{Aram Galstyan}.}
  \bibinfo{year}{2019}\natexlab{}.
\newblock \showarticletitle{Mixhop: Higher-order graph convolution
  architectures via sparsified neighborhood mixing}.
\newblock \bibinfo{journal}{\emph{ICML'19}} (\bibinfo{year}{2019}).
\newblock


\bibitem[Andersen et~al\mbox{.}(2006)]%
        {andersen2006local}
\bibfield{author}{\bibinfo{person}{Reid Andersen}, \bibinfo{person}{Fan Chung},
  {and} \bibinfo{person}{Kevin Lang}.} \bibinfo{year}{2006}\natexlab{}.
\newblock \showarticletitle{Local graph partitioning using pagerank vectors}.
  In \bibinfo{booktitle}{\emph{FOCS'06}}. IEEE, \bibinfo{pages}{475--486}.
\newblock


\bibitem[Banerjee and Lofgren(2015)]%
        {banerjee2015fast}
\bibfield{author}{\bibinfo{person}{Siddhartha Banerjee} {and}
  \bibinfo{person}{Peter Lofgren}.} \bibinfo{year}{2015}\natexlab{}.
\newblock \showarticletitle{Fast Bidirectional Probability Estimation in Markov
  Models}.
\newblock \bibinfo{journal}{\emph{NeurIPS}} (\bibinfo{year}{2015}).
\newblock


\bibitem[Berthelot et~al\mbox{.}(2019)]%
        {berthelot2019mixmatch}
\bibfield{author}{\bibinfo{person}{David Berthelot}, \bibinfo{person}{Nicholas
  Carlini}, \bibinfo{person}{J.~Ian Goodfellow}, \bibinfo{person}{Nicolas
  Papernot}, \bibinfo{person}{Avital Oliver}, {and} \bibinfo{person}{Colin
  Raffel}.} \bibinfo{year}{2019}\natexlab{}.
\newblock \showarticletitle{MixMatch: A Holistic Approach to Semi-Supervised
  Learning}.
\newblock \bibinfo{journal}{\emph{NeurIPS}} (\bibinfo{year}{2019}),
  \bibinfo{pages}{5050--5060}.
\newblock


\bibitem[Bojchevski et~al\mbox{.}(2020)]%
        {bojchevski2020scaling}
\bibfield{author}{\bibinfo{person}{Aleksandar Bojchevski},
  \bibinfo{person}{Johannes Klicpera}, \bibinfo{person}{Bryan Perozzi},
  \bibinfo{person}{Amol Kapoor}, \bibinfo{person}{Martin Blais},
  \bibinfo{person}{Benedek R{\'o}zemberczki}, \bibinfo{person}{Michal Lukasik},
  {and} \bibinfo{person}{Stephan G{\"u}nnemann}.}
  \bibinfo{year}{2020}\natexlab{}.
\newblock \showarticletitle{Scaling graph neural networks with approximate
  pagerank}. In \bibinfo{booktitle}{\emph{KDD'20}}.
  \bibinfo{pages}{2464--2473}.
\newblock


\bibitem[Chapelle et~al\mbox{.}(2006)]%
        {chapelle2009semi}
\bibfield{author}{\bibinfo{person}{Olivier Chapelle}, \bibinfo{person}{Bernhard
  Schölkopf}, {and} \bibinfo{person}{Alexander Zien}.}
  \bibinfo{year}{2006}\natexlab{}.
\newblock \bibinfo{booktitle}{\emph{Semi-supervised learning}}.
\newblock \bibinfo{publisher}{The MIT Press}.
\newblock
\urldef\tempurl%
\url{https://doi.org/10.7551/mitpress/9780262033589.001.0001}
\showDOI{\tempurl}


\bibitem[Chen et~al\mbox{.}(2018)]%
        {FastGCN}
\bibfield{author}{\bibinfo{person}{Jie Chen}, \bibinfo{person}{Tengfei Ma},
  {and} \bibinfo{person}{Cao Xiao}.} \bibinfo{year}{2018}\natexlab{}.
\newblock \showarticletitle{FastGCN: Fast Learning with Graph Convolutional
  Networks via Importance Sampling}.
\newblock \bibinfo{journal}{\emph{ICLR}} (\bibinfo{year}{2018}).
\newblock


\bibitem[Chen et~al\mbox{.}(2020a)]%
        {chen2020scalable}
\bibfield{author}{\bibinfo{person}{Ming Chen}, \bibinfo{person}{Zhewei Wei},
  \bibinfo{person}{Bolin Ding}, \bibinfo{person}{Yaliang Li},
  \bibinfo{person}{Ye Yuan}, \bibinfo{person}{Xiaoyong Du}, {and}
  \bibinfo{person}{Ji-Rong Wen}.} \bibinfo{year}{2020}\natexlab{a}.
\newblock \showarticletitle{Scalable Graph Neural Networks via Bidirectional
  Propagation}.
\newblock \bibinfo{journal}{\emph{NeurIPS}} (\bibinfo{year}{2020}).
\newblock


\bibitem[Chen et~al\mbox{.}(2020b)]%
        {chen2020simple}
\bibfield{author}{\bibinfo{person}{Ming Chen}, \bibinfo{person}{Zhewei Wei},
  \bibinfo{person}{Zengfeng Huang}, \bibinfo{person}{Bolin Ding}, {and}
  \bibinfo{person}{Yaliang Li}.} \bibinfo{year}{2020}\natexlab{b}.
\newblock \showarticletitle{Simple and deep graph convolutional networks}. In
  \bibinfo{booktitle}{\emph{ICML}}. PMLR, \bibinfo{pages}{1725--1735}.
\newblock


\bibitem[Chiang et~al\mbox{.}(2019)]%
        {Chiang2019ClusterGCN}
\bibfield{author}{\bibinfo{person}{Wei-Lin Chiang}, \bibinfo{person}{Xuanqing
  Liu}, \bibinfo{person}{Si Si}, \bibinfo{person}{Yang Li},
  \bibinfo{person}{Samy Bengio}, {and} \bibinfo{person}{Cho-Jui Hsieh}.}
  \bibinfo{year}{2019}\natexlab{}.
\newblock \showarticletitle{Cluster-GCN: An Efficient Algorithm for Training
  Deep and Large Graph Convolutional Networks}. In
  \bibinfo{booktitle}{\emph{KDD'19}}.
\newblock


\bibitem[Ding et~al\mbox{.}(2018)]%
        {ding2018semi}
\bibfield{author}{\bibinfo{person}{Ming Ding}, \bibinfo{person}{Jie Tang},
  {and} \bibinfo{person}{Jie Zhang}.} \bibinfo{year}{2018}\natexlab{}.
\newblock \showarticletitle{Semi-supervised learning on graphs with generative
  adversarial nets}.
\newblock \bibinfo{journal}{\emph{CIKM'18}} (\bibinfo{year}{2018}).
\newblock


\bibitem[Feng et~al\mbox{.}(2020)]%
        {feng2020grand}
\bibfield{author}{\bibinfo{person}{Wenzheng Feng}, \bibinfo{person}{Jie Zhang},
  \bibinfo{person}{Yuxiao Dong}, \bibinfo{person}{Yu Han},
  \bibinfo{person}{Huanbo Luan}, \bibinfo{person}{Qian Xu},
  \bibinfo{person}{Qiang Yang}, \bibinfo{person}{Evgeny Kharlamov}, {and}
  \bibinfo{person}{Jie Tang}.} \bibinfo{year}{2020}\natexlab{}.
\newblock \showarticletitle{Graph Random Neural Network for Semi-Supervised
  Learning on Graphs}.
\newblock \bibinfo{journal}{\emph{NeurIPS}} (\bibinfo{year}{2020}).
\newblock


\bibitem[Goyal et~al\mbox{.}(2017)]%
        {goyal2017accurate}
\bibfield{author}{\bibinfo{person}{Priya Goyal}, \bibinfo{person}{Piotr
  Doll{\'a}r}, \bibinfo{person}{Ross Girshick}, \bibinfo{person}{Pieter
  Noordhuis}, \bibinfo{person}{Lukasz Wesolowski}, \bibinfo{person}{Aapo
  Kyrola}, \bibinfo{person}{Andrew Tulloch}, \bibinfo{person}{Yangqing Jia},
  {and} \bibinfo{person}{Kaiming He}.} \bibinfo{year}{2017}\natexlab{}.
\newblock \showarticletitle{Accurate, large minibatch sgd: Training imagenet in
  1 hour}.
\newblock \bibinfo{journal}{\emph{arXiv preprint arXiv:1706.02677}}
  (\bibinfo{year}{2017}).
\newblock


\bibitem[Hamilton et~al\mbox{.}(2017)]%
        {hamilton2017inductive}
\bibfield{author}{\bibinfo{person}{Will Hamilton}, \bibinfo{person}{Zhitao
  Ying}, {and} \bibinfo{person}{Jure Leskovec}.}
  \bibinfo{year}{2017}\natexlab{}.
\newblock \showarticletitle{Inductive representation learning on large graphs}.
\newblock \bibinfo{journal}{\emph{NeurIPS}} (\bibinfo{year}{2017}),
  \bibinfo{pages}{1025--1035}.
\newblock


\bibitem[Ioffe and Szegedy(2015)]%
        {ioffe2015batch}
\bibfield{author}{\bibinfo{person}{Sergey Ioffe} {and}
  \bibinfo{person}{Christian Szegedy}.} \bibinfo{year}{2015}\natexlab{}.
\newblock \showarticletitle{Batch normalization: Accelerating deep network
  training by reducing internal covariate shift}.
\newblock \bibinfo{journal}{\emph{ICML}} (\bibinfo{year}{2015}).
\newblock


\bibitem[Keskar et~al\mbox{.}(2017)]%
        {keskar2016large}
\bibfield{author}{\bibinfo{person}{Shirish~Nitish Keskar},
  \bibinfo{person}{Dheevatsa Mudigere}, \bibinfo{person}{Jorge Nocedal},
  \bibinfo{person}{Mikhail Smelyanskiy}, {and} \bibinfo{person}{Tak Peter~Ping
  Tang}.} \bibinfo{year}{2017}\natexlab{}.
\newblock \showarticletitle{On Large-Batch Training for Deep Learning:
  Generalization Gap and Sharp Minima}.
\newblock \bibinfo{journal}{\emph{ICLR}} (\bibinfo{year}{2017}).
\newblock


\bibitem[Kihyuk et~al\mbox{.}(2020)]%
        {sohn2020fixmatch}
\bibfield{author}{\bibinfo{person}{Sohn Kihyuk}, \bibinfo{person}{Berthelot
  David}, \bibinfo{person}{Li Chun-Liang}, \bibinfo{person}{Zhang Zizhao},
  \bibinfo{person}{Carlini Nicholas}, \bibinfo{person}{Ekin~Cubuk D.},
  \bibinfo{person}{Kurakin Alex}, \bibinfo{person}{Zhang Han}, {and}
  \bibinfo{person}{Raffel Colin}.} \bibinfo{year}{2020}\natexlab{}.
\newblock \showarticletitle{FixMatch: Simplifying Semi-Supervised Learning with
  Consistency and Confidence}.
\newblock \bibinfo{journal}{\emph{NeurIPS}} (\bibinfo{year}{2020}).
\newblock


\bibitem[Kingma and Ba(2015)]%
        {kingma2014adam}
\bibfield{author}{\bibinfo{person}{P.~Diederik Kingma} {and}
  \bibinfo{person}{Lei~Jimmy Ba}.} \bibinfo{year}{2015}\natexlab{}.
\newblock \showarticletitle{Adam: A Method for Stochastic Optimization}.
\newblock \bibinfo{journal}{\emph{ICLR}} (\bibinfo{year}{2015}).
\newblock


\bibitem[Kipf and Welling(2017)]%
        {kipf2016semi}
\bibfield{author}{\bibinfo{person}{N.~Thomas Kipf} {and} \bibinfo{person}{Max
  Welling}.} \bibinfo{year}{2017}\natexlab{}.
\newblock \showarticletitle{Semi-Supervised Classification with Graph
  Convolutional Networks}.
\newblock \bibinfo{journal}{\emph{ICLR}} (\bibinfo{year}{2017}).
\newblock


\bibitem[Klicpera et~al\mbox{.}(2019a)]%
        {klicpera2018predict}
\bibfield{author}{\bibinfo{person}{Johannes Klicpera},
  \bibinfo{person}{Aleksandar Bojchevski}, {and} \bibinfo{person}{Stephan
  Günnemann}.} \bibinfo{year}{2019}\natexlab{a}.
\newblock \showarticletitle{Predict then Propagate: Graph Neural Networks meet
  Personalized PageRank}.
\newblock \bibinfo{journal}{\emph{ICLR}} (\bibinfo{year}{2019}).
\newblock


\bibitem[Klicpera et~al\mbox{.}(2019b)]%
        {klicpera2019diffusion}
\bibfield{author}{\bibinfo{person}{Johannes Klicpera}, \bibinfo{person}{Stefan
  Wei{\ss}enberger}, {and} \bibinfo{person}{Stephan G{\"u}nnemann}.}
  \bibinfo{year}{2019}\natexlab{b}.
\newblock \showarticletitle{Diffusion improves graph learning}.
\newblock \bibinfo{journal}{\emph{arXiv preprint arXiv:1911.05485}}
  (\bibinfo{year}{2019}).
\newblock


\bibitem[Li et~al\mbox{.}(2019)]%
        {li2019optimizing}
\bibfield{author}{\bibinfo{person}{Pan Li}, \bibinfo{person}{Eli Chien}, {and}
  \bibinfo{person}{Olgica Milenkovic}.} \bibinfo{year}{2019}\natexlab{}.
\newblock \showarticletitle{Optimizing generalized pagerank methods for
  seed-expansion community detection}.
\newblock \bibinfo{journal}{\emph{arXiv preprint arXiv:1905.10881}}
  (\bibinfo{year}{2019}).
\newblock


\bibitem[Li et~al\mbox{.}(2018)]%
        {li2018deeper}
\bibfield{author}{\bibinfo{person}{Qimai Li}, \bibinfo{person}{Zhichao Han},
  {and} \bibinfo{person}{Xiao-Ming Wu}.} \bibinfo{year}{2018}\natexlab{}.
\newblock \showarticletitle{Deeper insights into graph convolutional networks
  for semi-supervised learning}. In \bibinfo{booktitle}{\emph{AAAI'18}}.
\newblock


\bibitem[Pascanu et~al\mbox{.}(2013)]%
        {pascanu2013difficulty}
\bibfield{author}{\bibinfo{person}{Razvan Pascanu}, \bibinfo{person}{Tomas
  Mikolov}, {and} \bibinfo{person}{Yoshua Bengio}.}
  \bibinfo{year}{2013}\natexlab{}.
\newblock \showarticletitle{On the difficulty of training recurrent neural
  networks}.
\newblock \bibinfo{journal}{\emph{ICML}} (\bibinfo{year}{2013}),
  \bibinfo{pages}{1310--1318}.
\newblock


\bibitem[Qiu et~al\mbox{.}(2018)]%
        {qiu2018network}
\bibfield{author}{\bibinfo{person}{Jiezhong Qiu}, \bibinfo{person}{Yuxiao
  Dong}, \bibinfo{person}{Hao Ma}, \bibinfo{person}{Jian Li},
  \bibinfo{person}{Kuansan Wang}, {and} \bibinfo{person}{Jie Tang}.}
  \bibinfo{year}{2018}\natexlab{}.
\newblock \showarticletitle{Network embedding as matrix factorization: Unifying
  deepwalk, line, pte, and node2vec}. In \bibinfo{booktitle}{\emph{WSDM'18}}.
  \bibinfo{pages}{459--467}.
\newblock


\bibitem[Sinha et~al\mbox{.}(2015)]%
        {sinha2015an}
\bibfield{author}{\bibinfo{person}{Arnab Sinha}, \bibinfo{person}{Zhihong
  Shen}, \bibinfo{person}{Yang Song}, \bibinfo{person}{Hao Ma},
  \bibinfo{person}{Darrin Eide}, \bibinfo{person}{Paul Bo-June Hsu}, {and}
  \bibinfo{person}{Kuansan Wang}.} \bibinfo{year}{2015}\natexlab{}.
\newblock \showarticletitle{An Overview of Microsoft Academic Service (MAS) and
  Applications}.
\newblock \bibinfo{journal}{\emph{WWW (Companion Volume)}}
  (\bibinfo{year}{2015}), \bibinfo{pages}{243--246}.
\newblock


\bibitem[Tang et~al\mbox{.}(2008)]%
        {tang2008arnetminer}
\bibfield{author}{\bibinfo{person}{Jie Tang}, \bibinfo{person}{Jing Zhang},
  \bibinfo{person}{Limin Yao}, \bibinfo{person}{Juanzi Li}, \bibinfo{person}{Li
  Zhang}, {and} \bibinfo{person}{Zhong Su}.} \bibinfo{year}{2008}\natexlab{}.
\newblock \showarticletitle{Arnetminer: extraction and mining of academic
  social networks}. In \bibinfo{booktitle}{\emph{KDD'08}}.
\newblock


\bibitem[Velickovic et~al\mbox{.}(2018)]%
        {Velickovic:17GAT}
\bibfield{author}{\bibinfo{person}{Petar Velickovic}, \bibinfo{person}{Guillem
  Cucurull}, \bibinfo{person}{Arantxa Casanova}, \bibinfo{person}{Adriana
  Romero}, \bibinfo{person}{Pietro Li{\`{o}}}, {and} \bibinfo{person}{Yoshua
  Bengio}.} \bibinfo{year}{2018}\natexlab{}.
\newblock \showarticletitle{Graph Attention Networks}.
\newblock \bibinfo{journal}{\emph{ICLR}} (\bibinfo{year}{2018}).
\newblock


\bibitem[Wu et~al\mbox{.}(2019)]%
        {wu2019simplifying}
\bibfield{author}{\bibinfo{person}{Felix Wu}, \bibinfo{person}{Amauri Souza},
  \bibinfo{person}{Tianyi Zhang}, \bibinfo{person}{Christopher Fifty},
  \bibinfo{person}{Tao Yu}, {and} \bibinfo{person}{Kilian Weinberger}.}
  \bibinfo{year}{2019}\natexlab{}.
\newblock \showarticletitle{Simplifying graph convolutional networks}.
\newblock \bibinfo{journal}{\emph{ICML}} (\bibinfo{year}{2019}),
  \bibinfo{pages}{6861--6871}.
\newblock


\bibitem[Yang et~al\mbox{.}(2016)]%
        {yang2016revisiting}
\bibfield{author}{\bibinfo{person}{Zhilin Yang}, \bibinfo{person}{William~W
  Cohen}, {and} \bibinfo{person}{Ruslan Salakhutdinov}.}
  \bibinfo{year}{2016}\natexlab{}.
\newblock \showarticletitle{Revisiting semi-supervised learning with graph
  embeddings}.
\newblock \bibinfo{journal}{\emph{ICML}} (\bibinfo{year}{2016}).
\newblock


\bibitem[Zeng et~al\mbox{.}(2020)]%
        {zeng2020graphsaint}
\bibfield{author}{\bibinfo{person}{Hanqing Zeng}, \bibinfo{person}{Hongkuan
  Zhou}, \bibinfo{person}{Ajitesh Srivastava}, \bibinfo{person}{Rajgopal
  Kannan}, {and} \bibinfo{person}{Viktor Prasanna}.}
  \bibinfo{year}{2020}\natexlab{}.
\newblock \showarticletitle{GraphSAINT: Graph Sampling Based Inductive Learning
  Method}.
\newblock \bibinfo{journal}{\emph{ICLR}} (\bibinfo{year}{2020}).
\newblock


\bibitem[Zhou et~al\mbox{.}(2004)]%
        {zhou2004learning}
\bibfield{author}{\bibinfo{person}{Dengyong Zhou}, \bibinfo{person}{Olivier
  Bousquet}, \bibinfo{person}{Thomas~N Lal}, \bibinfo{person}{Jason Weston},
  {and} \bibinfo{person}{Bernhard Sch{\"o}lkopf}.}
  \bibinfo{year}{2004}\natexlab{}.
\newblock \showarticletitle{Learning with local and global consistency}.
\newblock \bibinfo{journal}{\emph{NeurIPS}} (\bibinfo{year}{2004}),
  \bibinfo{pages}{321--328}.
\newblock


\bibitem[Zhu et~al\mbox{.}(2003)]%
        {zhu2003semi}
\bibfield{author}{\bibinfo{person}{Xiaojin Zhu}, \bibinfo{person}{Zoubin
  Ghahramani}, {and} \bibinfo{person}{John~D Lafferty}.}
  \bibinfo{year}{2003}\natexlab{}.
\newblock \showarticletitle{Semi-supervised learning using gaussian fields and
  harmonic functions}.
\newblock \bibinfo{journal}{\emph{ICML}} (\bibinfo{year}{2003}).
\newblock


\bibitem[Zou et~al\mbox{.}(2019)]%
        {zou2019layer}
\bibfield{author}{\bibinfo{person}{Difan Zou}, \bibinfo{person}{Ziniu Hu},
  \bibinfo{person}{Yewen Wang}, \bibinfo{person}{Song Jiang},
  \bibinfo{person}{Yizhou Sun}, {and} \bibinfo{person}{Quanquan Gu}.}
  \bibinfo{year}{2019}\natexlab{}.
\newblock \showarticletitle{Layer-dependent importance sampling for training
  deep and large graph convolutional networks}.
\newblock \bibinfo{journal}{\emph{NeurIPS}} (\bibinfo{year}{2019}).
\newblock


\end{thebibliography}
\clearpage
\appendix
%\newpage

%\yd{put Reproducibility first}

\section{Appendix}
\setcounter{thm}{0}
\setcounter{lemma}{0}
\setcounter{footnote}{4}

\subsection{Implementation Note}
\label{sec:imp}
\subsubsection{Running environment.} The experiments are conducted on a single Linux server with Intel(R) Xeon(R) CPU Gold 6420 @ 2.60GHz, 376G RAM and 10 NVIDIA GeForce RTX 3090TI-24GB. The Python version is 3.8.5.

\subsubsection{Implementation details.} We implement GFPush with C++, and use OpenMP to perform parallelization. We use Pytorch to implement the training process of \model, and use pybind\footnote{\url{https://github.com/pybind/pybind11}} to create Python binding for approximation module. In \model\ and other baselines, we use BatchNorm~\cite{ioffe2015batch} and gradient clipping~\cite{pascanu2013difficulty} to stabilize the model training, and adopt Adam~\cite{kingma2014adam} for optimization.

\subsubsection{Dataset details.} 
\label{sec:dataset_details}
There are totally 7 datasets used in this paper, that is, Cora, Citeseer, Pubmed, AMiner-CS, Reddit, Amazon2M and MAG-Scholar-C. Our preprocessing scripts for Cora, Citeseer and Pubmed are implemented with reference to the codes of Planetoid\footnote{\url{https://github.com/kimiyoung/planetoid}}. Following the experimental setup used in~\cite{kipf2016semi,yang2016revisiting,Velickovic:17GAT}, we run 100 trials with random seeds for the results on Cora, Citeseer and Pubmed reported in Section 4.2.
AMiner-CS is constructed by Feng et al.~\cite{feng2020grand} based on the AMiner citation network~\cite{tang2008arnetminer}. In AMiner-CS, each node represents a paper, the edges are citation relations, labels are research topics of papers.
Reddit is published by Hamilton et al.~\cite{hamilton2017inductive}, in which each node represents a post in the Reddit community, a graph link represents the two posts have been commented by the same user. The task is to predict the category of each post. Amazon2M is published by Chiang et al.~\cite{Chiang2019ClusterGCN}, where each node is a product, each edge denotes the two products are
purchased together, labels represent the categories of products.
MAG-Scholar-C is constructed by Bojchevski~\cite{bojchevski2020scaling} based on Microsoft Academic Graph (MAG)~\cite{sinha2015an}, in which nodes refer to papers, edges represent citation relations among papers and features are bag-of-words  of paper abstracts.

For AMiner-CS, Reddit, Amazon2M and MAG-Scholar-C, we use 20$\times$\#classes for training, 30$\times$\#classes nodes for validation and the remaining nodes for test. For Aminer, Reddit and  MAG-Scholar-C, we randomly sample the same number of nodes for each class---20 nodes per class for training and 30 nodes per class for validation. For Amazon2M, we uniformly sample all the training and validation nodes from the whole datasets, as the node counts of some classes are less than 20. 
For these datasets, we report the average results of 10 trails with random splits. 
%For MAG-Scholar-C, we report the average results of 10 trails (10 random splits $\times$ 1 run per split).

%We obtain this dataset from\footnote{\url{}}

%We also evaluate our methods on six relatively large datasets,  Cora-Full is proposed in~\cite{Bojchevski2017DeepGE}. Coauthor CS, Coauthor Physics, Amazon Computers and Amazon Photo are proposed in~\cite{shchur2018pitfalls}. We download the processed versions of the five datasets here\footnote{\url{https://github.com/shchur/gnn-benchmark}}.\AMiner-CS CS is conducted by ourselves based on AMiner-CS citation network\cite{tang2008arnetminer}. In AMiner-CS CS, each node corresponds to a paper in computer science, and edges represent citation relations between papers.These papers are manually categorized into 18 topics based on their publication venues.We use averaged GLOVE-100~\cite{pennington-etal-2014-glove} word vector of paper abstract as the node feature vector. Our goal is to predict the corresponding topic of each paper based on feature matrix and citation graph structure.The corresponding results on the six datasets are introduced in Appendix \ref{exp:large_data}. 

\subsubsection{Hyperparameter Selection.} 
\label{sec:hyper_selection}
For results in Table 2, we adjust hyperparameters of \model\ on validation set, and use the best configuration for prediction, and the results of other baseline methods are taken from previous works~\cite{Velickovic:17GAT, chen2020simple,feng2020grand}. For results in Table 3-5, we conduct detailed hyperparameter search for both \model\ and other GNN baselines (i.e., FastGCN, GraphSAINT, SGC, GBP and PPRGo). %The detailed hyperparameter configuration of each method is reported below. 
For each search, we run 3 experiments with random seeds, and select the hyperparameter configuration which gets the best average accuracy on validation set. Then we train model with the selected configuration. 

The hyperparameter selection for \model consists of two stages: We first conduct search for basic hyperparameters of neural network. Specifically, we search learning rate $lr$ from \{$10^{-2},10^{-3},10^{-4}$\}, weight decay rate $wr$ from \{$0,10^{-5},10^{-3},10^{-2}$\}, number of hidden layer $L_m$ from \{1,2\}  and dimension of hidden layer $d_h$ from \{32, 64, 128, 256, 512, 1024\}. 
%he search space of these parameters are listed in Table~\ref{tab:hyper1}. T

In the second stage, we fix these basic hyperparameters as best configurations and search the following specific hyperparameters:
%\model\ introduces several specific hyperparameters: 
DropNode rate $\delta$, augmentation times per batch $M$, threshold $r_{max}$ in GFPush, maximum neighborhood size $k$, propagation order $N$, confidence threshold $\gamma$, maximum consistency loss weight $\lambda_{max}$, size of unlabeled subset $|U'|$ and consistency loss function $\mathcal{D}$. To reduce searching cost, we keep some hyperparameters fixed.  Specifically, we fix $\delta=0.5$, $M=2$ and $\gamma=\frac{2}{\text{\#classes}}$ across all datasets. We set $|U'|=|U|$ for Cora, Pubmed and Citeseer, and set $|U'|=10000$ for other datasets. We also provide an analysis for the effect of $|U'|$ in Appendix~\ref{sec:add_exp}.
We adopt $KL$ divergence as the consistency loss function for AMiner-CS, Reddit and Amazon2M, and use $L_2$ distance for other datasets. This is because $L_2$ distance is easily to suffer from gradient vanishing problem when dealing with datasets with a large number of classes. 
%are not sensitive to performance and fix them to 0.5 and 2 respectively. We also fix $\gamma$ to $2/\#classes$, and find it works well across all datasets. 
We then conduct hyperparameter selection for $r_{max}$, $k$, $N$ and $\lambda_{max}$. Specifically, we search $r_{max}$ from \{$10^{-5},10^{-6},10^{-7}$\}, $k$ from \{16, 32, 64, 128\}, $N$ from \{2, 4, 6, 8, 10, 20\} and $\lambda_{max}$ from \{0.5, 0.8, 1.0, 1.2, 1.5\}.
%the configurations reported in Table~\ref{tab:hyper2}. 
The selected best hyperparameter configurations of \model\ are reported in Table~\ref{tab:grand+p}.
\hide{
\begin{table}[h!]
    \centering
    \small
% \footnotesize
 \caption{Search Space of Basic Hyperparameters.}
 \vspace{-0.1in}
    \begin{tabular}{c|c}
    \toprule
    Hyperparameter & Search Space\\
 \midrule   
 $lr$  & \{$10^{-2},10^{-3},10^{-4}$\} \\
 $wr$ & \{$0,10^{-1},10^{-2},10^{-3}$\} \\
 $L$ & \{1,2\}        \\
 $h$ & \{32, 64, 128, 256, 512, 1024\} \\
\bottomrule
 \end{tabular}
 \label{tab:hyper1}
\end{table}

\begin{table}[h!]
    \centering
    \small
% \footnotesize
 \caption{Search Space of \model's Hyperparameters.}
 \vspace{-0.1in}
    \begin{tabular}{c|c}
    \toprule
    Hyperparameter & Search Space\\
 \midrule   
 $r_{max}$  & \{$10^{-5},10^{-6},10^{-7}$\} \\
 $k$ & \{16, 32, 64, 128\} \\
 $N$ & \{2, 4, 6, 8, 10\}        \\
 $\lambda_{max}$ & \{0.5, 0.8, 1.0, 1.2, 1.5\} \\
\bottomrule
 \end{tabular}
 \label{tab:hyper2}
\end{table}
}
 
\begin{table}[t]
\small
\caption{Hyperparameter configuration of \model.}
\vspace{-5px}
\label{tab:grand+p}
\setlength{\tabcolsep}{0.8mm}
\begin{tabular}{c c|c c c c c c c c c c c}
\toprule
& & lr & wr & $L_m$ & $d_h$ & $r_{max}$ & $k$ & $N$ & $\lambda_{max}$ \\
\hline
%  & & \multicolumn{2}{c|}{Alibaba} & \multicolumn{2}{c|}{Yelp2018} & \multicolumn{2}{c}{Amazon} \\ 
%  & & Recall & NDCG & Recall & NDCG & Recall & NDCG \\ \hline\hline

\multirow{3}*{Cora} & \model\ (P) & $10^{-2}$ & $10^{-3}$ & 2 & 64 & $10^{-7}$ & 32 & 20 & 1.5 \\
 & \model\ (A) & $10^{-2}$ & $10^{-3}$ & 2 & 64 & $10^{-7}$ & 32 & 4 & 1.5\\
 & \model\ (S) & $10^{-2}$ & $10^{-3}$ & 2 & 64 & $10^{-7}$ & 32 & 2 & 1.5 \\
\midrule
\multirow{3}*{Citeseer} & \model\ (P) & $10^{-3}$ & $10^{-3}$ & 2 & 256 & $10^{-7}$ & 32 & 10 & 0.8 \\
 & \model\ (A) & $10^{-3}$ & $10^{-3}$  & 2 & 256 & $10^{-7}$ & 32 & 2 & 0.8 \\
 & \model\ (S) & $10^{-3}$ & $10^{-3}$  & 2 & 256 & $10^{-7}$ & 32 & 2 & 0.8 \\
\midrule
\multirow{3}*{Pubmed} & \model\ (P) & $10^{-2}$ & $10^{-2}$ & 1 & -  & $10^{-5}$ & 16 & 6 & 1.0 &\\
 & \model\ (A) & $10^{-2}$ & $10^{-2}$ & 1 & - & $10^{-5}$ & 16 & 4 & 1.0 \\
 & \model\ (S) & $10^{-2}$ & $10^{-2}$ & 1 & - & $10^{-5}$ & 16 & 2 & 1.0 \\
\midrule
\multirow{3}*{AMiner-CS} & \model\ (P) & $10^{-2}$ & $10^{-2}$ & 1 & -  & $10^{-5}$ & 64 & 6 & 1.5 \\
 & \model\ (A) & $10^{-2}$ & $10^{-2}$ & 1 & -  & $10^{-5}$ & 64 & 4 & 1.5 \\
& \model\ (S) & $10^{-2}$ & $10^{-2}$ & 1 & - & $10^{-5}$ & 64 & 2 & 1.5 \\
\midrule
\multirow{3}*{Reddit} & \model\ (P) & $10^{-4}$ & $0$ & 2 & 512 & $10^{-5}$ & 64 & 6 & 1.5\\
 & \model\ (A) & $10^{-4}$ & $0$ & 2 & 512 & $10^{-5}$ & 64 & 6 & 1.5 \\
& \model\ (S) & $10^{-4}$ & $0$ & 2 & 512 & $10^{-7}$ & 64 & 2 & 1.5\\
\midrule
\multirow{3}*{Amazon2M} & \model\ (P) & $10^{-3}$ & $10^{-5}$ & 2 & 1024 & $10^{-6}$ & 64 & 6 & 0.8  \\
 & \model\ (A) &  $10^{-3}$ & $10^{-5}$ & 2 & 1024 & $10^{-6}$ & 64 & 4 & 0.8 \\
 & \model\ (S) & $10^{-3}$ & $10^{-5}$& 2 & 1024 & $10^{-6}$ & 32 & 2 & 0.8 \\
\midrule
\multirow{3}*{MAG-Scholar-C} & \model\ (P) & $10^{-2}$ & $0$ & 2 & 32 & $10^{-5}$ & 32 & 10 & 1.0 \\
 & \model\ (A) & $10^{-2}$ & $0$ & 2 & 32 & $10^{-5}$ & 32 & 10 & 1.0 \\
 & \model\ (S) & $10^{-2}$ & $0$ & 2 & 32 & $10^{-5}$ & 32 & 2 & 1.0 \\
\bottomrule
\end{tabular}
\vspace{-0.15in}
\end{table}

\hide{
\vpara{\model.}
\model\ adopts following hyperparameters: learning rate (lr), weight decay rate (wr), the number of MLP layers $L$, hidden size of MLP $h$, DropNode rate $\delta$, augmentation times per batch $M$, threshold $r_{max}$ in GFPush, maximum neighborhood size $k$, propagation order $T$, confidence threshold $\gamma$, maximum consistency loss weight $\lambda_{min}$, maximum consistency loss weight $\lambda_{max}$, minimum \#batches for dynamic loss weight scheduling $n_{max}$, batch size of labeled nodes $b_l$,  batch size of unlabeled nodes $b_u$, distance function $\mathcal{D}(\cdot, \cdot)$ used in consistency loss ($L_2$ norm or KL divergence) and decay factor for ppr matrix $\alpha$. In our implementation, we set  $\lambda_{min}$ to $0$, set $\lambda_{max}$ to $1$, and DropNode rate $\delta=0.5$. Other hyperparameters of the three variants of \model\ on each dataset are reported in Table~\ref{tab:grand+p}.
\begin{table}[t]
\scriptsize
\caption{Hyperparameter configuration of \model.}
\vspace{-10px}
\label{tab:grand+p}
\setlength{\tabcolsep}{0.2mm}
\resizebox{0.48\textwidth}{!}{
\begin{tabular}{c c|c c c c c c c c c c c c c c c}
\hline
& & lr & wr & $L$ & $h$ & $M$ & $r_{max}$ & $k$ & $T$ & $\gamma$ & $n_{max}$ & $\tau$ & $\alpha$ & $b_l$ & $b_u$ & $\mathcal{D}$\\
\hline
%  & & \multicolumn{2}{c|}{Alibaba} & \multicolumn{2}{c|}{Yelp2018} & \multicolumn{2}{c}{Amazon} \\ 
%  & & Recall & NDCG & Recall & NDCG & Recall & NDCG \\ \hline\hline

\multirow{3}*{Cora} & \model\ (ppr) & $10^{-2}$ & $10^{-3}$ & 2 & 64 & 2 & $10^{-7}$ & 32 & 20 & 0.3 & 1000 & 0.1 & 0.2 & 50 & 100 & $L_2$\\
 & \model\ (avg) & $10^{-2}$ & $10^{-3}$ & 2 & 64 & 2 & $10^{-7}$ & 32 & 4 & 0.3 & 1000 & 0.1 & - & 50 & 100 & $L_2$\\
 & \model\ (single) & $10^{-2}$ & $10^{-3}$ & 2 & 64 & 2 & $10^{-7}$ & 32 & 2 & 0.3 & 1000 & 0.1 & - & 50 & 100 & $L_2$\\
\hline
\multirow{3}*{Citeseer} & \model\ (ppr) & $10^{-2}$ & $10^{-3}$ & 2 & 256 & 4 & $10^{-6}$ & 32 & 6 & 0.4 & 500 & 0.1 & 0.4 & 50 & 100 & $L_2$\\
 & \model\ (avg) & $10^{-2}$ & $10^{-3}$ & 2 & 256 & 4 & $10^{-6}$ & 32 & 2 & 0.3 & 1000 & 0.1 & - & 50 & 100 & $L_2$\\
 & \model\ (single) & $10^{-2}$ & $10^{-3}$ & 2 & 256 & 4 & $10^{-6}$ & 32 & 2 & 0.4 & 1000 & 0.1 & - & 50 & 100 & $L_2$\\
\hline
\multirow{3}*{Pubmed} & \model\ (ppr) & $10^{-2}$ & $10^{-2}$ & 1 & - & 2 & $10^{-5}$ & 16 & 6 & 0.6 & 100 & 0.1 & 0.5 & 5 & 100 & $L_2$\\
 & \model\ (avg) & $10^{-2}$ & $10^{-2}$ & 1 & - & 2 & $10^{-5}$ & 16 & 4 & 0.6 & 100 & 0.1 & - & 5 & 100 & $L_2$\\
 & \model\ (single) & $10^{-2}$ & $10^{-2}$ & 1 & - & 2 & $10^{-5}$ & 16 & 2 & 0.6 & 100 & 0.1 & - & 5 & 100 & $L_2$\\
\hline
\multirow{3}*{AMiner-CS} & \model\ (ppr) & $10^{-2}$ & $10^{-2}$ & 1 & - & 2 & $10^{-4}$ & 64 & 6 & 0.1 & 100 & 0.1 & 0.1 & 20 & 100 & $KL$\\
 & \model\ (avg) & $10^{-2}$ & $10^{-2}$ & 1 & - & 2 & $10^{-4}$ & 64 & 5 & 0.1 & 100 & 0.1 & - & 20 & 100 & $KL$\\
& \model\ (single) & $10^{-2}$ & $10^{-2}$ & 1 & - & 2 & $10^{-4}$ & 64 & 2 & 0.1 & 100 & 0.1 & - & 20 & 100 & $KL$\\
\hline
\multirow{3}*{Reddit} & \model\ (ppr) & $10^{-4}$ & $0$ & 2 & 512 & 2 & $10^{-5}$ & 64 & 6 & 0.1 & 100 & 0.5 & 0.05 & 50 & 50 & $KL$\\
 & \model\ (avg) & $10^{-4}$ & $0$ & 2 & 512 & 2 & $10^{-5}$ & 64 & 6 & 0.05 & 100 & 0.5 & -& 50 & 50 & $KL$\\
& \model\ (single) & $10^{-4}$ & $0$ & 2 & 512 & 2 & $10^{-5}$ & 64 & 2 & 0.05 & 100 & 0.5 & -& 50 & 50 & $KL$\\
\hline
\multirow{3}*{Amazon2M} & \model\ (ppr) & $10^{-3}$ & $0$ & 2 & 1024 & 2 & $10^{-6}$ & 64 & 4 & 0.5 & 1000 & 0.5 & 0.2& 50 & 50 & $KL$ \\
 & \model\ (avg) &  $10^{-3}$ & $0$ & 2 & 1024 & 2 & $10^{-6}$ & 64 & 4 & 0.5 & 1000 & 0.5 & - & 50 & 50 & $KL$\\
 & \model\ (single) & $10^{-3}$ & $0$ & 2 & 1024 & 2 & $10^{-6}$ & 32 & 4 & 0.5 & 1000 & 0.5 & - & 50 & 50 & $KL$\\
\hline
\multirow{3}*{MAG-Scholar-C} & \model\ (ppr) & $10^{-2}$ & $0$ & 2 & 32 & 2 & $10^{-5}$ & 32 & 10 & 0.5 & 1000 & 0.1 & 0.2 & 20 & 20 & $L_2$\\
 & \model\ (avg) & $10^{-2}$ & $0$ & 2 & 32 & 2 & $10^{-5}$ & 32 & 10 & 0.5 & 1000 & 0.1 & - & 20 & 20 & $L_2$\\
 & \model\ (single) & $10^{-2}$ & $0$ & 2 & 32 & 2 & $10^{-5}$ & 32 & 2 & 0.5 & 1000 & 0.1 & - & 20 & 20 & $L_2$\\
\hline
\end{tabular}}
\vspace{-10px}
\end{table}

\begin{table}[ht]
\scriptsize
\caption{Hyperparameter configuration of GraphSAGE.}
\vspace{-10px}
\label{tab:graphsagep}
\resizebox{0.48\textwidth}{!}{
\begin{tabular}{c| c c c c c c }
\hline
Dataset & lr & wr & $L$ & $h$ & neighborhood size & batch size \\
\hline
AMiner-CS &$10^{-2}$ & $10^{-4}$ & 2 & 64 & 10 & 20\\
Reddit &$10^{-4}$ & $0$ & 2 & 256 & 10 & 20 \\
Amazon2M & $10^{-3}$ & 0 & 2 & 1024 & 10 & 150 \\
% MAG-Scholar-C &$10^{-2}$ & $0$ & 2 & 128 & 10 & 40\\
\hline
\end{tabular}}
\end{table}

\vpara{GraphSAGE.} For GraphSAGE, we use Mean Aggregator in neighborhood aggregation. The hyperparameters mainly include: learning rate (lr), weight decay rate (wr), the number of layers $L$, hidden size $h$, batch size and neighborhood size in node sampling. Table~\ref{tab:graphsagep} shows the corresponding hyperparameter configurations  on AMiner-CS, Reddit and Amazon2M.% and MAG-Scholar-C.

\begin{table}[ht]
\scriptsize
\caption{Hyperparameter configuration of FastGCN.}
\vspace{-10px}
\label{tab:fastgcnp}
\resizebox{0.48\textwidth}{!}{
\begin{tabular}{c| c c c c c c }
\hline
Dataset & lr & wr & $L$ & $h$ & count of sampled nodes & batch size \\
\hline
AMiner-CS &$10^{-2}$ & $5\times 10^{-4}$ & 2 & 64 & 200 & 40\\
Reddit &$10^{-4}$ & $0$ & 2 & 512 & 600 & 100 \\
Amazon2M & $10^{-3}$ & 0 & 2 & 256 & 500 & 150 \\
MAG-Scholar-C &$10^{-2}$ & $0$ & 2 & 128 & 400 & 40\\
\hline
\end{tabular}}
\vspace{-10px}
\end{table}

\vpara{FastGCN.} FastGCN uses the following hyperparameters: learning rate (lr), weight decay rate (wr), the number of layers $L$, hidden size $h$, batch size and count of sampled nodes in each layer. Table~\ref{tab:fastgcnp} shows the corresponding hyperparameter configurations  on AMiner-CS, Reddit, Amazon2M and MAG-Scholar-C.
\begin{table}[ht]
\scriptsize
\caption{Hyperparameter configuration of GraphSAINT.}
\vspace{-10px}
\label{tab:graphsaintp}
\resizebox{0.48\textwidth}{!}{
\begin{tabular}{c| c c c c c c }
\hline
Dataset & lr & weight decay & $L$ & $h$ & walk length & batch size \\
\hline
AMiner-CS &$10^{-3}$ & $0$ & 2 & 128 & 3 & 360\\
Reddit &$10^{-3}$ & $0$ & 2 & 128 & 3 & 820 \\
Amazon2M & $10^{-3}$ & 0 & 2 & 128 & 3 & 940 \\
MAG-Scholar-C &$10^{-2}$ & $0$ & 2 & 128 & 3 & 160\\
\hline
\end{tabular}}
\vspace{-10px}
\end{table}

\vpara{GraphSAINT.} For GraphSAINT, we use Random Walk Sampler for sub-graph sampling. The hyperparameters include: learning rate (lr), weight decay rate (wr), the number of layers $L$, hidden size $h$, batch size and random walk length in sub-graph sampling. Table~\ref{tab:graphsaintp} shows the corresponding hyperparameter configurations  on AMiner-CS, Reddit, Amazon2M and MAG-Scholar-C.

\begin{table}[ht]
\scriptsize
\caption{Hyperparameter configuration of SGC.}
\vspace{-10px}
\label{tab:sgcp}
\resizebox{0.48\textwidth}{!}{
\begin{tabular}{c| c c c c c c }
\hline
Dataset & lr & weight decay & $L$ & $h$ & propagation order & batch size \\
\hline
AMiner-CS &$10^{-2}$ & $5\times10^{-4}$ & 2 & 1024 & 2 & 150\\
Reddit &$10^{-3}$ & $0$ & 2 & 1024 & 2 & 150 \\
Amazon2M & $10^{-4}$ & 0 & 2 & 1024 & 2 & 150 \\
% MAG-Scholar-C &$10^{-3}$ & $0$ & 2 & 128 & 400 & 40\\
\hline
\end{tabular}}
\vspace{-10px}
\end{table}
\vpara{SGC.} The hyperparameters of SGC include: learning rate (lr), weight decay rate (wr), the number of MLP layers $L$, hidden size $h$, batch size and propagation order. Table~\ref{tab:sgcp} shows the corresponding hyperparameter configurations  on AMiner-CS, Reddit, Amazon2M.

\begin{table}[ht]
\scriptsize
\caption{Hyperparameter configuration of GBP.}
\vspace{-10px}
\label{tab:gbpp}
\resizebox{0.48\textwidth}{!}{
\begin{tabular}{c| c c c c c c c}
\hline
Dataset & lr & wr & $L$ & $h$ & $\alpha$ &  $r_{max}$ & batch size\\
\hline
AMiner-CS &$10^{-2}$ & $10^{-4}$ & 2 & 64 & 0.1 & $10^{-6}$ & 30\\
Reddit &$10^{-4}$ & $10^{-5}$ & 2 & 1024 & 0.1 & $10^{-6}$ & 100 \\
Amazon2M & $10^{-3}$ & 0 & 2 & 1024 & 0.2 & $10^{-7}$ & 100 \\
% MAG-Scholar-C &$10^{-3}$ & $0$ & 2 & 128 & 400 & 40\\
\hline
\end{tabular}}
\vspace{-10px}
\end{table}

\vpara{GBP.} For GBP, we adopt personalized pagerank matrix for propagation. The hyperparameters consist of  learning rate (lr), weight decay rate (wr), the number of MLP layers $L$, hidden size $h$, batch size, decay factor $\alpha$ and matrix approximation threshold $r_{max}$. Table~\ref{tab:gbpp} shows the corresponding hyperparameter configurations  on AMiner-CS, Reddit, Amazon2M.

\begin{table}[ht]
\scriptsize
\caption{Hyperparameter configuration of PPRGo.}
\vspace{-10px}
\label{tab:pprgop}
\resizebox{0.48\textwidth}{!}{
\begin{tabular}{c| c c c c c c c}
\hline
Dataset & lr &  wr & $L$ & $h$ & $\alpha$ & $\epsilon$ & batch size \\
\hline
AMiner-CS &$10^{-2}$ & $10^{-2}$ & 1 & - & 0.1 & $10^{-4}$ & 40\\
Reddit &$10^{-4}$ & $0$ & 2 & 1024 & 0.1 & $10^{-5}$& 100 \\
Amazon2M & $10^{-3}$ & 0 & 2 & 1024 & 0.1 & $10^{-7}$ &100 \\
MAG-Scholar-C &$10^{-2}$ & $0$ & 2 & 32 & 0.2 & $10^{-5}$ & 40\\
\hline
\end{tabular}}
\vspace{-10px}
\end{table}
\vpara{PPRGo.} The hyperparameters of PPRGo consist of  learning rate (lr), weight decay rate (wr), the number of MLP layers $L$, hidden size $h$, batch size, decay factor $\alpha$ and matrix approximation threshold $\epsilon$. Table~\ref{tab:pprgop} shows the corresponding hyperparameter configurations  on AMiner-CS, Reddit, Amazon2M and MAG-Scholar-C.
}

%\subsubsection{Details of Ablation Study in Section \ref{sec:ablation}}
\subsection{Theorem Proofs}
\label{sec:proof}
%In this section, we provide the detailed proofs for the Theorem \ref{thm1} and Theorem \ref{thm2} in Section \ref{sec:theory}.
To prove Theorem \ref{thm1}, we first introduce the following lemma:
\begin{lemma}
	\label{lemma1_appendix}
For any reserve vector $\mathbf{q}^{(n)}$, residue vector $\mathbf{r}^{(n)}$ and random walk transition vector $\mathbf{P}^{n}_s=  (\widetilde{\mathbf{D}}^{-1}\widetilde{\mathbf{A}})^n_s$ ($0 \leq n \leq N$), we have:
\begin{equation}
\label{equ:lema}
\small
\mathbf{P}^{n}_s = \mathbf{q}^{(n)} + \sum_{i=1}^{n}  (\mathbf{P}^i)^\mathsf{T} \cdot \mathbf{r}^{(n-i)}
\end{equation}
\end{lemma}

\begin{proof}
\small
We prove the Lemma by induction. For brevity, we use $\mathcal{RHS}^{(n)}$ to denote the right hand side of Equation~\ref{equ:lema}. In Algorithm 1, $\mathbf{q}^{(n)}$ and $\mathbf{r}^{(n)}$ are initialized as $\vec{0}$ for $1 \leq n\leq N$, $\mathbf{r}^{(0)}$ and $\mathbf{q}^{(0)}$ are initialized as $\mathbf{e}^{(s)}$. Thus, Equation~\ref{equ:lema} holds at the algorithm beginning based on the following facts:
$$\mathcal{RHS}^{(0)}
= \mathbf{e}^{(s)} = \mathbf{P}^{0}_s,
$$
$$
\mathcal{RHS}^{(n)} = (\mathbf{P}^{n})^\mathsf{T}\cdot \mathbf{e}^{(s)} = \mathbf{P}^n_s, \quad 1\leq n \leq N.
$$
Then we assume Equation~\ref{equ:lema} holds at beginning of the $n'-$th iteration, we will show that the equation is still correct after a push operation on node $v$.
%with $\mathbf{R}^{(t')}(s,v)>\mathbf{d}_v\cdot r_{max}$. 
We have three cases with different values of $n$:

\textbf{1)} When $n < n'$, the push operation does not change $\mathbf{q}^{(n)}$ and $\mathbf{r}^{(n-i)}, 1\leq i \leq n$. Thus Equation~\ref{equ:lema} holds for $n < n'$.

\textbf{2)} When $n = n'$, the push operation decrements $\mathbf{r}^{(n-1)}$ by $\mathbf{r}^{(n-1)}_v\cdot \mathbf{e}^{(v)}$ and increments $\mathbf{q}^{(n)}$ by $\sum_{u \in \mathcal{N}_v}\mathbf{r}^{(n-1)}_v/\mathbf{d}_v \cdot \mathbf{e}^{(u)}$. Consequently, we have
$$    \mathcal{RHS}^{(n)} = \mathbf{P}^n_s + \sum_{u\in \mathcal{N}_v}\mathbf{r}^{(n-1)}_v/\mathbf{d}_v \cdot \mathbf{e}^{(u)} -\mathbf{r}^{(n-1)}_v \cdot \mathbf{P}_v
 =  \mathbf{P}^n_s + \vec{0} = \mathbf{P}^n_s.
$$
Thus Equation~\ref{equ:lema} holds for $n = n' + 1$.
%For $u\notin  \mathcal{N}(v)$, $\mathcal{RHS}_u$ will not change, thus Equation~\ref{equ:lema} holds for $t = t' +1$.

\textbf{3)} When $n > n'$, the push operation will decrease $\mathbf{r}^{(n')}$ by $\mathbf{r}^{(n')}_v \cdot \mathbf{e}^{(v)}$ and increase $\mathbf{r}^{(n'+1)}$ by $\sum_{u \in \mathcal{N}_v}\mathbf{r}^{(n')}_v/\mathbf{d}_v \cdot \mathbf{e}^{(u)}$. Thus we have
\begin{align*}
\small
    \mathcal{RHS}^{(n)} &=  \mathbf{P}^n_s +  (\mathbf{P}^{n-n'-1})^\mathsf{T} \sum_{u\in \mathcal{N}_v} \mathbf{r}^{(n')}_v/\mathbf{d}_v \cdot \mathbf{e}^{(u)} - (\mathbf{P}^{n-n'})^\mathsf{T} (\mathbf{r}^{(n')}_v \cdot \mathbf{e}^{(v)})\\
    & =  \mathbf{P}^n_s + (\mathbf{P}^{n-n'-1})^\mathsf{T}\cdot \left(\sum_{u\in \mathcal{N}_v} \mathbf{r}^{(n')}_v/\mathbf{d}_v \cdot \mathbf{e}^{(u)} - (\mathbf{r}^{(n')}_v \cdot \mathbf{P}_v)^\mathsf{T}\right)\\
     & =  \mathbf{P}^n_s + (\mathbf{P}^{n-n'-1})^\mathsf{T}\cdot\vec{0} = \mathbf{P}^n_s.
\end{align*}
Thus Equation~\ref{equ:lema} holds for $n > n' + 1$. 

Hence the induction holds, and the lemma is proved.

\end{proof}
Then, we could prove Theorem~\ref{thm1} as following.
 \begin{thm}
	Algorithm 1 has  $\mathcal{O}(N/r_{max})$  time complexity and $\mathcal{O}(N/r_{max})$ memory complexity, and returns $\widetilde{\mathbf{\Pi}}_s$ as an approximation of $\mathbf{\Pi}_s$ with the $L_1$ error bound: $\parallel\mathbf{\Pi}_s - \widetilde{\mathbf{\Pi}}_s\parallel_1 \leq N\cdot (2|E| +|V|) \cdot r_{max}$.
 \end{thm}

\begin{proof}
\small
Let $\mathcal{V}_n$ be the set of nodes to be pushed in step $n$. When the push operation is performed on $v\in \mathcal{V}_n$, the value of $\parallel \mathbf{r}^{(n-1)}\parallel_1$ will be decreased by at least $r_{max} \cdot \mathbf{d}_{v}$. Since $\parallel\mathbf{r}^{(n-1)}\parallel_1 \leq 1$, we must have $\sum_{v \in \mathcal{V}_n}  \mathbf{d}_{v} \cdot r_{max} \leq 1$, thus:
\begin{equation}
\label{equ:tbound}
    \sum_{v\in\mathcal{V}_n} \mathbf{d}_{v} \leq 1/r_{max}.
\end{equation}
\vpara{Time Complexity.} For the push operation on $v$ in step $n$, we need to perform $\mathbf{d}_{v}$ times of updates for $\mathbf{r}^{(n)}$ . So the total time of push operations in step $n$ is bounded by $\sum_{v \in \mathcal{V}_n} \mathbf{d}_{v}$. %Similarly, when the $t$-th step iteration finishes, the count of non-zero elements of $\mathbf{r}^{(t+1)}$ is no more than $\sum_{v \in \mathcal{V}_t} \mathbf{d}_{v}$. 
%Thus the time of copy operation $\mathbf{q}^{(t+1)}_u \leftarrow \mathbf{r}^{(t)}_u$ (Line 4 of Algorithm 1) is also bounded by $\sum_{i=1}^{\mathcal{T}_t} \mathbf{d}_{v_i}$. 
Therefore, based on Equation~\ref{equ:tbound}, the time complexity of each step is bounded by $\mathcal{O}(1/r_{max})$  and the total time complexity of Algorithm 1 has a  $\mathcal{O}(N/r_{max})$ bound.
 
\vpara{Memory Complexity.}
When the $n$-th step iteration finishes, the count of non-zero elements of $\mathbf{r}^{(n)}$ is no more than $\sum_{v \in \mathcal{V}_n} \mathbf{d}_{v}$, as the push operation on $v$ only performs $\mathbf{d}_v$ times of updates for $\mathbf{r}^{(n)}$. 
%We have shown that $\mathbf{R}^{(t+1)}_s$ consists of at most $\sum_{i=1}^{\mathcal{T}_t} \mathbf{d}_{v_i}$ non-zero elements at th beginning of step $t+1$. 
Thus the count of non-zero elements of $\mathbf{q}^{(n)}$ is also less than $\sum_{v \in \mathcal{V}_n} \mathbf{d}_{v}$. 
According to Equation~\ref{equ:tbound}, we can conclude that  $\mathbf{\widetilde{\Pi}}_s$ has at most $N/r_{max}$ non-zero elements.
In implementation, all the vectors are stored as sparse structures. Thus Algorithm 1 has a memory complexity of $\mathcal{O}(N/r_{max})$.

\vpara{Error Bound.} According to Lemma~\ref{lemma1_appendix}, we can conclude the following equations:
\begin{equation}
\label{equ:l1bound1}
\begin{aligned}
    \norm{ \mathbf{P}_s^{(n)} - \mathbf{q}^{(n)} }_1 
    &= \norm{\sum_{i=1}^{n}  (\mathbf{P}^i)^\mathsf{T} \cdot \mathbf{r}^{(n-i)} }_1  \\
    &=  \norm{\sum_{i=1}^{n}\mathbf{r}^{(n-i)} }_1 \\
    & \leq \sum_{i=1}^{n} \norm{\mathbf{r}^{(n-i)}}_1.
\end{aligned}
\end{equation}
After algorithm termination, we have $0 \leq \mathbf{r}^{(n)}_v \leq \mathbf{d}_v \cdot r_{max}$ for all $v \in V$.
%and $0 \leq t \leq T-1$. 
Hence,
\begin{equation}
\norm{\mathbf{r}^{(n)}}_1 = \sum_{v\in V} \mathbf{r}^{(n)}_v \leq \sum_{v\in V}\mathbf{d}_v \cdot r_{max} = (2|E| +|V|) \cdot r_{max}.
\end{equation}
According to Equation~\ref{equ:l1bound1}, we can conclude that $\norm{\mathbf{P}_s^{(n)} - \mathbf{q}^{(n)} }_1 \leq n\cdot (2|E| +|V|) \cdot r_{max}$. Further more, we have:
\begin{equation}
    \begin{aligned}
        \norm{\mathbf{\Pi}_s - \widetilde{\mathbf{\Pi}}_s}_1  \leq \sum_{n=0}^N w_n \norm{\mathbf{P}_s^{(n)} - \mathbf{q}^{(n)}}_1
         \leq \sum_{n=0}^N w_n \cdot n \cdot (2|E| +|V|) \cdot r_{max} 
    \end{aligned}
\end{equation}
As for $0 \leq w_n $ and $\sum_{n=0}^N w_n=1$, hence: 
\begin{equation}
    \sum_{n=0}^N w_n \cdot n \cdot (2|E| +|V|) \cdot r_{max} \leq N \cdot (2|E| +|V|) \cdot r_{max},
\end{equation}    
which indicates $\parallel\mathbf{\Pi}_s - \widetilde{\mathbf{\Pi}}_s\parallel_1 \leq N\cdot (2|E| +|V|) \cdot r_{max}$.
\end{proof}

\subsection{Additional Experiments}
\label{sec:add_exp}
\begin{table}[h]
 	\scriptsize
	\caption{Effects of unlabeled subset size ($|U'|$).}
	\vspace{-0.1in}
	\label{tab:unlabelsize}
	\setlength{\tabcolsep}{1.mm}\begin{tabular}{c|ccc|ccc|ccc}
% 		\toprule
		\toprule
		\multirow{2}{*}{|U'|}  & \multicolumn{3}{c|}{Aminer} & \multicolumn{3}{c|}{Reddit} & \multicolumn{3}{c}{Amazon2M} \\
		& Acc (\%) & RT (s) & AT (ms) & Acc (\%) & RT (s) & AT (ms) & Acc (\%) & RT (s) & AT (ms) \\
		\midrule
		 0 & 51.1 $\pm$ 1.4 & 10& 149 & 92.3 $\pm$ 0.2 & 53 & 717 & 75.0 $\pm$ 0.7 & 63 & 2356 \\
%		 100& 52.2 $\pm$ 1.7 & 8.1 & 151.4 & 92.4 $\pm$ 0.3 & 56.9 & 725.6  &  75.1 $\pm$ 0.7 & 57.5 & 2454.5 \\
         $10^3$ & 53.6 $\pm$ 1.6  & 9 & 153 & 92.6 $\pm$ 1.2 & 58 & 882 & 75.2 $\pm$ 0.5 & 62 &  2630\\       
         $10^4$ & 54.2 $\pm$ 1.6  & 10 & 250 & 92.8 $\pm$ 0.2 & 62 & 2407 &  76.1 $\pm$ 0.6 & 69 & 3649 \\
	    $10^5$ & 54.4 $\pm$ 1.2  & 13 & 1121 &  92.9 $\pm$ 0.2 & 78 & 17670 & 76.3 $\pm$ 0.7 & 86 & 14250 \\
% 		\bottomrule
		\bottomrule
	\end{tabular}
	\vspace{-0.1in}
\end{table}

\vpara{Analysis for the size of $U'$.} In \model, a subset of unlabeled nodes $U'$ are sampled from $U$ for consistency regularization. To this end, we need to pre-compute the sparsified approximation $\widetilde{\mathbf{\Pi}}_v$ of row vector $\mathbf{\Pi}_v$ for each node $v \in U'$.
%$\mathbf{\Pi}_v$ with $v \in U'$.
%order to perform consistency regularization, we need to sample a subset of unlabeled nodes $U'$ from $U$ and approximate the transition vectors for nodes in $U'$ before training. 
Here we empirically analyze how the size of $U'$ affects the classification accuracy (Acc), running time (RT) and approximation time (AT) of \model. %by varying $|U'|$ from 0 to $10^5$.
Table~\ref{tab:unlabelsize} presents the results of \model\ (S) when we vary $|U'|$ from 0 to $10^5$ on AMiner-CS, Reddit and Amazon2M. We have the two observations: First, 
%The results are presented in Table~\ref{tab:unlabelsize}, where we report the  of \model\ (S) on Aminer, Reddit and Amazon2M under different settings of $|U'|$.
%The results are presented in Table~\ref{tab:unlabelsize}. 
as $|U'|$ changes from 0 (meaning the consistency loss degenerates to 0) to $10^3$, the classification performances are improved significantly with little changes on running time, which indicates the consistency regularization serves as an economic way for improving \model's generalization performance under semi-supervised setting. Second, when $|U'|$ exceeds $10^4$, the increase rate of the accuracy will slow down, while the running time and approximation time increase more faster. This observation indicates the sampling procedure on unlabeled nodes is important for ensuring model's efficiency, which also enables us to explicitly control the trade-off between effectiveness and efficiency of \model\ through the sampling size $|U'|$.
%in practice.
%that we can choose an appropriate value for $|U'|$ to balance the trade-off between effectiveness and efficiency of \model\ in practice.

% \input{reproducibility}
%% 
%% If your work has an appendix, this is the place to put it.
%\appendix
%
%\section{Research Methods}
%
%\subsection{Part One}
%
%Lorem ipsum dolor sit amet, consectetur adipiscing elit. Morbi malesuada, quam in pulvinar varius, metus nunc fermentum urna, id sollicitudin purus odio sit amet enim. Aliquam ullamcorper eu ipsum vel mollis. Curabitur quis dictum nisl. Phasellus vel semper risus, et lacinia dolor. Integer ultricies commodo sem nec semper. 
%
%\subsection{Part Two}
%
%Etiam commodo feugiat nisl pulvinar pellentesque. Etiam auctor sodales ligula, non varius nibh pulvinar semper. Suspendisse nec lectus non ipsum convallis congue hendrerit vitae sapien. Donec at laoreet eros. Vivamus non purus placerat, scelerisque diam eu, cursus ante. Etiam aliquam tortor auctor efficitur mattis. 
%
%\section{Online Resources}
%
%Nam id fermentum dui. Suspendisse sagittis tortor a nulla mollis, in pulvinar ex pretium. Sed interdum orci quis metus euismod, et sagittis enim maximus. Vestibulum gravida massa ut felis suscipit congue. Quisque mattis elit a risus ultrices commodo venenatis eget dui. Etiam sagittis eleifend elementum. 
%
%Nam interdum magna at lectus dignissim, ac dignissim lorem rhoncus. Maecenas eu arcu ac neque placerat aliquam. Nunc pulvinar massa et mattis lacinia.

\end{document}